\documentclass[letterpaper,journal]{IEEEtran}
\usepackage{graphicx}
\usepackage{algorithm}
\usepackage{algpseudocode}
\usepackage[dvipsnames]{xcolor}
\usepackage{amsthm}
\usepackage{amsmath}
\usepackage{amssymb}
\usepackage{booktabs}
\usepackage{multirow}
\usepackage{multicol}
\usepackage[caption=false,font=footnotesize]{subfig}
\usepackage{adjustbox}
\usepackage{booktabs, nicematrix}
\usepackage[acronym]{glossaries}
\usepackage[bookmarks=true]{hyperref}
\usepackage[capitalize]{cleveref}
\usepackage{cite}
\usepackage{ifthen}
\newboolean{anonymous}
\setboolean{anonymous}{false}
\newcommand{\tup}[1]{\left( #1\right)}
\newcommand{\soc}{\mathrm{SoC}}

\newcommand{\TableCaptionSkip}{0pt}
\newcommand{\TableVerticalStretch}{1}

\theoremstyle{definition}
\newtheorem{definition}{Definition}[section]

\newtheorem{remark}{Remark}
\newtheorem{proposition}{Proposition}

\DeclareMathOperator*{\argmin}{arg\,min}
\DeclareMathOperator*{\argmax}{arg\,max}

\newacronym{acr:mapf}{MAPF}{Multi-Agent Path Finding}
\newacronym{acr:cbs}{CBS}{Conflict-Based Search}
\newacronym[plural=Mixed Integer Programs,shortplural=MIPs]{acr:mip}{MIP}{Mixed Integer Programming}
\newacronym{acr:bnb}{B\&B}{Branch and Bound}
\newacronym{acr:bfs}{BFS}{Best-First Search}
\newacronym{acr:dfs}{DFS}{Depth-First Search}
\newacronym{acr:id}{ID}{Iterative Deepening}
\newacronym{acr:lp}{LP}{Linear Program}
\newacronym{acr:ct}{CT}{Constraint Tree}
\newacronym{acr:mcd}{MC-DIVE}{Memory-Constrained DIVE}
\newacronym{acr:pibt}{PIBT}{Priority Inheritance with Backtracking}
\newacronym{acr:dive}{DIVE}{Dual-Informed Vertical Expansion}
\newacronym{acr:idcbs}{IDCBS}{Iterative Deepening CBS}

\usepackage[markup=underlined]{changes} %
\definechangesauthor[name=Jiarui, color=purple]{Jiarui}
\definechangesauthor[name=Willem, color=orange]{Willem}
\definechangesauthor[name=Meshal, color=magenta]{Meshal}

\title{Dual-Informed Vertical Expansion for \\ Multi-Objective Node Selection in Anytime Conflict-Based Search}

\ifthenelse{\boolean{anonymous}}{%
  \author{Anonymous Authors}%
}{
\author{
Willem van Osselaer, Jiarui Li, Meshal Alharbi, and Gioele Zardini
\thanks{The authors are with the Laboratory for Information and Decision Systems, Massachusetts Institute of Technology, Cambridge, MA, USA (e-mails: \{willemvo, jiarui01, meshal, gzardini\}@mit.edu).}
\thanks{This work was supported by Prof. Zardini's grant from the MIT Amazon Science Hub, hosted in the MIT Schwarzman College of Computing, and the MIT Maritime Consortium.
Van Osselaer was supported by the Air Force Office of Scientific Research under award number FA9550-25-C-B010 in the amount of \$135,030 through the National Defense Science \& Engineering Graduate (NDSEG) Fellowship Program.
}
}
}

\begin{document}

\maketitle

\begin{abstract}
Conflict-Based Search (CBS) is a leading exact algorithm for Multi-Agent Path Finding (MAPF), but its high-level node-selection rule is usually treated as a fixed implementation detail. 
Standard best-first selection is strong for minimizing expanded nodes and closing the optimality certificate, yet it can maintain a large frontier, interrupt parent-child expansion sequences, and provide no feasible incumbent until termination. 
This paper studies node selection as a first-class design choice for exact CBS. 
We introduce \gls{acr:dive}, a policy that is best-bound between dives and depth-oriented within a dive. 
\gls{acr:dive} starts each dive from the current best-bound frontier, follows promising children to exploit parent-child locality, and uses incumbent pruning to limit unproductive excursions.
We formalize CBS node selection through a branch-and-bound view, prove that the traversal policy can be changed without affecting exactness, and analyze the resulting trade-offs among expanded nodes, dive breaks, queue size, and primal-dual bound progress.
The analysis predicts three complementary extremes. 
Best-first search is node efficient, iterative deepening is memory efficient, and \gls{acr:dive} is dive efficient while retaining regular best-bound reanchoring. Experiments on standard MAPF benchmarks support this trade-off map. 
\gls{acr:dive} consistently reduces dive breaks, provides early incumbents with certified gaps, uses substantially less queue memory than best-first search, and benefits from warm starts and simple responsive variants in dense or memory-limited regimes.
\end{abstract}

\section{Introduction}

\IEEEPARstart{C}{oordinating} many robots through a shared environment is a core problem in robotics, with applications ranging from warehouse automation to airport surface operations and other graph-structured routing domains~\cite{MAPFBenchmarks_stern2019,airplaneTaxi_morris2016}. 
\Gls{acr:mapf} captures this problem through a compact abstraction. 
Each agent must move from a start vertex to a goal vertex, agents may not collide, and the objective is to minimize a global cost such as the sum of arrival times. 
The abstraction is simple, but optimal routing is computationally hard~\cite{NPHardProof_yu2023}. 
For robotic systems that require certificates of optimality, predictable resource use, or meaningful behavior under time limits, the internal search policy of the solver becomes as important as its final guarantee. 

\begin{figure}[tb]
\centering
\includegraphics[width=1\linewidth]{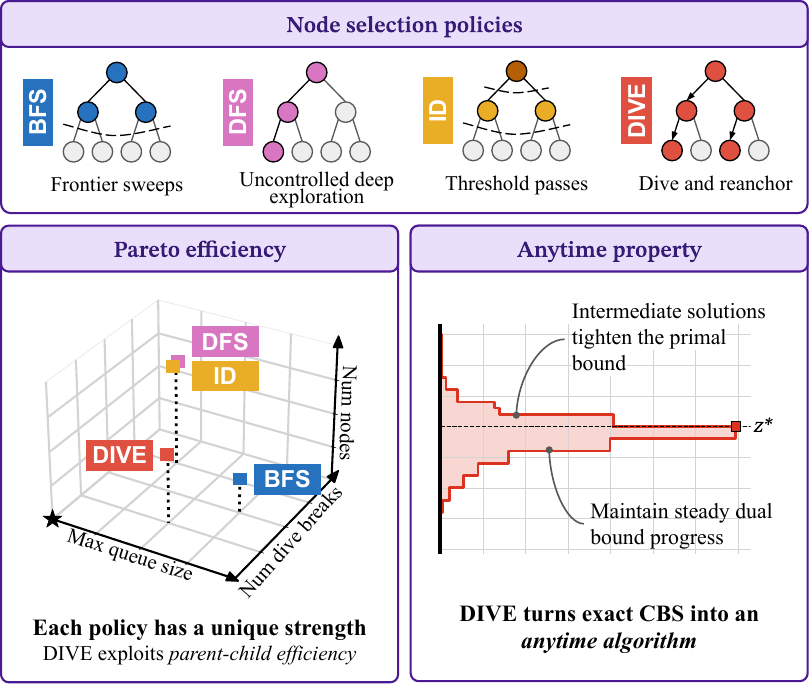}
\caption{Overview of the node-selection perspective developed in this paper. 
Top: \acrshort{acr:bfs} expands the best-bound frontier, \acrshort{acr:dfs} follows deep branches without reanchoring, \acrshort{acr:id} performs repeated threshold-limited passes, and DIVE combines depth-oriented dives with best-bound reanchoring. 
Bottom left: these policies occupy different regions of the multi-objective design space induced by expanded nodes, dive breaks, and queue size. 
Bottom right: by entering the primal zone in a controlled way while repeatedly returning to the best-bound frontier, DIVE can produce intermediate feasible solutions and maintain dual-bound progress, giving exact \acrshort{acr:cbs} an anytime mode.}
\label{fig:teaser}
\end{figure}

\Gls{acr:cbs} is one of the most widely used exact approaches for optimal \gls{acr:mapf}~\cite{OGCBS_sharon2012}, and decomposes the problem into a high-level search over a \gls{acr:ct}. 
Each node $n$ represents a relaxed subproblem with additional constraints on individual agents and a lower-bound key~$\ell(n)$. 
Expanding a node either produces a conflict-free solution or creates child subproblems that forbid one side of a selected conflict. 
The standard policy is to expand the open node with smallest~$\ell(n)$, and this best-first rule is so closely associated with \gls{acr:cbs} that it is often treated as part of the algorithm itself~\cite{OGCBS_sharon2012}.

In this paper, we take a different view.
We treat high-level node selection as a modular policy that shapes the behavior of exact \gls{acr:cbs} without changing the low-level search, branching rule, admissible heuristic, or pruning logic. 
The choice of node-selection policy controls several competing quantities. 
Specifically, it affects how many nodes are expanded, how quickly the dual bound moves toward the optimality certificate, how many open nodes must be stored, how often the search jumps away from a parent-child sequence, and whether feasible solutions appear before the algorithm has certified optimality.
These quantities trade off against one another rather than improving together, and the right balance is set by the robot's deployment context, including onboard memory limits, the real-time window available for replanning, and whether the system must act on a feasible plan before optimality is certified.
\Cref{fig:teaser} previews this multi-objective design space and the distinct operating point that \gls{acr:dive} targets.

The usual best-first policy is excellent for expanded-node count and dual-bound progress, but it has two important drawbacks: i) it can store a large frontier, and ii) it usually provides no incumbent solution until termination. 
It also jumps across the frontier, limiting opportunities to exploit the fact that a child \gls{acr:cbs} node differs from its parent by only one added constraint. 
\gls{acr:idcbs} addresses the frontier-size problem by using repeated depth-first passes under increasing thresholds~\cite{ID_boyarski2020}. 
While this approach greatly reduces explicit queue memory, it re-expands nodes across passes and still does not provide intermediate feasible solutions. 
Most other exact-CBS improvements refine the best-first key rather than changing this broader trade-off profile~\cite{CBSH_Felner2018,AdmissibleHeuisticImprovments_li2019}.

We propose \gls{acr:dive}, a high-level node-selection policy designed to occupy the missing middle ground. 
\gls{acr:dive} is best-bound between dives and depth-oriented within each dive. 
A new dive starts from the globally best-bound open node, and the dive then follows promising children while they can still improve the incumbent. 
When the dive terminates, the search returns to the best-bound frontier. 
This simple structure preserves regular progress toward the optimality certificate, while exposing feasible incumbents earlier, improving parent-child continuity, and keeping a smaller frontier than standard best-first search.

The paper also argues for a broader perspective on node selection in \gls{acr:cbs}. 
Exactness and traversal order should be separated formally. 
Policy quality should then be evaluated as a multi-objective trade-off across expanded nodes, parent-child continuity, memory, incumbent discovery, and certificate progress, rather than through the number of expanded nodes alone.
Because these objectives compete rather than improve together, the policies of interest form a Pareto front rather than admitting a single universally best rule, and \gls{acr:dive} is best understood as a previously missing point on that front.
This view becomes especially consequential when \gls{acr:mapf} serves as a routing subroutine inside a larger optimization or control system, where intermediate trade-offs can directly shape outer-loop decisions before optimality is certified \cite{CBSInLargerSystemExample_brown2020, gaber2026grand}.
It also motivates adaptive node selection, since the right balance between memory, incumbent quality, and certificate progress depends on both the instance and the surrounding operating context.

\paragraph*{Statement of contribution}
The paper makes four contributions. 
First, it formulates high-level \gls{acr:cbs} node selection as an explicit policy and proves, through a branch-and-bound abstraction, that changing this policy preserves exactness when standard pruning conditions are maintained. 
Second, it introduces \gls{acr:dive} and identifies the two mechanisms that make it robust, namely incumbent cutoffs and diversified best-bound restarts. 
Third, it develops a search-tree framework that compares best-first search, depth-first search, iterative deepening, and \gls{acr:dive} across expanded nodes, dive breaks, queue size, and primal-dual bound progress. 
Fourth, it evaluates these predictions on standard \gls{acr:mapf} benchmarks and ablations, showing that \gls{acr:dive} provides a distinct Pareto point with far fewer dive breaks, substantially smaller queues than best-first search, and early incumbents with certified gaps.

\paragraph*{Paper structure}
\Cref{sec:BackgroundAndRelatedWork} fixes the \gls{acr:mapf} and \gls{acr:cbs} notation and positions node selection relative to existing exact-\gls{acr:cbs} variants and \gls{acr:bnb}.
\Cref{sec:Dive} introduces node selection as a modular policy and defines \gls{acr:dive} together with its two governing mechanisms, incumbent cutoffs and diversified best-bound restarts.
\Cref{sec:FormalizedNodeSelection} develops the node-selection framework, i.e., it proves that traversal order is independent of exactness, defines the search-tree model and the five competing objectives, and introduces responsive node selection.
\Cref{sec:TechnicalPolicyAnalysis} uses this framework to compare \gls{acr:bfs}, \gls{acr:dfs}, \gls{acr:id}, and \gls{acr:dive}, locating the structural extreme that each policy occupies.
Finally, \Cref{sec:Experiments} validates the predicted trade-offs on standard \gls{acr:mapf} benchmarks, and \cref{sec:Conclusions} concludes.

\section{Background and Related Work}
\label{sec:BackgroundAndRelatedWork}

This section fixes the \gls{acr:mapf} notation used throughout the paper, reviews \gls{acr:cbs}, and places our node-selection focus relative to existing \gls{acr:cbs} variants and branch-and-bound methods.

\subsection{MAPF Formalization}
\label{sec:Background_MAPFFormalization}

We consider optimal \gls{acr:mapf} on a directed reflexive graph, where the reflexive edges represent wait actions. 
A \gls{acr:mapf} instance specifies a start and a goal vertex for each agent. 
The task is to compute one trajectory per agent such that agents do not collide and the sum of costs is minimized~\cite{MAPFBenchmarks_stern2019,li2025fico}.

\begin{definition}[MAPF instance]
A \emph{\gls{acr:mapf} instance} is a tuple~$\mathcal{I}=(G,A,\rho_s,\rho_g)$, where~$G=\tup{V,E}$ is a directed reflexive graph,~$A$ is a finite set of agents, and~$\rho_s,\rho_g:A\to V$ map each agent to its start and goal vertex, respectively.
\end{definition}

\begin{definition}[Trajectory and cost]
For an agent~$a\in A$, a \emph{trajectory} is a finite sequence~$\tau_a=\tup{v^a_0,\ldots,v^a_{T_a}}$, such that~$v^a_0=\rho_s(a)$,~$v^a_{T_a}=\rho_g(a)$, and~$(v^a_t,v^a_{t+1})\in E$ for all $t=0,\ldots,T_a-1$. 
After reaching its goal, the agent is assumed to remain there, so~$v^a_t:=v^a_{T_a}$ for all~$t>T_a$. The \emph{cost} of~$\tau_a$ is its arrival time~$c(\tau_a):=T_a.$
\end{definition}

\begin{definition}[Conflict]
Given two agents~$a,b\in A$ and time~$t\geq 0$, a \emph{vertex conflict} occurs if~$v^a_t=v^b_t$. 
An \emph{edge conflict} occurs if~$v^a_t=v^b_{t+1}$ and~$v^b_t=v^a_{t+1}$.
\end{definition}

\begin{definition}[MAPF solution]
\label{def:mapfSolution}
A \emph{\gls{acr:mapf} solution} is a set of trajectories~$\{\tau_a\}_{a\in A}$ with no vertex or edge conflicts. 
Throughout the paper, we consider the standard sum-of-costs objective
\begin{equation*}
    \soc\big(\{\tau_a\}_{a\in A}\big)=\sum_{a\in A} c(\tau_a).
\end{equation*}
\end{definition}

A set of individually valid trajectories that still contains conflicts will be referred to as a \emph{joint plan}. 

This is the relaxed object manipulated by \gls{acr:cbs} before outputting conflict-free optimal trajectories.

\subsection{Conflict-Based Search}

\gls{acr:cbs} solves \gls{acr:mapf} through a high-level search over a \gls{acr:ct}~\cite{OGCBS_sharon2012}. 
Each \gls{acr:ct} node~$n$ stores a set of agent-specific space-time constraints, one trajectory per agent satisfying those constraints, and a lower bound key~$\ell(n)$ on the best conflict-free objective attainable below that node.
The root node has no additional constraints.
In vanilla CBS,~$\ell(n)$ is the sum of the individually optimal path costs at~$n$.

To expand a \gls{acr:ct} node, a low-level search computes optimal single-agent trajectories consistent with that node's constraints, typically using A*. 
If the resulting joint plan is conflict-free, the node yields a feasible \gls{acr:mapf} solution. 
Otherwise, \gls{acr:cbs} selects one conflict and branches on it.
For a vertex conflict~$\tup{a,b,v,t}$, the two children respectively add constraints forbidding agent~$a$ or agent~$b$ from occupying vertex~$v$ at time~$t$. 
For an edge conflict, the two children analogously forbid one of the two conflicting traversals. 

A large body of work improves \gls{acr:cbs} through better conflict selection, bypassing, admissible high-level heuristics, finite-horizon variants, and bounded-suboptimal variants~\cite{CBSH_Felner2018,AdmissibleHeuisticImprovments_li2019, ICBS_boyarski2015,suboptCBS_barer2021,EECBS_li2020,ACCBS_li26,li2026certificate}. 
\gls{acr:cbs} has also been extended to settings that couple routing with other decisions, including configurable environments and task allocation~\cite{CMAPF_bellusci2020,taskCBS_zhou2026}. 
These developments change how nodes are evaluated, generated, or solved. 
The focus of this paper is orthogonal. 
We study which open high-level node should be expanded next.

\subsection{Node Selection in CBS}

\subsubsection{Best-first search and admissible high-level heuristics}
Standard optimal \gls{acr:cbs} uses \gls{acr:bfs}, selecting an open node with minimum lower-bound key~$\ell(n)$~\cite{OGCBS_sharon2012}. 
This choice is natural for exact search because it prioritizes nodes that can still improve the optimality certificate. 
It also supports the classical \gls{acr:cbs} proof, where the first conflict-free node selected by best-bound order is optimal \cite{OGCBS_sharon2012}.

Many exact-CBS improvements strengthen this best-first paradigm rather than replacing it. 
Admissible high-level heuristics tighten the key~$\ell(n)$ beyond the raw sum of individual path costs~\cite{CBSH_Felner2018}. 
Subsequent heuristics based on conflict structure have produced major practical gains~\cite{AdmissibleHeuisticImprovments_li2019}. 
These methods improve the quality of the best-first key, but the search remains fundamentally best-bound. 
As a result, they do not directly address the memory, parent-child locality, or incumbent-discovery issues studied here.

\subsubsection{Parent-child efficiency and iterative deepening}
A child \gls{acr:cbs} node differs from its parent by a single additional constraint, and only the newly constrained agent must be replanned. 
Importantly, this creates an opportunity for parent-child efficiency. 
If a child is processed immediately after its parent, implementation state from the parent can often be reused more effectively than when the search jumps to an unrelated node~\cite{ID_boyarski2020}. 
Node selection therefore affects runtime not only through the number of expanded nodes, but also through the continuity of parent-child expansion chains.

\gls{acr:idcbs} is the main exact-\gls{acr:cbs} alternative to standard best-first node selection~\cite{ID_boyarski2020}. 
It performs repeated depth-first passes under a nondecreasing cost threshold and restarts from the root when the threshold increases. 
This gives a small explicit frontier and preserves depth-first behavior within each pass. 
However, the policy discovers \emph{new} nodes in best-first order, performs repeated work across thresholds, and does not seek feasible incumbents before the threshold reaches them. 
Our analysis and experiments compare \gls{acr:dive} directly with this memory-oriented baseline.

\subsubsection{Suboptimal variants and learned node selection}
\label{sec:PreviousWork_SuboptimalVariants}
Node selection has also been studied in suboptimal \gls{acr:cbs} solvers. 
Greedy \gls{acr:cbs} selects nodes using non-admissible measures such as the number of conflicts, while bounded \gls{acr:cbs} and its descendants use focal search to explore promising nodes within a bounded suboptimality range~\cite{suboptCBS_barer2021,EECBS_li2020}. 
Learning-based node selection has been explored in bounded-suboptimal \gls{acr:cbs} as well~\cite{MLNodeSlctnEECBS_huang2021}. 
These methods are important, but they answer a different question. 
They trade exact best-bound ordering for speed or bounded-suboptimal performance. 
This paper asks how much freedom remains in node selection while preserving exact optimal \gls{acr:cbs}.

Fast suboptimal \gls{acr:mapf} methods are also relevant as sources of feasible warm starts. 
In our experiments, we use \gls{acr:pibt} to initialize an incumbent in dense regimes~\cite{OGPIBT_okumura2019}. 
The warm-start solution does not replace exact search. 
It only provides an initial primal bound that can activate valid incumbent pruning.

\subsection{Connection to Branch and Bound}
\gls{acr:cbs} has the same high-level structure as a \gls{acr:bnb} algorithm~\cite{B&BSurvey_lawler1966}. 
A node represents a relaxed subproblem, the node key is a lower bound, branching refines the relaxation, and an incumbent gives a primal bound. 
This connection is useful because node selection has long been treated as a major design dimension in branch-and-bound solvers \cite{SCIPThesis_achterberg2007}.

The contrast with \gls{acr:mip} is especially instructive. 
\gls{acr:mip} solvers routinely mix best-bound, depth-oriented, and responsive strategies, and they expose parameters that shift effort between finding feasible solutions and proving optimality~\cite{SCIPThesis_achterberg2007, 
gurobi_mipfocus_doc, cplex_mipemphasis_doc}. 
Hybrid best-first strategies with plunging have also been studied outside \gls{acr:mapf}, for example, in weighted constraint-satisfaction problems~\cite{HBFS_allouche2015}. 
\gls{acr:dive} adapts this general idea to exact \gls{acr:cbs}, where parent-child locality, incumbent discovery, and certificate progress have a \gls{acr:mapf}-specific interpretation.

\section{DIVE}
\label{sec:Dive}
We have identified two dominant high-level traversal regimes for \gls{acr:cbs}. 
\gls{acr:bfs} expands nodes in best-bound order, which is favorable for expanded-node count and dual-bound progress, but it tends to maintain a large frontier and provides no incumbent before termination. 
\gls{acr:id} greatly reduces the frontier size by using repeated depth-first passes, but it pays for this memory reduction through repeated work and still does not provide intermediate feasible solutions. 
\gls{acr:dive} is designed to occupy the missing middle ground: it preserves the best-bound frontier between dives while exploiting parent-child continuity and exposing feasible incumbents during the solve.

This section first introduces node selection as a modular component of \gls{acr:cbs}.
It then defines the \gls{acr:dive} policy and summarizes the mechanisms responsible for its trade-off profile.
The formal correctness argument and the technical comparison with existing policies are deferred to \cref{sec:FormalizedNodeSelection,sec:TechnicalPolicyAnalysis}.

\subsection{Node Selection as a Modular Policy}

\begin{algorithm}[tb]
\begin{minipage}{\linewidth}
\caption{\gls{acr:cbs} with node-selection policy $\pi$}
{\small{Blue lines collect information that is specific to DIVE.}}
\smallskip
\label{alg:CBS}
\begin{algorithmic}[1]
\State $\mathrm{OPEN} \gets \{n_0\}$, initialize queue with the root node
\State \textcolor{RoyalBlue}{$z^P \gets \infty$ or warm-start value \Comment{e.g. from \gls{acr:pibt} solution}}
\State $n^{\mathrm{inc}} \gets$ \textbf{none} \textcolor{RoyalBlue}{or a warm-start solution}
\State \textcolor{RoyalBlue}{$n^{\mathrm{last}}\gets\varnothing$}
\While{$\mathrm{OPEN}\ne\varnothing$}
    \State $n\gets \pi(\mathrm{OPEN}\textcolor{RoyalBlue}{,n^{\mathrm{last}},z^{\mathrm P}})$\footnotemark  \Comment{$\pi$: node selection policy.}

    \State remove $n$ from $\mathrm{OPEN}$
    \State \textcolor{RoyalBlue}{$n^{\mathrm{last}}\gets n$}
    \If{$n$ is infeasible \textcolor{RoyalBlue}{\textbf{or} $\ell(n)\ge z^{\mathrm P}$}}
        \State \textbf{continue}
    \EndIf
    \If{$n$ is conflict-free}
        \State \textcolor{RoyalBlue}{$z^{\mathrm P}\gets \ell(n)$}
        \State $n^{\mathrm{inc}}\gets n$
        \State \textbf{continue}
    \EndIf
    \State branch on a selected conflict \& generate children of $n$
    \State add each feasible child $c$ \textcolor{RoyalBlue}{with $\ell(c)<z^{\mathrm P}$} to $\mathrm{OPEN}$
\EndWhile
\State \Return $n^{\mathrm{inc}}$ \Comment{returning $\varnothing$ certifies infeasibility}
\end{algorithmic}
\end{minipage}
\end{algorithm}

A \gls{acr:cbs} solve maintains a set of open constraint-tree nodes, denoted by~$\mathrm{OPEN}$.
At each high-level iteration, the solver removes one node from~$\mathrm{OPEN}$, processes it, and either discards it or replaces it with its expanded children.
In standard optimal \gls{acr:cbs}, this node is selected by best-bound order: the next expanded node has minimum lower-bound key~$\ell(n)$ among the open nodes.
\gls{acr:idcbs} realizes the same best-bound discovery order indirectly, through repeated depth-first searches under increasing cost thresholds~\cite{ID_boyarski2020}.

We separate this traversal choice from the rest of the algorithm.
A high-level node-selection policy~$\pi$ maps the current observable solver state to a node in~$\mathrm{OPEN}$.
\Cref{alg:CBS} shows \gls{acr:cbs} with this policy made explicit. 
The low-level search, conflict selection, branching rule, admissible high-level heuristic, and pruning tests are unchanged; only the order in which open \gls{acr:ct} nodes are processed is delegated to $\pi$.

This relaxation is algorithmically safe under the usual finite-search assumption stated formally in \cref{sec:FormalizedNodeSelection_Correctness}: any policy~$\pi$ that always selects an open node changes only the order in which subproblems are processed. 
The choice of~$\pi$ is nevertheless consequential because it determines the traversal pattern of the \gls{acr:ct}, and therefore the trade-off among expanded nodes, parent-child continuity, frontier size, incumbent discovery, and dual-bound progress.

\begin{remark}[machine learning-driven node selection]
    Our framework of node selection as modular yet algorithmically safe makes it naturally suited for machine-learning driven approaches, particularly with cutoffs as in \cref{prop:DIVE_incumbent_cutoff} as a safety guard.
    While this paper remains focused on analyzing the node selection tradeoff space through its node count, diving, and queue size extremes, other methods may exploit the rest of the design space.
    Policies can be trained online on previous iterations of the same solve, offline on instances within the same distribution, or as part of a reinforcement learning scheme.
    Such approaches are well-studied in the \gls{acr:mip} context \cite{OGMLNodeSelection_he2014,RLNodeSelection_mattick2024,GNNNodeSelection_labassi2022}.
    In CBS, machine learning has been confined to conflict selection \cite{MLConflictSelection_huang2020} and suboptimal variants \cite{MLNodeSlctnEECBS_huang2021}, but detaching node selection from correctness expands this scope.
\end{remark}

\footnotetext{This signature is specific to DIVE. Policies may use any observable information as input.}

\subsection{\texorpdfstring{\gls{acr:dive}}{DIVE} Policy}
\label{sec:Dive_Algorithm}
\gls{acr:dive} is a high-level node-selection policy for optimal \gls{acr:cbs}.

\begin{definition}[Dive]
Given a node expansion sequence~$\tup{n_0,n_1,\ldots}$, a \emph{dive} is a maximal contiguous subsequence~$\tup{n_i,\ldots,n_j}$ such that~$n_{t+1}$ is a child of~$n_t$ for every~$t=i,\ldots,j-1$. 
The first node of the subsequence is referred to as the \emph{dive root}.
\end{definition}

\gls{acr:dive} is \emph{best-bound between dives} and \emph{depth-oriented within a dive}.
Each dive starts from the current best-bound node in the open node queue $\mathrm{OPEN}$, then repeatedly follows the best child in order to exploit parent-child locality. 
The dive is terminated as soon as continuing it becomes unproductive: Either the current node is conflict-free, no feasible child exists, or the node lower bound~$\ell(n)$ can no longer improve the incumbent solution.

\begin{definition}[Incumbent and primal bound]
At iteration~$t$, an \emph{incumbent} is a conflict-free solution found by the search whose objective value is smallest among all conflict-free solutions found up to that iteration. 
Its objective value $z^{\mathrm P}_t$ is used as a \emph{primal bound}.
\begin{equation*}
    z^{\mathrm P}_t =
    \begin{cases}
        \ell(n^{\mathrm{inc}}_t), & \text{if an incumbent } n^{\mathrm{inc}}_t \text{ exists},\\
        +\infty, & \text{otherwise.}
    \end{cases}
\end{equation*}
\end{definition}

\begin{figure*}[tb]
    \centering
    \includegraphics[width=0.85\linewidth]{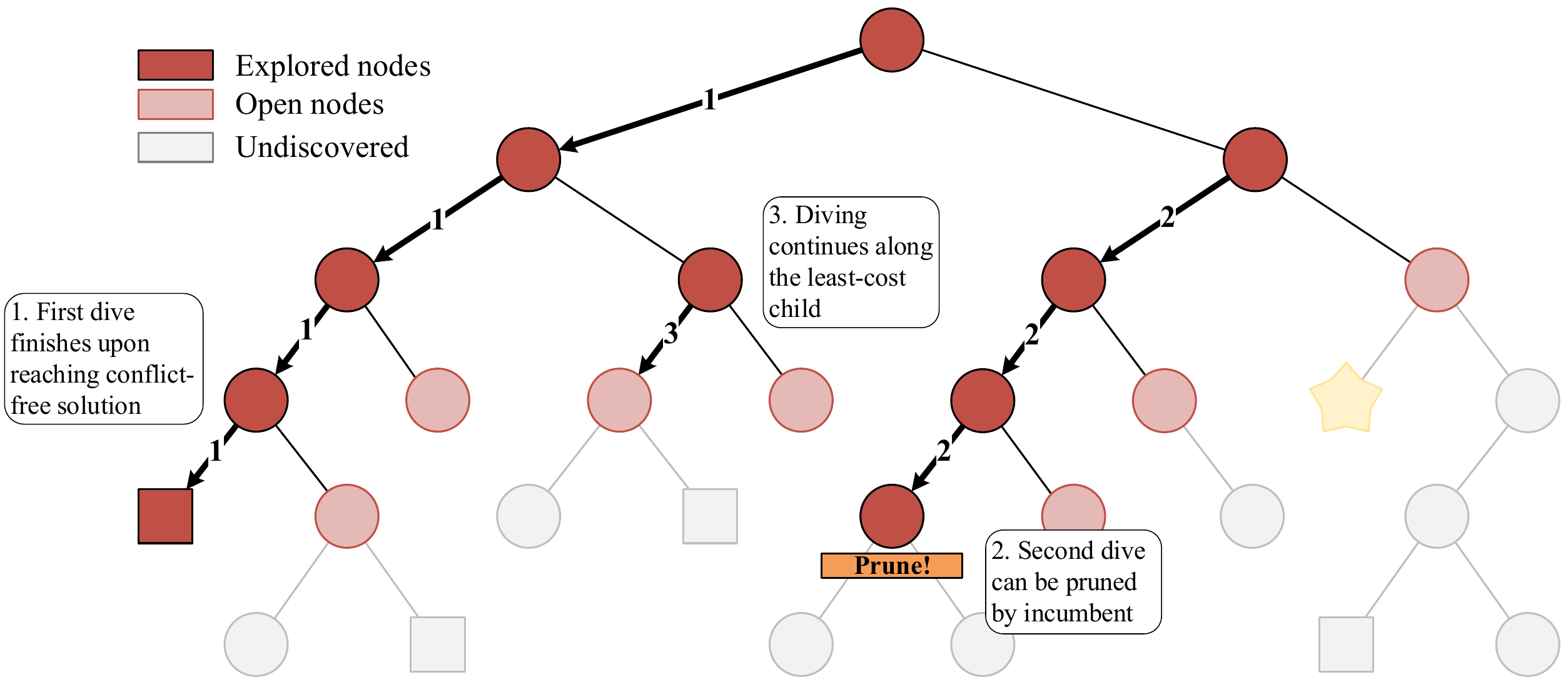}
    \caption{Example \gls{acr:dive} traversal on a \gls{acr:ct}. 
    Parents and left siblings have smaller lower-bound $\ell(n)$. 
    Squares denote conflict-free nodes, and the star denotes the best conflict-free node. 
    Edges with the same label are traversed within the same dive; labels indicate the order of the dives.}
    \label{fig:search_tree_diagram}
\end{figure*}

\begin{algorithm}[t]
\caption{DIVE node-selection policy}
\label{alg:DIVE}
\begin{algorithmic}[1]

\Function{$\pi_{\mathrm{DIVE}}$}{$\mathrm{OPEN}, n^{\mathrm{last}}, z^{\mathrm P}$}
    \If{$n^{\mathrm{last}}=\varnothing$}
        \State \Return $\pi_{\mathrm{BFS}}(\mathrm{OPEN})$
    \EndIf
    \State $\mathcal C \gets $ set of children of $n^{\mathrm{last}}$  %
    \If{$\mathcal C=\varnothing$}
        \State \Return $\pi_{\mathrm{BFS}}(\mathrm{OPEN})$
    \EndIf
    \State $c^{\min} \gets \argmin_{c\in\mathcal C}\ell(c)$ \Comment{with fixed tie-breaking}
    \If{$\ell(c^{\min}) < z^{\mathrm P}$}
        \State \Return $c^{\min}$ \Comment{continue the current dive}
    \EndIf
    \State \Return $\pi_{\mathrm{BFS}}(\mathrm{OPEN})$ \Comment{start a new dive}
\EndFunction
\Statex
\Function{$\pi_{\mathrm{BFS}}$}{$\mathrm{OPEN}$}
    \State \Return $\argmin_{n\in\mathrm{OPEN}}\ell(n)$ \Comment{with fixed tie-breaking}
\EndFunction

\end{algorithmic}
\end{algorithm}

\Cref{alg:DIVE} states the policy and \Cref{fig:search_tree_diagram} shows a graphical example. 
Let~$n^{\mathrm{last}}$ be the node expanded in the previous high-level iteration. 
If one of its children remains in~$\mathrm{OPEN}$ and the lowest-bound such child has~$\ell(c)<z^{\mathrm P}$, \gls{acr:dive} expands that child next. Otherwise, \gls{acr:dive} starts a new dive by selecting the best-bound node in $\mathrm{OPEN}$. 
Ties among equal-bound nodes are resolved by a fixed deterministic rule\footnote{In our implementation, deeper nodes are preferred first, followed by nodes whose selected conflict has larger cardinality. The definition of \gls{acr:dive} does not depend on this particular tie-breaking rule.}. 

This policy deliberately differs from both main baselines. 
Relative to \gls{acr:bfs}, \gls{acr:dive} is willing to leave best-bound order temporarily in order to obtain longer parent-child expansion chains and earlier feasible solutions. 
Relative to \gls{acr:id}, \gls{acr:dive} does not restart from the root after each threshold; it keeps the open frontier and reanchors each new dive by a global best-bound selection. 
Similar ideas outside \gls{acr:mapf} are known as Hybrid \gls{acr:bfs} or best-bound search with plunging~\cite{HBFS_allouche2015,SCIPThesis_achterberg2007}.

\subsection{Mechanisms and Trade-offs}
\label{sec:Dive_StructuralProperties}

\gls{acr:dive}'s behavior is governed by two mechanisms: incumbent cutoff and diversified search. Together, they allow \gls{acr:dive} to enter deeper parts of the \gls{acr:ct} without turning into plain \gls{acr:dfs}.

\subsubsection{Incumbent cutoff}
\gls{acr:dive} may expand nodes whose lower bound is larger than the optimal objective. 
Such nodes are not required for the final optimality certificate, but they can produce feasible incumbents early and extend parent-child expansion chains. 
The risk is that a depth-oriented policy may spend too long in regions that cannot improve the best solution already found. 
Incumbent pruning controls this risk.

\begin{proposition}[Incumbent cutoff]
\label{prop:DIVE_incumbent_cutoff}
Let $z^{\mathrm P}_t$ be the current incumbent after $t$ iterations. If an open node $n$ satisfies $\ell(n)\ge z^{\mathrm P}_t$, then no descendant of $n$ can improve the incumbent.
\end{proposition}

\begin{proof}
The value~$\ell(n)$ is a lower bound on every feasible solution in the subproblem represented by $n$. 
Every feasible descendant of $n$ is also feasible within that subproblem, and therefore has objective value at least $\ell(n)\ge z^{\mathrm P}_t$. 
Such a descendant cannot have objective value strictly smaller than the current incumbent.
\end{proof}

The cutoff is a valid branch-and-bound pruning test and is not unique to \gls{acr:dive}. 
It is central to \gls{acr:dive} because \gls{acr:dive} intentionally follows child nodes before the global best-bound frontier has been exhausted. 
In \cref{alg:DIVE}, if the best available child of the previously expanded node already satisfies $\ell(c)\ge z^{\mathrm P}_t$, then every sibling child has lower bound at least $z^{\mathrm P}_t$ as well, so the current dive is stopped.

\subsubsection{Diversified search}
Stopping a dive is useful only if the next dive is chosen carefully. 
\gls{acr:dive} reanchors each new dive at the global best-bound frontier rather than continuing near the branch that happened to be explored most recently.

\begin{definition}[Diversified diving policy]
A diving policy is \emph{$\sigma$-diversified} if every dive root is selected by a global rule $\sigma(\mathrm{OPEN})$ over the open set, with $\sigma$ independent of traversal history.
\gls{acr:dive} is best-bound-diversified, meaning that $\sigma(\mathrm{OPEN})\in\argmin_{n\in\mathrm{OPEN}}\ell(n)$ up to tie-breaking.
\end{definition}

Best-bound diversification separates \gls{acr:dive} from \gls{acr:dfs} with incumbent pruning.
A poor child choice in plain \gls{acr:dfs} can be followed by many additional depth-first choices in the same unproductive region.
In \gls{acr:dive}, the effect of such a choice is limited to one dive: once the dive stops, the next dive root is selected from the global best-bound frontier. 
This is the mechanism that preserves regular progress toward the optimality certificate while still allowing incumbent-seeking excursions.

\subsubsection{Frontier size and parent-child continuity}
\gls{acr:dive} keeps substantially more frontier information than \gls{acr:id}, because it does not discard the open nodes left behind by a dive.
In binary \gls{acr:cbs}, each internal node visited during a dive can leave at most one unchosen child in $\mathrm{OPEN}$. 
This is the price of reanchoring future dives at the best-bound frontier. 
Conversely, \gls{acr:dive} usually keeps a smaller frontier than \gls{acr:bfs}, because it does not expand the \gls{acr:ct} layer by layer before moving deeper.

The same traversal pattern improves parent-child continuity. 
Whenever \gls{acr:dive} continues a dive, the next high-level expansion is a child of the previous one. 
Implementations that exploit data between a parent and child can therefore process these consecutive expansions more efficiently than unrelated jumps across a wide frontier. 
The formal analysis in \cref{sec:TechnicalPolicyAnalysis} makes this tradeoff explicit: \gls{acr:bfs} is optimized for expanded-node count, \gls{acr:id} for queue size, and \gls{acr:dive} for dive continuity.

\subsection{Bounds, Anytime Behavior, and Warm Starts}
\label{sec:Dive_BoundBehavior}

\gls{acr:dive} can find feasible solutions before the optimality proof is complete. 
After an incumbent has been found, the solver can be interrupted and still return the best feasible solution found so far. 
This gives exact \gls{acr:cbs} an anytime mode: the returned solution is feasible at interruption time, and it becomes certified optimal only when the remaining open nodes can no longer contain a better solution.

\gls{acr:dive} does not obtain this anytime behavior by neglecting the lower bound. 
Since each new dive starts from the best-bound frontier, the search continues to raise the certificate side of the solve. 
The active dual bound can be written as
\begin{equation*}
    z^D_t = \min_{n\in \mathrm{OPEN}_t} \ell(n),
\end{equation*}

When an incumbent exists,~$z^{\mathrm D}_t\le z^*\le z^{\mathrm P}_t$. 
If $z^D_t= z^{\mathrm P}_t$, then the incumbent is certified optimal. 
For nonzero incumbent cost, the relative primal-dual gap is therefore
\begin{equation*}
    \mathrm{Gap}_t = \frac{z^{\mathrm P}_t-z^{\mathrm D}_t}{z^{\mathrm P}_t},
    \qquad 0<z^{\mathrm P}_t<+\infty .
\end{equation*}

The zero-cost degenerate case can be handled by reporting zero gap once $z^{\mathrm P}_t=z^{\mathrm D}_t=0$. 
The gap is not a runtime predictor, but it gives an instance-specific, scale-normalized certificate of how much improvement remains possible.

\gls{acr:dive} also interfaces naturally with warm starts.

\begin{definition}[Warm start]
A \emph{warm start} is any feasible \gls{acr:mapf} solution supplied before the \gls{acr:cbs} solve begins and used to initialize $n^{\mathrm{inc}}$ and $z^{\mathrm P}_t$.
\end{definition}

A warm start activates incumbent pruning immediately. 
This creates a simple interface between fast suboptimal planners and exact search: a heuristic planner can provide a feasible solution, while \gls{acr:dive} continues searching only for a proof of optimality or a better incumbent. 
The resulting solver remains exact because the warm start changes only the initial primal bound; it does not remove any node that could contain a solution better than that bound.

\section{Node-Selection Framework}
\label{sec:FormalizedNodeSelection}

\Cref{sec:Dive} made high-level node selection explicit in \gls{acr:cbs}. 
This section formalizes what that choice can and cannot change. 
The central point is simple: under standard \gls{acr:bnb} conditions, node selection does not affect exactness, but it strongly affects the route by which exactness is obtained.

We organize the discussion around three principles. 
First, node selection can be studied formally rather than only empirically. 
Second, it is inherently multi-objective: expanded nodes, parent-child continuity, memory, incumbent discovery, and dual-bound progress are distinct quantities. 
Third, the preferred trade-off is context dependent, so useful policies may dynamically respond to the state of the ongoing solve.
\Cref{sec:TechnicalPolicyAnalysis} then uses the notation introduced here to compare \gls{acr:bfs}, \gls{acr:dfs}, \gls{acr:id}, and \gls{acr:dive}.

\subsection{Correctness}
\label{sec:FormalizedNodeSelection_Correctness}

We first separate correctness from traversal order by abstracting high-level \gls{acr:cbs} as an instance of \gls{acr:bnb}. 
This abstraction serves to capture not just the standard scope of \gls{acr:cbs}, but also its many extensions, such as A-CBS and TA-CBS which capture map configuration and task assignment respectively in addition to agent routing~\cite{CMAPF_bellusci2020,taskCBS_zhou2026}. Other adaptations, such as suboptimal and finite-horizon variants \cite{ACCBS_li26,li2026certificate} are likewise captured under the same \gls{acr:bnb} scheme.
Consider a minimization problem with objective~$g$ and feasible set~$Q$ contained in a relaxation~$O$. 
A \gls{acr:bnb} node~$P\subseteq O$ represents a relaxed subproblem. 
Solving~$P$ either proves it infeasible or returns a relaxed optimizer~$s(P)$ together with a valid lower bound~$\ell(P)$ satisfying
\begin{equation}
    \ell(P) \leq g(q) \qquad \forall q\in P\cap Q .
    \label{eq:valid_lower_bound}
\end{equation}

If $s(P)\in Q$, the node is pruned because the relaxation optimum is already feasible and therefore no point in~$P\cap Q$ can improve on $s(P)$. 
If~$s(P)\notin Q$, branching creates children~$L$ and~$R$ satisfying
\begin{align}
    L,R &\subseteq P, \label{eq:bb_subset}\\
    s(P) &\notin L\cup R, \label{eq:bb_exclude}\\
    P\cap Q &\subseteq L\cup R. \label{eq:bb_cover}
\end{align}
The first condition makes children refinements of the parent, the second excludes the current infeasible relaxed optimizer, and the third preserves every feasible solution that could still be hidden inside the parent. 
The children need not be disjoint.

\begin{algorithm}[t]
\caption{Generic \gls{acr:bnb} with node-selection policy~$\pi$}
\label{alg:B&B}
\begin{algorithmic}[1]
\State $\mathrm{OPEN} \gets \{O\}$
\State $s^{\mathrm{inc}} \gets \emptyset$, \quad $z^{\mathrm P} \gets +\infty$
\While{$\mathrm{OPEN} \neq \emptyset$}
    \State $P \gets \pi(S)$, where $P\in\mathrm{OPEN}$
    \State Remove $P$ from $\mathrm{OPEN}$
    \State Solve $P$ to obtain status, lower bound $\ell(P)$, and relaxed optimizer $s(P)$ if one exists
    \If{$P$ is infeasible}
        \State \textbf{continue}
    \EndIf
    \If{$\ell(P) \geq z^{\mathrm P}$}
        \State \textbf{continue} \Comment{incumbent cutoff}
    \EndIf
    \If{$s(P)\in Q$}
        \If{$g(s(P)) < z^{\mathrm P}$}
            \State $s^{\mathrm{inc}} \gets s(P)$
            \State $z^{\mathrm P} \gets g(s(P))$
        \EndIf
        \State \textbf{continue}
    \EndIf
    \State Branch on a violation in $s(P)$ and create children $L,R$
    \State such that $L,R\subseteq P$, $s(P)\notin L\cup R$, and $P\cap Q\subseteq L\cup R$
    \State Add $L$ and $R$ to $\mathrm{OPEN}$
\EndWhile
\State \Return $s^{\mathrm{inc}}$
\end{algorithmic}
\end{algorithm}

\begin{proposition}[Node selection does not affect correctness]
\label{prop:bnb_nodesel}
Assume the search tree induced by the branching rule is finite, every node lower bound satisfies \eqref{eq:valid_lower_bound}, and every branch operation satisfies \eqref{eq:bb_subset}--\eqref{eq:bb_cover}. 
Then \cref{alg:B&B} returns an optimal solution, or correctly certifies infeasibility, for any policy that always selects a node from $\mathrm{OPEN}$ (completeness).
\end{proposition}

\begin{proof}[Proof sketch]
Node selection changes only the order in which open subproblems are processed.
Incumbent pruning is valid by \eqref{eq:valid_lower_bound}. 
Infeasible nodes contain no feasible solution. 
Feasible relaxed optimizers can be fathomed because they are already optimal within their subproblem. 
Branching preserves all feasible solutions by \eqref{eq:bb_cover}. 
Thus, after each iteration, every feasible solution that could still improve the incumbent remains covered by the open nodes. 
Since the induced tree is finite, every open node is eventually processed or validly pruned. 
At termination no improving feasible solution remains; hence the incumbent is optimal, and if no incumbent exists the instance is infeasible. 
\cref{Appendix} gives the full invariant proof.
\end{proof}

\begin{remark}[Finite \gls{acr:mapf} search space]
\cref{prop:bnb_nodesel} is stated for a finite induced search tree, which is also the setting used for the structural analysis in \cref{sec:TechnicalPolicyAnalysis}.
This assumption is not a claim that every unconstrained \gls{acr:mapf} formulation has a finite high-level tree. 
Rather, it isolates the node-selection question from separate completeness issues such as horizons and objective bounds. 
Once a finite incumbent value~$B$ is available in sum-of-costs \gls{acr:mapf}, any strictly improving solution has total cost less than~$B$, so the relevant bounded part of the search is finite. 
Practical solvers can also impose explicit horizons or rely on standard CBS completeness conditions; these choices are orthogonal to the node-selection order.
\end{remark}

\gls{acr:cbs} satisfies the above abstraction directly.

\begin{proposition}[\gls{acr:cbs} instantiates \gls{acr:bnb}]
    \label{prop:cbs_instance_bnb}
The high-level search of optimal \gls{acr:cbs} is an instance of the \gls{acr:bnb} scheme in~\cref{alg:B&B}.
\end{proposition}
\begin{proof}[Proof sketch]
Let $Q$ be the set of conflict-free joint plans for the \gls{acr:mapf} instance, and let $O$ be the relaxation that keeps individual path feasibility but drops inter-agent collision constraints. 
A \gls{acr:ct} node $P$ is defined by a set of agent-specific space-time constraints. 
Solving $P$ computes an individually optimal path for each agent subject to those constraints; the resulting node key $\ell(P)$ is a lower bound on every conflict-free joint plan satisfying the same constraints.

If the relaxed joint plan is conflict-free, it is feasible for the original \gls{acr:mapf} instance and is optimal within $P\cap Q$.
Otherwise, \gls{acr:cbs} branches on a selected conflict. 
For a vertex conflict, one child forbids the first agent from occupying the conflicted vertex at the conflicted time, and the other child forbids the second agent from doing so. 
Edge conflicts are handled analogously. 
These children are subsets of the parent, exclude the current conflicting joint plan, and preserve every conflict-free joint plan in the parent because any conflict-free plan must avoid at least one side of the selected conflict. 
Thus \eqref{eq:bb_subset}--\eqref{eq:bb_cover} hold.
\end{proof}

\begin{proposition}[Exactness of \gls{acr:dive}]
\label{prop:DIVE_exact}
\gls{acr:dive} is an exact high-level node-selection policy for optimal \gls{acr:cbs}. 
With valid incumbent pruning, it returns an optimal solution whenever one exists in the finite induced search tree.
\end{proposition}
\begin{proof}
\gls{acr:dive} only changes which node in $\mathrm{OPEN}$ is selected next.
Its pruning test is the incumbent cutoff of \cref{prop:DIVE_incumbent_cutoff}, which is exactly the valid lower-bound pruning test used in \cref{alg:B&B}. 
The result follows from \cref{prop:bnb_nodesel,prop:cbs_instance_bnb}.
\end{proof}

Importantly, enhancing \gls{acr:dive} with warm starts preserves its correctness.

\begin{proposition}[Warm-start monotonicity]
\label{prop:DIVE_warmstart}
Initializing \gls{acr:dive} with any feasible warm start preserves exactness and can only shrink the subset of nodes that survive incumbent pruning.
\end{proposition}

\begin{proof}
A warm start changes only the initial incumbent and primal bound.
Any node pruned by the resulting cutoff has lower bound at least the value of a known feasible solution, so by \cref{prop:DIVE_incumbent_cutoff} it cannot contain a strictly better solution. The branch-and-bound certificate is otherwise unchanged.
\end{proof}

\begin{remark}[Relation to \gls{acr:mip}]
\cref{prop:cbs_instance_bnb} places \gls{acr:cbs} in the same abstract family as classical \gls{acr:bnb} algorithms, including \glspl{acr:mip} solved by linear-programming relaxations and branching on fractional variables~\cite{B&BSurvey_lawler1966}.
We use this connection only as vocabulary and motivation for node selection; the analysis and experiments in this paper remain specific to \gls{acr:cbs}.
\end{remark}

\begin{remark}[Relation to suboptimal \gls{acr:cbs} variants]
    \cref{prop:cbs_instance_bnb,prop:bnb_nodesel} imply that the suboptimal CBS variants discussed in \cref{sec:PreviousWork_SuboptimalVariants} can be adopted to become exact variants by enforcing exact low-level search and restricting their non-admissible node selection to the high-level search.
    Primal zone exploration can be controlled by incumbent pruning following \cref{prop:DIVE_incumbent_cutoff}.
\end{remark}

\subsection{Search Tree}
\label{sec:FormalizedNodeSelection_SearchTree}

For the rest of the paper, fix a \gls{acr:mapf} instance $\mathcal I$ and a deterministic branching policy $\beta$. 
Together they induce a conceptual search tree $\mathbf T^{\mathcal I,\beta}$; different node-selection policies traverse different subsets of this same tree and in different orders. 
We usually omit the superscript when the instance and branching policy are clear.

\begin{definition}[Search tree]
    \label{def:searchTree}
    A \emph{search tree} is a tuple~$\mathbf{T} = \tup{\mathcal T, C, \ell, f}$, where $\mathcal T$ is a finite set of nodes, $C:\mathcal{T} \rightarrow 2^\mathcal{T}$ maps each node to its children,~$\ell:\mathcal T\to\mathbb R\cup\{+\infty\}$ is a valid lower-bound key, and $f:\mathcal T\to\{\top,\bot\}$ indicates whether the relaxed solution stored at the node is conflict-free.
    The tree is rooted, binary, and lower bounds are monotone along parent-child edges:
    \begin{equation}
    \ell(m)\geq \ell(n) \qquad \forall n\in\mathcal T,\; m\in C(n).
    \label{eq:tree_monotone_bound}
\end{equation}
\end{definition}

A node with $f(n)=\top$ is a feasible \gls{acr:mapf} solution. 
Low-level infeasible nodes are terminal and can be represented with $C(n)=\emptyset$. 
We assume throughout the technical analysis that the node key $\ell(n)$ is nondecreasing along every root-to-leaf path.
The monotonicity assumption in \eqref{eq:tree_monotone_bound} holds for the standard sum-of-costs key in vanilla CBS, since adding constraints cannot decrease the optimal relaxed cost.
The results apply to any keys satisfying the same monotonicity condition, as when admissible high-level heuristics are used.

\subsection{Tree Traversal}

Let~$n_t$ denote the node processed at iteration~$t$, with~$n_0$ as the root.
Let~$E_t\subseteq\mathcal T$ be the set of processed nodes after~$t$ iterations and let~$\mathrm{OPEN}_t$ be the open set immediately before iteration $t$.
\begin{align*}
    E_{t+1} &= E_t\cup\{n_t\},\\
    \mathrm{OPEN}_{t+1} &= (\mathrm{OPEN}_t\setminus\{n_t\})\cup C(n_t) .
\end{align*}
For a low-level infeasible node or a conflict-free node,~$C(n_t)=\emptyset$.
Note that a valid pruning rule such as \cref{prop:DIVE_incumbent_cutoff} may allow some of $C(n_t)$ to be discarded without adding to $\mathrm{OPEN}_t$.

A node-selection policy chooses the next processed node from the current open set.

\begin{definition}[Node-selection policy]
    A \emph{node-selection policy}~$\pi$ is any rule such that, for every solver state~$S_t$ with~$\mathrm{OPEN}\neq \emptyset$,~$    \pi(S_t)\in \mathrm{OPEN}_t$.
    The state $S_t$ may include the open set, the previously processed node, the incumbent value, queue size, tie-breaking data, or information learned from previous instances. 
    The unexpanded portions of $\mathbf T$ are not assumed to be known.
\end{definition}

In \Cref{fig:search_tree_diagram}, solid red nodes are processed nodes, light red nodes are open nodes, and gray nodes are undiscovered nodes in~$\mathcal T$. 
Squares indicate nodes with~$f(n)=\top$, and the star marks the best such node.

\subsection{Solve Completion Criteria}
\label{sec:FormalizedNodeSelection_SolveCompletionCriteria}

Assume for this subsection that the instance is feasible. 
The infeasible case is certified by exhausting $\mathrm{OPEN}$ without finding an incumbent. Let
\begin{equation}
    z^* := \min_{n\in\mathcal T: f(n)=\top} \ell(n)
    \label{eq:zstar_tree}
\end{equation}
be the optimal objective value represented in the tree. We partition $\mathcal T$ into three zones:
\begin{align*}
N^D&:=\{n\in \mathcal T \mid \ell(n)<z^*\} \text{ (dual zone)}, \\
N^O&:=\{n\in \mathcal T \mid \ell(n)=z^*\} \text{ (optimal zone)}, \\
N^P&:=\{n\in \mathcal T \mid \ell(n)>z^*\} \text{ (primal zone)}.
\end{align*}
The names describe their roles in the solve. 
Dual-zone nodes must be ruled out because their lower bounds are still strictly better than the optimum. 
At least one feasible optimal-zone node must be found. 
Primal-zone nodes are not needed for the final certificate, but they can provide incumbent solutions before optimality is proved.

An exact solve completed in $T$ iterations must satisfy two conditions. 
First, it must have found an optimal feasible node:
\begin{equation}
    \left|\{n\in E_T\cap N^O \mid f(n)=\top\}\right|\geq 1 .
    \label{eq:primal_crit}
\end{equation}
Second, every dual-zone node must have been resolved:
\begin{equation}
    N^D\subseteq E_T .
    \label{eq:dual_crit}
\end{equation}
Condition \eqref{eq:dual_crit} is necessary because a node with $\ell(n)<z^*$ cannot be eliminated by an incumbent cutoff before it is processed. 
Once \eqref{eq:primal_crit} holds, the incumbent value is $z^*$. 
Once \eqref{eq:dual_crit} also holds, no open node can certify a lower value than $z^*$, so the incumbent is optimal.

We extend our definition of $z^\mathrm P_t$ and $z^\mathrm D_t$ from \cref{sec:Dive} using the new notation. Define the primal bound and open-frontier lower bound at iteration $t$ by
\begin{align*}
    z^{\mathrm P}_t &:= \min_{n\in E_t: f(n)=\top}\ell(n),\\
    \underline z_t &:= \min_{n\in\mathrm{OPEN}_t}\ell(n),
\end{align*}
with the minimum over an empty set interpreted as $+\infty$. 
The active dual bound is
\begin{equation}
    z^{\mathrm D}_t := \min\{z^{\mathrm P}_t,\underline z_t\}.
    \label{eq:active_dual_bound}
\end{equation}
For a feasible incumbent,~$z^{\mathrm D}_t\leq z^*\leq z^{\mathrm P}_t$. 
The solve is certified when $z^{\mathrm P}_t=z^{\mathrm D}_t$, which occurs exactly when the incumbent is no worse than every open-node lower bound. 
Controlled visits to $N^P$ are useful because they can reduce $z^{\mathrm P}_t$ before \eqref{eq:dual_crit} has been completed.

\subsection{Node-Selection Objectives}
\label{sec:FormalizedNodeSelection_NodeSelectionObjectives}

Correctness fixes the destination of an exact solve, and node selection determines the route. 
We use five route-dependent quantities throughout the paper.

\paragraph{Expanded nodes}
The total number of processed nodes~$|E_T|$ is the most direct high-level measure of search effort. 
It is not a complete runtime model, because different processed nodes can have different low-level costs, but it captures how much of the high-level tree is examined.

\paragraph{Dive breaks}
Processing a child immediately after its parent can be substantially cheaper than jumping to an unrelated node, because the child differs from its parent by only one added constraint. 
Boyarski et al. introduced several ways to exploit this parent-child efficiency in \gls{acr:cbs}, and their results show that the dive continuity, rather than the raw number of expanded nodes alone, is an important indicator of implementation runtime~\cite{ID_boyarski2020}.
We therefore track the number of times the traversal starts or restarts a parent-child chain.

\begin{definition}[Dive break count]
\label{def:dive_break}
For an expansion sequence~$\tup{n_0,\ldots,n_T}$, the \emph{dive-break count} is
\begin{equation*}
    B_T := 1 + \left|\{t\in\{0,\ldots,T-1\}: n_{t+1}\notin C(n_t)\}\right| .
\end{equation*}
The initial $1$ counts the start of the first dive at the root. 
Lower $B_T$ means that the same processed-node set is covered by longer parent-child chains.
\end{definition}

Expanded nodes and dive breaks capture two limiting runtime regimes. 
If parent-child locality provides little benefit, runtime is driven mainly by $|E_T|$. 
If processing a child after its parent is nearly free, runtime is driven mainly by $B_T$.

\paragraph{Maximum queue size}
A major practical cost of \gls{acr:bnb} is the memory required to maintain the queue of open subproblems,~$\mathrm{OPEN}$. 
We measure it by
\begin{equation*}
    Q^{\max}:=\max_t |\mathrm{OPEN}_t| .
\end{equation*}
Queue size can also affect runtime through node-management overhead, especially when the frontier is frequently reordered or redistributed.

\paragraph{Primal-bound progress (intermediate solutions)}
A policy that enters $N^P$ may find feasible but not-yet-certified solutions. 
Such incumbents are valuable when the solve is interrupted, when a planner is embedded in a real-time system, or when users care about improving feasible solutions before the proof is complete. 
Relevant measurements include time to first incumbent, incumbent quality at a fixed budget, and the trajectory of~$z^{\mathrm P}_t$.

\paragraph{Dual bound progress (certificate of optimality)}
The dual bound measures certificate progress. 
Once an incumbent exists, the primal-dual gap
\begin{equation*}
    \mathrm{Gap}_t := \frac{z^{\mathrm P}_t-z^{\mathrm D}_t}{z^{\mathrm P}_t}
\end{equation*}
for $z^{\mathrm P}_t>0$ bounds the maximum relative improvement still possible. 
A policy can therefore be useful even before termination if it provides a good incumbent and a tight gap.

\subsection{Responsive Node Selection}
\label{sec:FormalizedNodeSelection_Feedback}

The preferred objective trade-off depends on the solve context. 
A memory-constrained robot may prefer a small open set. 
An offline benchmark may prioritize expanded-node count. 
An interactive planner may value an early incumbent and a meaningful gap certificate. 
A \emph{static} policy fixes this trade-off in advance. 
A \emph{responsive} policy lets the node-selection rule depend on monitored solve behavior while keeping all correctness-critical pruning tests unchanged.

We introduce \gls{acr:mcd} as a minimal example.
It follows \gls{acr:dive} while the frontier is below a soft memory threshold and switches to depth-first selection once the threshold is reached:
\begin{equation}
\pi_{\mathrm{MC\mbox{-}\gls{acr:dive}}}(S_t) :=
\begin{cases}
\pi_{\mathrm{\gls{acr:dive}}}(S_t), & |\mathrm{OPEN}_t| < Q_{\max},\\
\pi_{\mathrm{DFS}}(S_t), & |\mathrm{OPEN}_t| \geq Q_{\max}.
\end{cases}
\label{eq:mcdive_policy}
\end{equation}

The threshold is soft because processing one node can still add children before the next policy decision.
Exactness is unaffected: \gls{acr:mcd} still selects a node from $\mathrm{OPEN}$ at every iteration and relies on the same valid pruning rules as \gls{acr:dive}.

This example is intentionally simple. 
Its purpose is to show how node selection can form a feedback loop with the solve itself: use diversified \gls{acr:dive} when memory allows, then become more depth-oriented when the queue grows too large.
More elaborate responsive hybrids could adjust continuously between best-bound and depth-oriented behavior, or learn when long dives are likely to be productive on a given instance family. 
The next section focuses on the static policies that form the basic tradeoff extremes; \cref{sec:Experiments} then evaluates \gls{acr:mcd} as a proof of concept for responsiveness.

\begin{remark}[Responsive node selection in \gls{acr:mip}]
Responsive node selection is already commonplace in \gls{acr:mip}.
Solvers such as SCIP respond to memory constraints with more depth-focused search similarly to \gls{acr:mcd} \cite{SCIPThesis_achterberg2007}.
\end{remark}

\section{Node-Selection Trade-off Analysis}
\label{sec:TechnicalPolicyAnalysis}

This section compares \gls{acr:bfs}, \gls{acr:dfs}, \gls{acr:id}, and \gls{acr:dive} using the framework of \cref{sec:FormalizedNodeSelection}.
We will not identify a universally dominant policy. 
Instead, the analysis makes explicit which objective each policy is structurally designed to favor: \gls{acr:bfs} is strongest for expanded-node count and frontier-bound progress, \gls{acr:id} is strongest for queue size, and \gls{acr:dive} is strongest for dive continuity while still returning to the best-bound frontier between dives, occupying a unique position in the multi-objective design space.

All statements are with respect to the finite induced search tree of \cref{sec:FormalizedNodeSelection}. 
We assume no external warm start unless stated otherwise, and all policies use the same branching rule, lower-bound key, incumbent cutoff, and deterministic tie-breaking convention. 
Two assumptions are used below. 
The \emph{singleton optimal-zone assumption},~$|N^O|=1$, is used when comparing the nodes or dive breaks that are unavoidable for exactness. 
A separate \emph{perfect-tree} model is used only to obtain closed-form queue and dive-break counts.

\begin{definition}[Perfect search tree]
\label{def:perfectTree}
A \emph{perfect search tree} of depth~$d$ is a search tree (\cref{def:searchTree}) in which the root has depth 0, every leaf has depth~$d$, and every non-leaf node has exactly two children.
Its lower-bound key is strictly depth-consistent: for any nodes $m,n\in\mathcal T$, 
\begin{equation*}
    D(n)<D(m) \quad\Rightarrow\quad \ell(n)<\ell(m).
\end{equation*}
When analyzing traversal counts on a perfect tree, we consider the full-resolution case in which termination requires resolving every node of the tree.
This can be viewed as a stylized instance where the only feasible solution is encountered at the end of the final layer, so all nodes remain relevant to the certificate before termination.
Perfect-tree statements below describe the traversal of this complete depth-$d$ tree; they isolate structural queue and dive effects from instance-specific feasible-node placement.
\end{definition}

Note that the perfect-tree model is deliberately stylized. 
It should not be read as a generative model for every \gls{acr:cbs} tree. 
Its purpose is to make the policy trade-offs transparent in a controlled setting, while the experiments in \cref{sec:Experiments} test whether the same qualitative behavior appears on practical \gls{acr:mapf} instances.

\subsection{Best-First Search}

\gls{acr:bfs} always selects an open node with minimum lower-bound key:
\begin{equation*}
\pi_{\mathrm{BFS}}(S_t)\in\argmin_{n\in \mathrm{OPEN}_t} \ell(n).
\end{equation*}
Thus, before termination, \gls{acr:bfs} processes nodes in nondecreasing lower-bound order. 
It never needs to enter the primal zone:
\begin{equation*}
    E_T^{\mathrm{BFS}}\cap N^P=\emptyset .
\end{equation*}
This is the source of its node efficiency, but also the reason it provides no incumbent before the final iteration.

\begin{proposition}[\gls{acr:bfs} expanded-node optimality]
    \label{lem:BFS_NumNodes_SearchTree}
    Assume $|N^O|=1$. 
    Among exact policies on the same induced search tree, no policy expands fewer nodes than \gls{acr:bfs}.
\end{proposition}

\begin{proof}
Every exact policy must process all nodes in~$N^D$ by \eqref{eq:dual_crit} and must process at least one feasible node in~$N^O$ by \eqref{eq:primal_crit}.
Under~$|N^O|=1$, this lower bound is $|N^D|+1$ processed nodes. 
Since \gls{acr:bfs} processes all nodes with~$\ell(n)<z^*$ before any node with~$\ell(n)=z^*$, and then processes the unique optimal-zone node, it attains this lower bound.

\end{proof}

\gls{acr:bfs} also gives the strongest possible frontier-bound progress at any fixed processed-node budget. 
Let $z_t^{\mathrm D,\pi}$ denote the open-frontier lower bound \eqref{eq:active_dual_bound} under policy $\pi$.

\begin{proposition}[\gls{acr:bfs} frontier-bound dominance]
    \label{lem:BFS_DualBound_SearchTree}
For any policy $\pi$ and any iteration count $t$,
\begin{equation*}
    z_t^{\mathrm D, \pi}\leq z_t^{\mathrm D, \mathrm{BFS}}.
\end{equation*}
Consequently, \gls{acr:bfs} maximizes the dual bound whenever the active dual bound of \eqref{eq:active_dual_bound} is determined by the open frontier.
\end{proposition}

\begin{proof}
Fix any threshold $\lambda$. 
If a policy has $z_t^{\mathrm D,\pi}\geq\lambda$, then every node with key below $\lambda$ must already have been processed; otherwise such a node, or an unprocessed ancestor of it with key below $\lambda$, would still be open. 
\gls{acr:bfs} never processes a node with key at least $\lambda$ while a node with key below $\lambda$ is open.
Therefore, whenever any policy can exhaust all nodes below $\lambda$ within $t$ iterations, \gls{acr:bfs} can also do so within $t$ iterations. 
Since this holds for every $\lambda$, the claimed dominance follows.
\end{proof}

The same best-bound behavior is costly for memory.
On a perfect tree, \gls{acr:bfs} sweeps the tree layer by layer.

\begin{proposition}[\gls{acr:bfs} queue size on a perfect tree]
    \label{lem:BFS_MaxQueueSize_PerfectTree}
On a perfect tree of depth $d$, \gls{acr:bfs} reaches maximum queue size of~$2^d$.
\end{proposition}

\begin{proof}
Immediately before the first leaf is processed, all $2^d$ leaves are open and no leaf has yet been removed. 
No later queue can be larger, because leaves do not generate children.
\end{proof}

Best-first traversal also produces many short dives, limiting opportunities to exploit parent-child efficiency.

\begin{proposition}[\gls{acr:bfs} dive breaks on a perfect tree]
\label{lem:BFS_NumDives_PerfectTree}
On a perfect tree of depth~$d\geq 1$, \gls{acr:bfs} requires between~$2^{d+1}-d-1$ and $2^{d+1}-2$ dive breaks.
The lower value is attained when tie-breaking makes the first node of each new layer a child of the last processed node in the preceding layer.
\end{proposition}

\begin{proof}
The perfect tree contains~$2^{d+1}-1$ nodes. 
By \cref{def:dive_break}, the dive-break count equals the number of processed nodes minus the number of consecutive parent-child continuations. 
In a layer-by-layer traversal, parent-child continuations can occur only during transitions between consecutive layers.
There are exactly~$d$ such transitions, so at most~$d$ parent-child continuations are possible, yielding the lower value. 
At least one continuation is unavoidable, from the root to the first depth-one node, yielding the upper value.

\end{proof}

Thus, \gls{acr:bfs} is the policy optimized for the dual side of exact search, as it minimizes unavoidable work under a singleton optimal zone and maximizes frontier-bound progress. 
Its weaknesses are equally structural, as it can require an exponentially large frontier on broad trees, creates many dive breaks, and does not produce intermediate incumbents before termination.

\subsection{Depth-First Search}
DFS imposes the exact opposite tradeoffs.
DFS selects the node n with the highest depth D(n):
\begin{equation*}
    \pi_{\mathrm{DFS}}(S_t)\in\argmax_{n\in \mathrm{OPEN}_t} D(n)
\end{equation*}
This policy is useful as an analytical baseline and as a component of \gls{acr:id}, \gls{acr:dive}, and \gls{acr:mcd}.
By prioritizing depth, it produces long dives and very small frontiers, which is attractive for exploiting parent-child efficiency. 

\begin{proposition}[\gls{acr:dfs} dive breaks]
\label{lem:DFS_NumDives_AnyTree}
For any rooted binary tree that is fully traversed, \gls{acr:dfs} attains the minimum possible dive-break count. 
The count is the number of leaves of the traversed tree, which is $2^d$ on a perfect tree of depth $d$.
\end{proposition}

\begin{proof}
A parent-child chain can contain at most one leaf, so any traversal needs at least one dive per leaf. 
\gls{acr:dfs} continues the current chain until a leaf is reached, and distinct dives terminate at distinct leaves.
It therefore uses exactly one dive per leaf.
\end{proof}

The corresponding memory behavior is the opposite of \gls{acr:bfs}.

\begin{proposition}[\gls{acr:dfs} queue size]
\label{lem:DFS_MaxQueueSize_BinaryTree}
In any binary tree with maximum depth $d$, \gls{acr:dfs} has maximum queue size at most $d+1$.
\end{proposition}

\begin{proof}
After processing a node on the current root-to-node path, the open set contains at most one unprocessed sibling branch for each depth on that path, plus possibly the next child to continue the path. 
Hence $|\mathrm{OPEN}_t|\leq k+1$ when the current path has depth $k$, and $k\leq d$.
\end{proof}

The advantages of \gls{acr:dfs} are therefore clear: it has excellent dive continuity and a small frontier. 
Its drawback is also clear. 
Without best-bound reanchoring or incumbent cutoffs, it can spend arbitrarily much effort in $N^P$ before resolving nodes in $N^D$, so it provides weak certificate progress and unreliable node count.

\subsection{Iterative Deepening}

\gls{acr:id} uses depth-first passes under increasing lower-bound thresholds. 
Each pass has the memory profile of \gls{acr:dfs}, and the sequence of newly discovered nodes follows best-first order when thresholds and tie-breaking match \gls{acr:bfs}. 
Its cost is repeated work, as shallow nodes are reprocessed in multiple passes.

\begin{proposition}[\gls{acr:id} queue size]
    \label{lem:ID_MaxQueueSize_BinaryTree}
In a binary tree with maximum depth $d$, \gls{acr:id} has maximum queue size at most $d+1$.
\end{proposition}
\begin{proof}
Each threshold pass is a depth-first traversal of a subtree with depth at most $d$.
By \cref{lem:DFS_MaxQueueSize_BinaryTree}, every pass uses queue size at most $d+1$, and the maximum over passes satisfies the same bound.
\end{proof}

\begin{proposition}[\gls{acr:id} dive breaks on a perfect tree]
\label{lem:ID_NumDives_PerfectTree}
On a perfect tree of depth $d$, \gls{acr:id} uses $2^{d+1}-1$ dive breaks.
\end{proposition}
\begin{proof}
The pass with depth limit $k$ performs a depth-first traversal of a perfect tree of depth $k$, which uses $2^k$ dive breaks by \cref{lem:DFS_NumDives_AnyTree}. 
Summing over $k=0,\ldots,d$ gives
\begin{equation*}
    \sum_{k=0}^{d}2^k=2^{d+1}-1 .
\end{equation*}
\end{proof}

This shows why \gls{acr:idcbs} is not simply a dive-efficiency improvement over \gls{acr:bfs}.
The long depth-first chains in each individual threshold pass are balanced by the additional dive breaks in repetitive passes.
In the perfect-tree model, \gls{acr:id} uses $d$ more dive breaks than the most favorable \gls{acr:bfs} tie-breaking. 
Its primary structural advantage is memory, not parent-child continuity over the complete solve.

\begin{remark}[Memory accounting for iterative deepening]
The queue-size comparison for \gls{acr:id} should be interpreted with care.
Unlike one-pass policies such as \gls{acr:bfs}, \gls{acr:dfs}, and \gls{acr:dive}, \gls{acr:id} does not maintain a persistent frontier across the entire solve.
Instead, it repeatedly discards the current depth-first frontier and restarts from the root with a larger cost threshold. 
Thus, part of the information about which regions have already been exhausted is stored implicitly in the threshold rather than explicitly in \textsc{OPEN}. 
This is precisely why \gls{acr:id} can achieve a very small explicit queue, but it is also why the same nodes may be re-expanded across successive passes. 
Consequently, the linear queue-size bound for \gls{acr:id} should be read as a bound on explicit frontier memory, not as a claim that \gls{acr:id} stores all information needed for a one-pass traversal of the required subtree \(N_D \cup N_O\).
\end{remark}

\subsection{\texorpdfstring{\gls{acr:dive}}{DIVE}}
\label{sec:TechnicalPolicyAnalysis_DIVE}

\gls{acr:dive} combines best-bound reanchoring with depth-oriented continuation. 
Let
\begin{equation*}
    C_t^{\mathrm P}(n_{t-1}) := \{c\in C(n_{t-1})\cap\mathrm{OPEN}_t : \ell(c)<z_t^{\mathrm P}\}
\end{equation*}
be the open children of the previously processed node that can still improve the incumbent. 
\gls{acr:dive} selects
\begin{equation*}
\pi_{\mathrm{\gls{acr:dive}}}(S_t)=
\begin{cases}
\displaystyle \argmin_{c\in C_t^{\mathrm P}(n_{t-1})}\ell(c), & C_t^{\mathrm P}(n_{t-1})\neq\emptyset, \\
\pi_{\mathrm{BFS}}(S_t), & C_t^{\mathrm P}(n_{t-1})=\emptyset .
\end{cases}
\end{equation*}

Thus a dive continues only while an open child of the previous node remains capable of improving the incumbent. 
Otherwise, the next dive restarts from the best-bound frontier.

For a processed node set $E$, define its induced leaf set
\begin{equation*}
    \mathcal L(E):=\{n\in E : C(n)\cap E=\emptyset\} .
\end{equation*}

\begin{proposition}[\gls{acr:dive} number of dive breaks on any tree]
\label{lem:DIVE_NumDives_AnyTree}
For the rooted tree induced by the nodes processed by \gls{acr:dive}, \gls{acr:dive} attains the minimum possible dive-break count:~$|\mathcal L(E_T^{\mathrm{\gls{acr:dive}}})|$.
\end{proposition}
\begin{proof}
Any parent-child chain contains at most one induced leaf, so any traversal of the same processed tree requires at least $|\mathcal L(E_T^{\mathrm{\gls{acr:dive}}})|$ dive breaks. 
\gls{acr:dive} continues from a node to one of its open children whenever such a child can still be processed by the policy. 
Therefore, a \gls{acr:dive} dive terminates only at a node with no processed child in the induced tree. 
Each dive ends at a distinct induced leaf, so \gls{acr:dive} attains the lower bound.

\end{proof}

The more important exact-search statement concerns the part of the tree that must be resolved to certify optimality.

\begin{proposition}[\gls{acr:dive} number of dive breaks on search tree]
\label{lem:DIVE_NumDives_SearchTree}
Assume $|N^O|=1$ and no external warm start. 
Let $\mathcal T_R:=N^D\cup N^O$. 
Among exact policies on the same induced search tree, no policy completes the solve with fewer dive breaks than \gls{acr:dive}.
\end{proposition}

\begin{proof}
By monotonicity of $\ell$, $\mathcal T_R$ is a rooted subtree: every ancestor of a node with key at most $z^*$ also has key at most $z^*$. Under $|N^O|=1$, every exact policy must process every node in $\mathcal T_R$. A parent-child chain can contain at most one leaf of $\mathcal T_R$, so every exact policy needs at least $|\mathcal L(\mathcal T_R)|$ dive breaks.

Before the solve is complete, every new \gls{acr:dive} dive starts from the best-bound open node. 
If any node in $\mathcal T_R$ remains open, this start node lies in $\mathcal T_R$. 
While the current node has a child in $\mathcal T_R$, that child has lower key than any child in $N^P$ and is therefore selected by \gls{acr:dive}. 
Thus each \gls{acr:dive} dive reaches exactly one leaf of $\mathcal T_R$ before either terminating or continuing into $N^P$. 
\gls{acr:dive} covers all leaves of $\mathcal T_R$ and uses exactly $|\mathcal L(\mathcal T_R)|$ dive breaks, matching the lower bound.
\end{proof}

\begin{remark}[Pareto extreme]
\Cref{lem:BFS_NumNodes_SearchTree,lem:ID_MaxQueueSize_BinaryTree,lem:DIVE_NumDives_SearchTree} identify three different structural extremes. 
\gls{acr:bfs} minimizes unavoidable expanded nodes under a singleton optimal zone, \gls{acr:id} preserves the linear-depth queue profile of \gls{acr:dfs}, and \gls{acr:dive} minimizes the dive breaks needed for exact search. 
This is the core technical reason \gls{acr:dive} should be viewed as a distinct Pareto point rather than as a replacement for either baseline.
This formal claim on minimizing dive breaks is not just new to the \gls{acr:cbs} context but to the broader literature of \gls{acr:bnb} to the best of our knowledge.
\end{remark}

\gls{acr:dive} does not match \gls{acr:bfs} on pointwise frontier-bound progress; \cref{lem:BFS_DualBound_SearchTree} says no non-best-first policy can. 
Its advantage over plain \gls{acr:dfs} is that every dive is followed by a best-bound restart. 
Consequently, primal-zone excursions are localized, as a poor branch choice can consume at most one dive before the policy reanchors at the global frontier.

This reanchoring is also visible at the level of bound trajectories.
To make this certificate behavior precise, consider the unit-increment perfect tree in which every node at depth~$k$ has key~$k$ and the optimal value is the leaf-layer value~$z^*=d$.
For a policy~$\pi$, let $z^{\mathrm D,\pi}_t$ denote the active dual bound after the $t$-th high-level expansion and let $T_\pi$ be the first iteration at which the incumbent is certified.
We measure certificate progress by the discrete bound-area metric
\begin{equation}
\label{eq:bound_auc}
    \mathcal A_\pi(d) := \sum_{t=0}^{T_\pi-1} \bigl(z^*-z^{\mathrm D,\pi}_t\bigr).
\end{equation}
This is the area under the remaining duality gap (smaller values are better).
The metric rewards both a high dual bound and reaching certification in few expansions, so it is more informative than comparing only the final certificate time or only the pointwise bound.

\begin{figure}[tb]
    \centering
    \subfloat{%
        \includegraphics[width=0.5\linewidth]{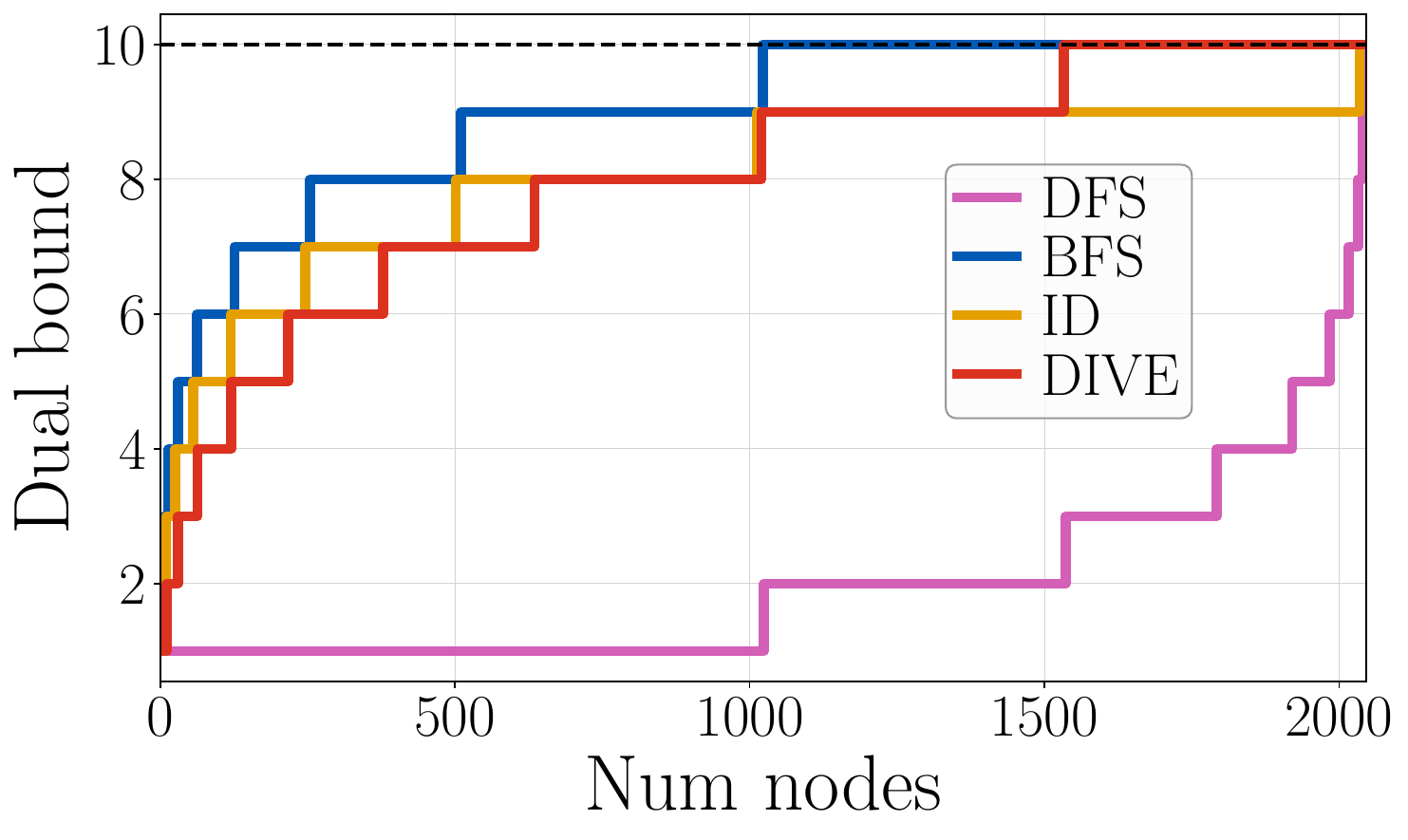}%
    }%
    \subfloat{%
        \includegraphics[width=0.5\linewidth]{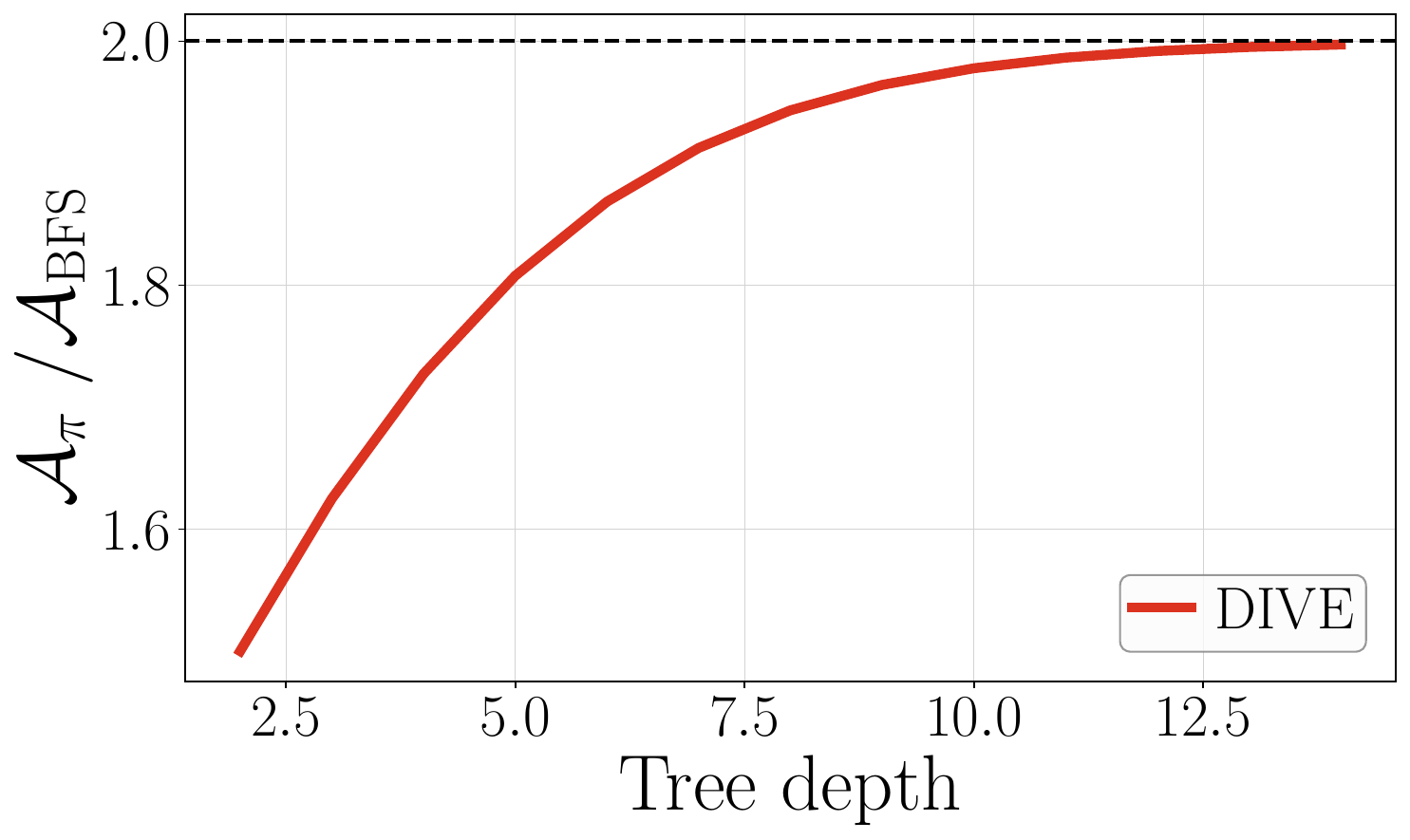}%
    }%
    \caption{Left: Illustrative frontier-bound progression on a unit-increment perfect tree of depth $10$. 
    \gls{acr:bfs} is pointwise best, but \gls{acr:dive} repeatedly returns to the best-bound frontier and therefore avoids the prolonged certificate stagnation of plain \gls{acr:dfs}. Right: Illustrative bound-progress of \gls{acr:dive} relative to \gls{acr:bfs} on unit-increment perfect trees. 
    The curve approaches two as the depth grows, reflecting the cost of \gls{acr:dive}'s depth probes in this stylized setting.}
    \label{fig:DualBound_PerfectTree}
    \label{fig:AUC}
\end{figure}

\Cref{fig:DualBound_PerfectTree,fig:AUC} summarize the resulting stylized calculation.
As expected from \cref{lem:BFS_DualBound_SearchTree}, \gls{acr:bfs} is pointwise dominant in dual-bound progress.
The more useful observation is that \gls{acr:dive} tracks the certificate frontier much more closely than plain \gls{acr:dfs}: after each depth probe, it returns to the best-bound layer instead of continuing indefinitely in a single deep region.
In the unit-increment perfect tree, a direct enumeration of the traversal gives
\begin{equation*}
    \mathcal A_{\mathrm{\gls{acr:dive}}}(d)=\mathcal A_{\mathrm{ID}}(d),
    \qquad
    \frac{\mathcal A_{\mathrm{\gls{acr:dive}}}(d)}{\mathcal A_{\mathrm{BFS}}(d)} \longrightarrow 2 .
\end{equation*}
Thus \gls{acr:dive} and \gls{acr:id} pay the same aggregate dual-bound price in this stylized model, but in different ways: \gls{acr:dive} pays it through temporary depth probes, whereas \gls{acr:id} pays it through repeated restarts.
This equality is not used as a general theorem for arbitrary \gls{acr:cbs} trees.
Its role is to expose a useful structural point, as \gls{acr:dive}'s dive continuity does not come from abandoning the certificate frontier, but from repeatedly making short, controlled departures from it.

In general search trees, \gls{acr:dive}'s initial dual-bound lag can be larger when a dive enters the primal zone.
The same excursions, however, are exactly what can produce early incumbents.
Once an incumbent exists, \gls{acr:dive} can report a primal--dual gap throughout the remainder of the solve, and if the incumbent is already optimal then the solve terminates as soon as the dual bound catches up.
By contrast, \gls{acr:bfs} and \gls{acr:id} usually provide no primal bound before they reach an optimal feasible node, even when their dual bounds are strong.

\gls{acr:dive}'s memory behavior lies between \gls{acr:bfs} and \gls{acr:id}. 
It leaves one unchosen sibling behind at each internal node of a dive, but it does not maintain the entire best-bound layer as \gls{acr:bfs} does.

\begin{proposition}[\gls{acr:dive} queue size on a perfect tree]
\label{lem:DIVE_MaxQueueSize_PerfectTree}
On a perfect tree of depth $d\geq 1$, \gls{acr:dive} reaches maximum queue size $2^{d-1}+1$ under the queue convention of \cref{sec:FormalizedNodeSelection_NodeSelectionObjectives}.
\end{proposition}

\begin{proof}
Call the first node of a dive its start node.
The root is the unique start node at depth $0$.
For each $k\geq1$, there are $2^{k-1}$ start nodes at depth $k$: each internal node at depth $k-1$ contributes exactly one child that is not followed in its parent's dive and is therefore saved as a future start node.

A dive starting at depth $k$ increases the queue by one for each internal node on the dive and then removes the terminal leaf.
Hence its net change after completion is $d-k-1$, while its maximum temporary increase during the dive is $d-k$.
Let $M_k$ be the largest queue size reached while processing start nodes of depth $k$.
After the root dive and all completed start-node dives of depths $1,\ldots,k-1$, direct summation gives
\begin{equation*}
    M_k=d+1+\sum_{j=1}^{k}2^{j-1}(d-j-1), \qquad k=0,\ldots,d-1,
\end{equation*}
where the empty sum gives $M_0=d+1$.
The increment from $M_{k-1}$ to $M_k$ is $2^{k-1}(d-k-1)$, which is positive for $k\leq d-2$ and zero for $k=d-1$.
Therefore the maximum is attained for $k=d-2$ or $k=d-1$.
Evaluating the sum yields
\begin{equation*}
    Q^{\max}_{\mathrm{\gls{acr:dive}}}=d+1+\sum_{j=1}^{d-2}2^{j-1}(d-j-1)=2^{d-1}+1 .
\end{equation*}
\end{proof}

Thus, on the perfect tree, \gls{acr:dive} uses essentially half the frontier of \gls{acr:bfs}, while \gls{acr:id} remains exponentially smaller.
This is precisely the intended compromise: \gls{acr:dive} gives up the linear-depth memory guarantee of \gls{acr:id} in order to avoid repeated restarts and preserve best-bound diversification between dives.

\begin{table}[t]
\setlength{\abovecaptionskip}{\TableCaptionSkip}
\caption{On a perfect tree of depth $d$, \gls{acr:id} sacrifices dive breaks while DIVE sacrifices maximum queue size while in order to control primal zone exploration.}
\label{tab:perfect_tree_stats}
\centering
\renewcommand{\arraystretch}{\TableVerticalStretch}
\begin{NiceTabular}{l|cc}
\toprule
Policy & Max queue size & Number of dive breaks\\
\midrule

\gls{acr:bfs}  & $2^d$       & $[2^{d+1}-d-1,\,2^{d+1}-2]$ \\
DFS     & $\mathbf{d+1}$   & $\mathbf{2^d}$     \\
ID      & $\mathbf{d+1}$   & $2^{d+1}-1$        \\
DIVE    & $2^{d-1}+1$         & $\mathbf{2^d}$    \\

\bottomrule
\end{NiceTabular}

\end{table}

\Cref{tab:perfect_tree_stats} summarizes the stylized counts.
\gls{acr:dfs} is optimal on dive breaks and queue size when the whole tree must be traversed, but it lacks a mechanism to control superfluous primal-zone exploration.
\gls{acr:id} keeps the \gls{acr:dfs} memory profile but pays for it through repeated threshold passes.
\gls{acr:dive} keeps the \gls{acr:dfs} dive-break profile on the required subtree while restoring best-bound reanchoring between dives.
This combination is the structural tradeoff evaluated empirically in \cref{sec:Experiments}.

\section{Experiments}
\label{sec:Experiments}
We evaluate \gls{acr:dive} against \gls{acr:bfs} and \gls{acr:id} in a shared \gls{acr:cbs} implementation to test whether the structural trade-offs developed in previous sections persist in practice.
The central prediction is that \gls{acr:dive} should not uniformly dominate the existing policies.
Rather, \gls{acr:bfs} should be strongest in expanded-node count and dual-bound progress, \gls{acr:id} should be strongest in explicit queue size, and \gls{acr:dive} should occupy a distinct Pareto point by minimizing dive breaks while retaining best-bound reanchoring and producing quality intermediate solutions.

The experiments therefore serve five roles. 
First, they locate \gls{acr:dive} within the multi-objective design space induced by expanded nodes, dive breaks, queue size, and bound progression. 
Second, they explain \gls{acr:dive}'s node overhead through zone accounting, showing whether the extra work comes from controlled primal-zone exploration rather than uncontrolled search. 
Third, they test whether that primal-zone exploration translates into useful anytime behavior, namely early incumbents together with certified primal-dual gaps before the exact solve is complete. 
Fourth, they examine whether warm starts stabilize \gls{acr:dive} in dense regimes by activating incumbent pruning early. 
Finally, they isolate the two mechanisms that make \gls{acr:dive} different from plain depth-first search: diversified search through best-bound reanchoring, and feedback through a simple memory-responsive variant.

\subsection{Experimental Protocol}

\subsubsection{Implementation}
All policies are implemented in the same code base and share the same low-level search, conflict selection, admissible high-level heuristic, bypassing logic, pruning rules, and node data structures. 
The main comparisons involve \gls{acr:bfs}, \gls{acr:id}, and \gls{acr:dive}.
Additional policies are used only for ablations in \cref{sec:Experiments_Ablation,sec:Experiments_Responsive}.
The only algorithmic difference in the main comparison is the high-level node-selection rule.
All experiments were run single-threaded on a Dell Latitude 5530 with an Intel Core i7-1265U processor and 16 GB RAM under Windows 11 Pro. 
Wall-clock time is used only as a truncation criterion and, for \gls{acr:dive}, as an operational measure of how quickly the first incumbent is found.

\subsubsection{Tie-breaking and shared CBS features}
When two nodes have the same lower bound, \gls{acr:bfs} and \gls{acr:dive} both prefer the deeper node.
If depth is also tied, we prefer the node whose selected conflict has higher cardinality. 
For \gls{acr:id}, depth is the primary key within a threshold pass and lower bound is used as the secondary key. 
This tie-breaking gives \gls{acr:bfs} and \gls{acr:id} substantial opportunity to exploit parent-child locality whenever the lower bound is indifferent, making the comparison conservative with respect to \gls{acr:dive}'s dive continuity advantage.

To avoid evaluating node selection inside an unrealistically weak solver, all policies use the same standard \gls{acr:cbs} enhancements: cardinal-conflict prioritization, bypassing, and an admissible high-level heuristic based on a greedy matching of independent cardinal conflicts. 
We also give all policies the same priority for conflict-free nodes once they appear in the queue. 
This implementation detail can only help the baseline policies in incumbent-related comparisons, because it allows them to process an available feasible node earlier than a strict best-bound-only implementation would.

\subsubsection{Benchmarks and instance generation}
We evaluate on four standard map families: \texttt{empty-32-32}, \texttt{random-32-32-20}, \texttt{room-32-32-4}, and \texttt{warehouse-small} \cite{MAPFBenchmarks_stern2019}. 
Each map is tested at 3\%, 5\%, 10\%, and 15\% robot occupancy. 
We also use \texttt{test-5-5} as a dense stress test in \cref{sec:Experiments_WarmStart}; its small low-level graph makes it possible to evaluate 33\% occupancy without making the low-level search itself dominate the experiment.

For \texttt{empty-32-32}, \texttt{random-32-32-20}, and \texttt{room-32-32-4}, the first 25 instances per robot count are generated from the even scenarios of the standard \gls{acr:mapf} benchmark~\cite{MAPFBenchmarks_stern2019}. 
Additional instances, and all instances for the other maps, are generated by sampling start and goal vertices uniformly without replacement from the free cells and discarding invalid or duplicate draws. 
The same paired instance set is used for every policy.

\subsubsection{Filtering and reported metrics}
For complete-solve comparisons, we exclude instances on which at least one policy solves the problem in at most 100 high-level node expansions. 
These shallow cases are useful as tests, but they are too small for the structural metrics of \cref{sec:FormalizedNodeSelection_NodeSelectionObjectives} to be informative. 
We also truncate each run after five minutes of wall-clock time. 
An instance is called \emph{complete} if all three main policies solve it within the cutoff and the instance is not excluded by the 100-node filter. 
\Cref{tab:instance_validity_table} reports the full accounting.\footnote{The full experimental campaign required five days and fourteen hours of cumulative compute time.}{} 
For each map, at least one robot count has at least 30 complete instances, highlighted in bold.

The primary reported quantities are expanded nodes, dive breaks, and maximum queue size.
These are the implementation-independent quantities that most directly correspond to the objectives in \cref{sec:FormalizedNodeSelection_NodeSelectionObjectives}: expanded nodes approximate high-level search effort, dive breaks approximate the loss of parent-child locality, and maximum queue size approximates explicit frontier memory. 
We do not claim that these metrics fully determine wall-clock time for every \gls{acr:cbs} implementation. 
Rather, they isolate the algorithmic effect of node selection from lower-level engineering choices.

\begin{table}[tb]
\setlength{\abovecaptionskip}{\TableCaptionSkip}
\caption{Generated-instance accounting. 
Complete (C): all main policies solved the instance within the cutoff and the instance passed the 100-node filter.
Timeout (T/O): at least one policy hit the 5-minute cutoff.
Insufficient nodes (IN): at least one policy solved the instance in at most 100 high-level expansions.}
\label{tab:instance_validity_table}
\centering
\renewcommand{\arraystretch}{\TableVerticalStretch}   
\begin{NiceTabular}{lcc|cccc}
    \toprule
    Map & & \# robots & C & T/O & IN & Tot \\
    \midrule
    \multirow[c]{4}{*}{empty-32-32}
     & \multirow[c]{4}{*}{\includegraphics[width=1cm]{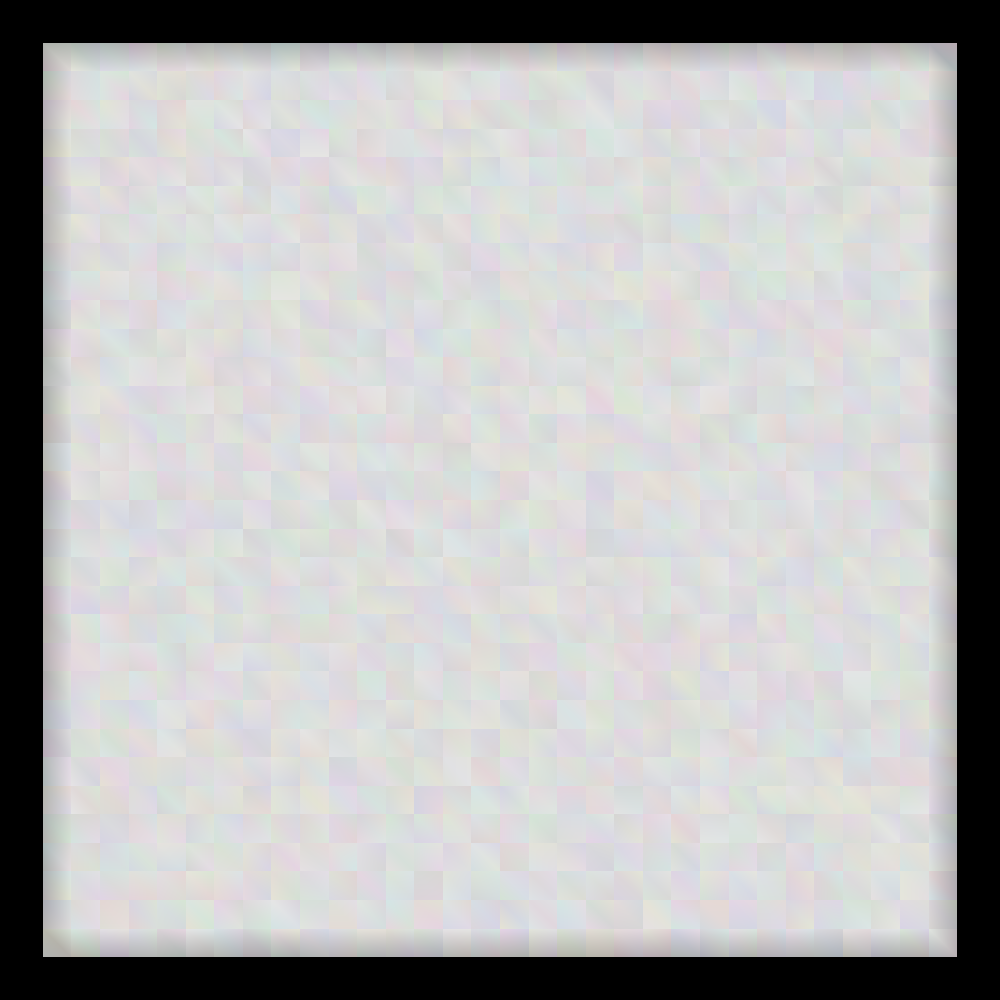}}
     & 31 (3\%)  & 0  & 1  & 24 & 25 \\
     & & \textbf{51 (5\%)}  & \textbf{31} & \textbf{37} & \textbf{33} & \textbf{101} \\
     & & 102 (10\%) & 0  & 25 & 0  & 25 \\
     & & 154 (15\%) & 0  & 10 & 0  & 10 \\
    \midrule

    \multirow[c]{4}{*}{random-32-32-20}
     & \multirow[c]{4}{*}{\includegraphics[width=1cm]{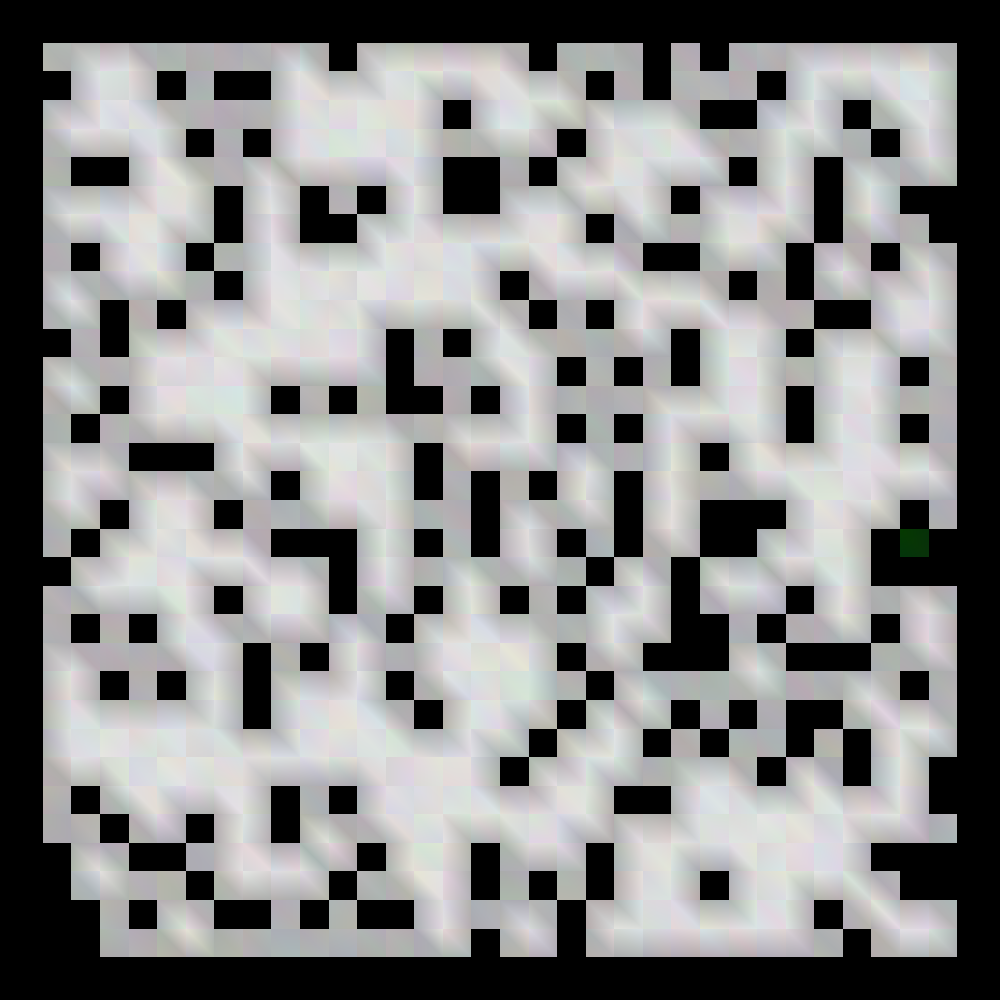}}
     & 25 (3\%)  & 3  & 1  & 21 & 25 \\
     & & \textbf{41 (5\%)}  & \textbf{36} & \textbf{20} & \textbf{4} & \textbf{60} \\
     & & 82 (10\%)  & 0  & 25 & 0  & 25 \\
     & & 123 (15\%) & 0  & 10 & 0  & 10 \\
    \midrule

    \multirow[c]{4}{*}{room-32-32-4}
     & \multirow[c]{4}{*}{\includegraphics[width=1cm]{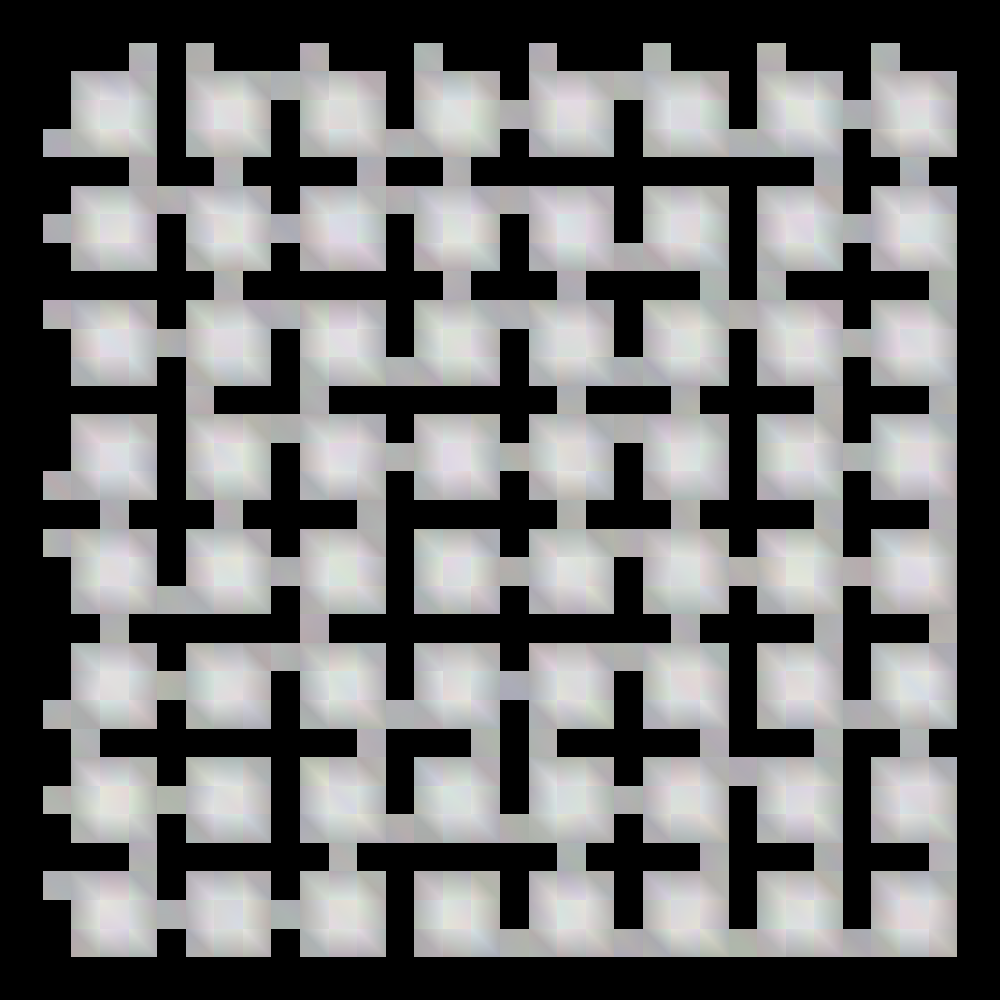}}
     & \textbf{20 (3\%)}  & \textbf{32} & \textbf{13}  & \textbf{15} & \textbf{60} \\
     & & 34 (5\%)  & 0  & 25 & 0  & 25 \\
     & & 68 (10\%)  & 0  & 25 & 0  & 25 \\
     & & 102 (15\%) & 0  & 10 & 0  & 10 \\
    \midrule

    \multirow[c]{4}{*}{warehouse-small}
     & \multirow[c]{4}{*}{\includegraphics[width=1.4cm]{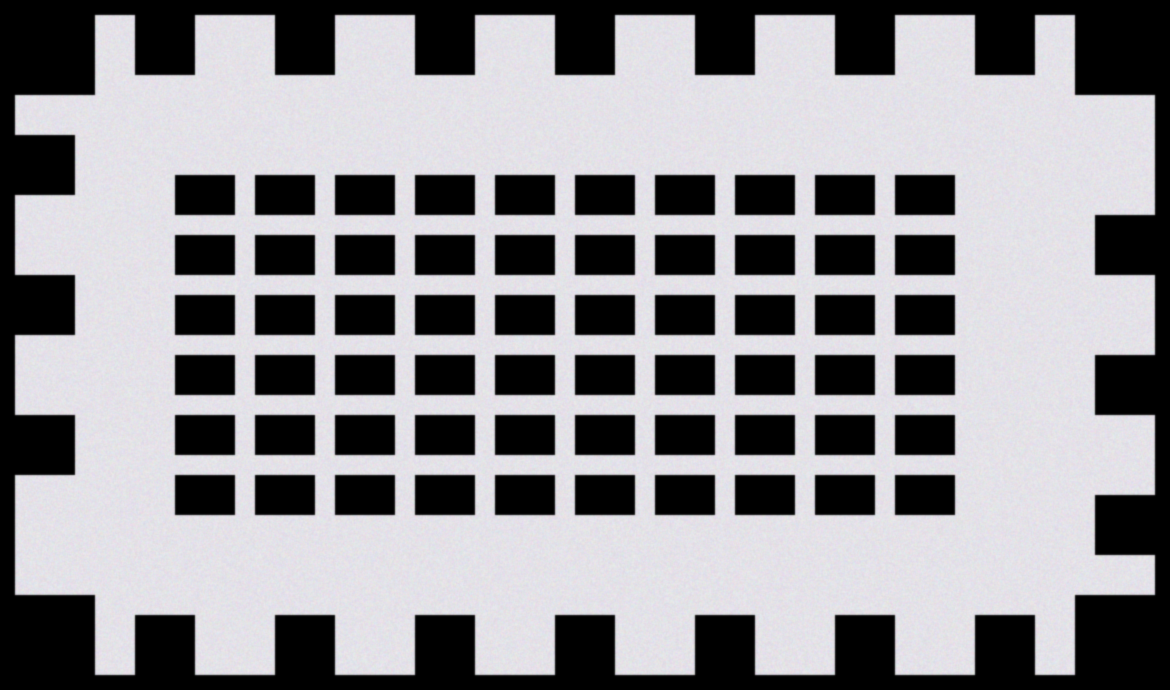}}
     & \textbf{38 (3\%)}  & \textbf{50} & \textbf{9}  & \textbf{40} & \textbf{99} \\
     & & 64 (5\%)  & 0  & 25 & 0  & 25 \\
     & & 128 (10\%) & 0  & 25 & 0  & 25 \\
     & & 192 (15\%) & 0  & 10 & 0  & 10 \\
    \bottomrule
\end{NiceTabular}
\end{table}

\subsection{Complete-Solve Trade-Offs}
\label{sec:Experiments_CompleteSolve}

\begin{table}[t]
\setlength{\abovecaptionskip}{\TableCaptionSkip}
\caption{Complete-solve structural metrics. Values are mean percentage differences relative to \gls{acr:bfs}. Lower is better for all three metrics.}
\label{tab:summaryTable_completeSeeds}
\centering
\renewcommand{\arraystretch}{\TableVerticalStretch}
\setlength{\tabcolsep}{2.5pt}
\begin{NiceTabular}{l|rrr|rrr|rrr}
    \toprule
     \Block{2-1}{Map\! (\# robots)} & \multicolumn{3}{c}{\# nodes} & \multicolumn{3}{c}{\# dive breaks} & \multicolumn{3}{c}{Max queue size} \\
     & BFS & DIVE & ID & BFS & DIVE & ID & BFS & DIVE & ID \\
    \midrule
    empty\! (51) & \textbf{0\%} & 0\% & 13\% & 0\% & \textbf{-6\%} & 3\% & 0\% & -86\% & \textbf{-89\%} \\
    random\! (41) & \textbf{0\%} & 57\% & 118\% & 0\% & \textbf{-40\%} & 47\% & 0\% & -52\% & \textbf{-95\%} \\
    room\! (20) & \textbf{0\%} & 55\% & 134\% & 0\% & \textbf{-51\%} & 26\% & 0\% & -44\% & \textbf{-97\%} \\
    warehouse\! (38) & \textbf{0\%} & 18\% & 89\% & 0\% & \textbf{-23\%} & 54\% & 0\% & -66\% & \textbf{-92\%} \\
    \bottomrule
\end{NiceTabular}
\end{table}

\cref{tab:summaryTable_completeSeeds} gives the main complete-solve comparison.
The punchline is that the three policies separate almost exactly as predicted universally across each map.
\Gls{acr:bfs} remains the best policy for minimizing expanded nodes.
\Gls{acr:id} gives the smallest explicit queue, reducing the maximum queue by 89--97\% relative to \gls{acr:bfs}.
\gls{acr:dive} is the only policy that consistently reduces dive breaks relative to \gls{acr:bfs}, with reductions from 6\% to 51\%, while also reducing the maximum queue by 44--86\%. 

The cost of \gls{acr:dive} is controlled node overhead, as it matches \gls{acr:bfs} on \texttt{empty-32-32} and expands 18--57\% more nodes on the other complete benchmark groups. 
This is substantially less than the node overhead of \gls{acr:id}, which ranges from 13\% to 134\%. 
Thus, \gls{acr:dive} should not be interpreted as a universal replacement for \gls{acr:bfs} or \gls{acr:id}. 
It is a third operating point, which sacrifices the strict node-minimality of \gls{acr:bfs} and the near-minimal queue of \gls{acr:id} in order to obtain far better parent-child continuity without repeated restarts.

\begin{figure}[tb]
    \centering
    \subfloat{%
        \includegraphics[width=0.5\linewidth]{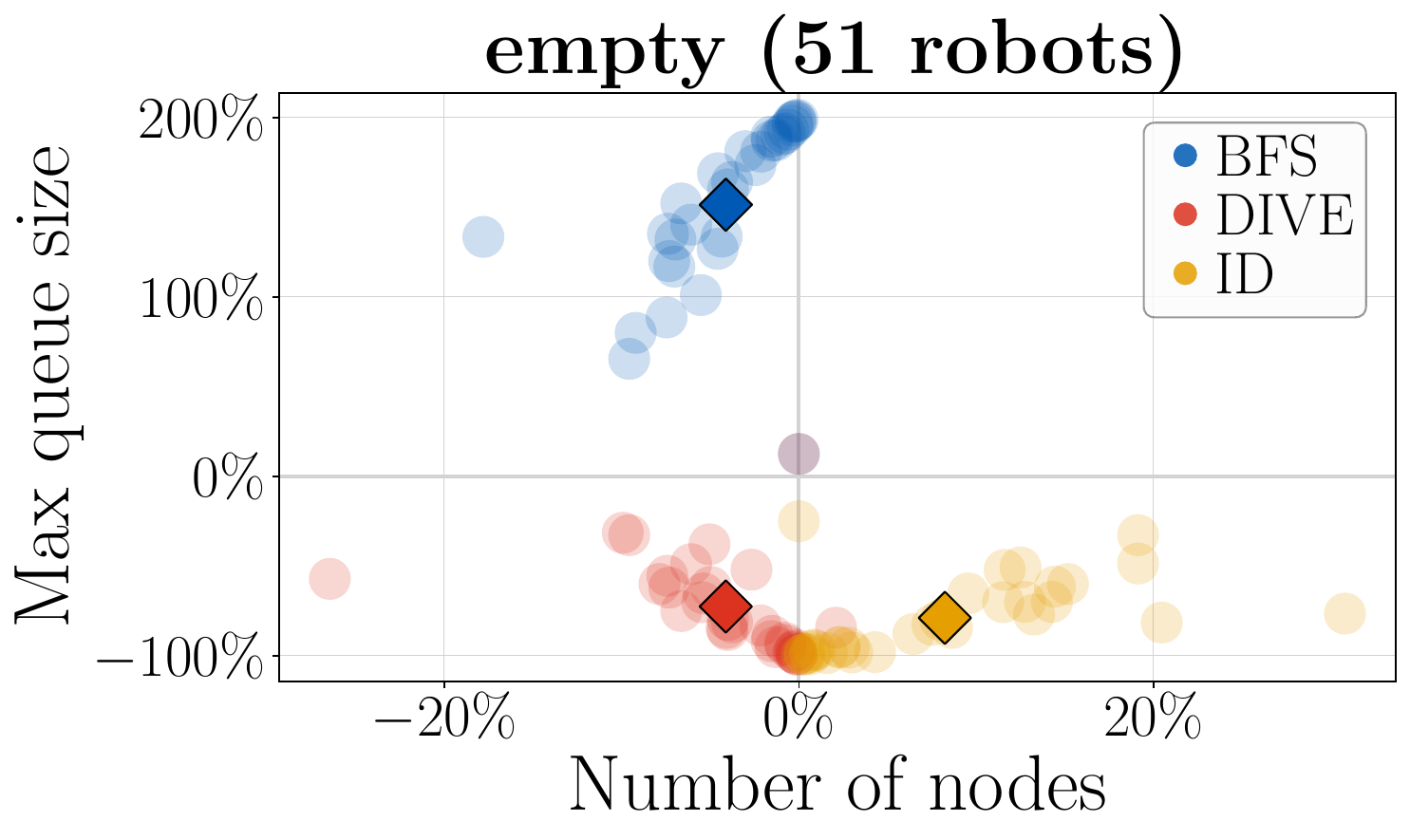}%
    }%
    \subfloat{%
        \includegraphics[width=0.5\linewidth]{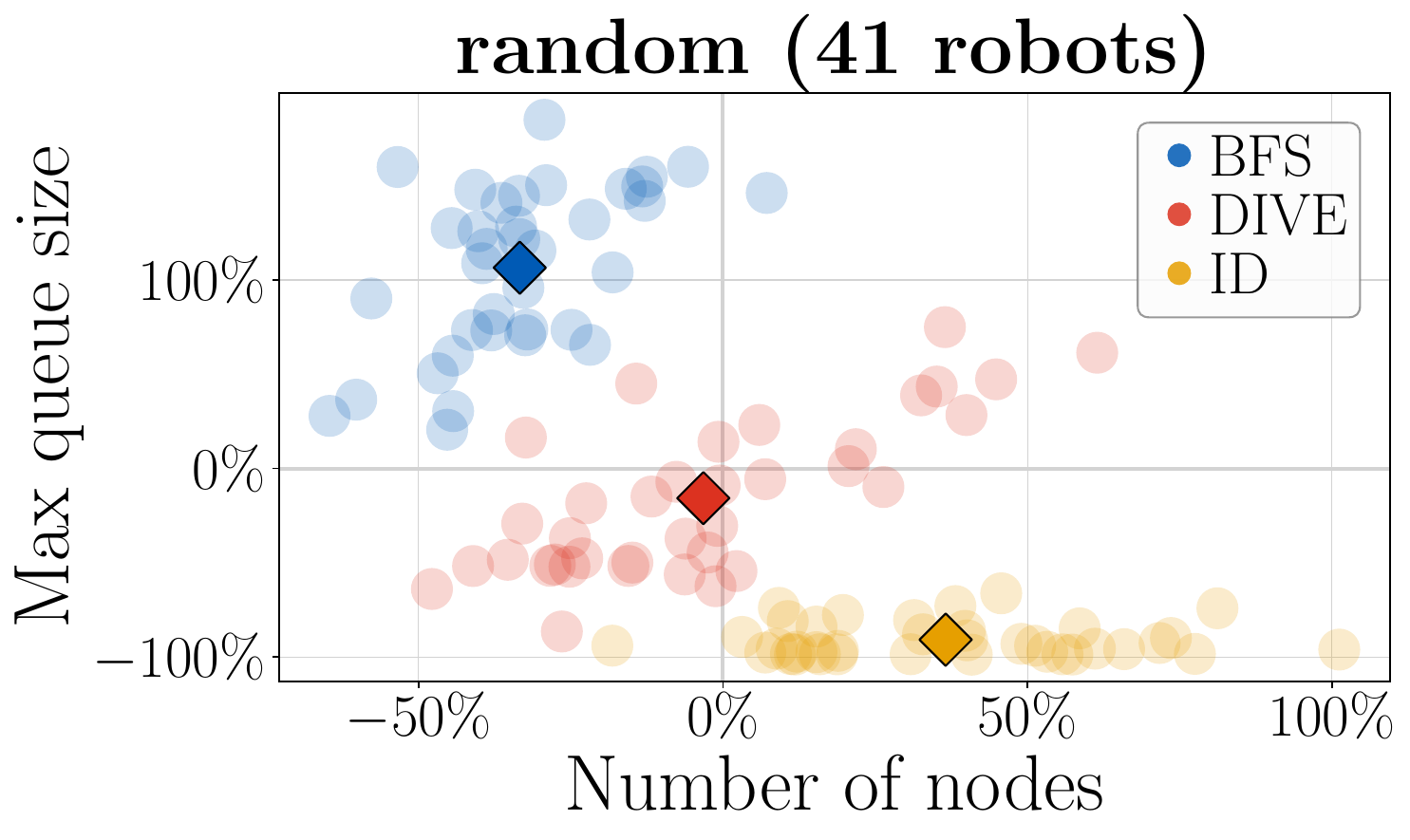}%
    }\\
    \subfloat{%
        \includegraphics[width=0.5\linewidth]{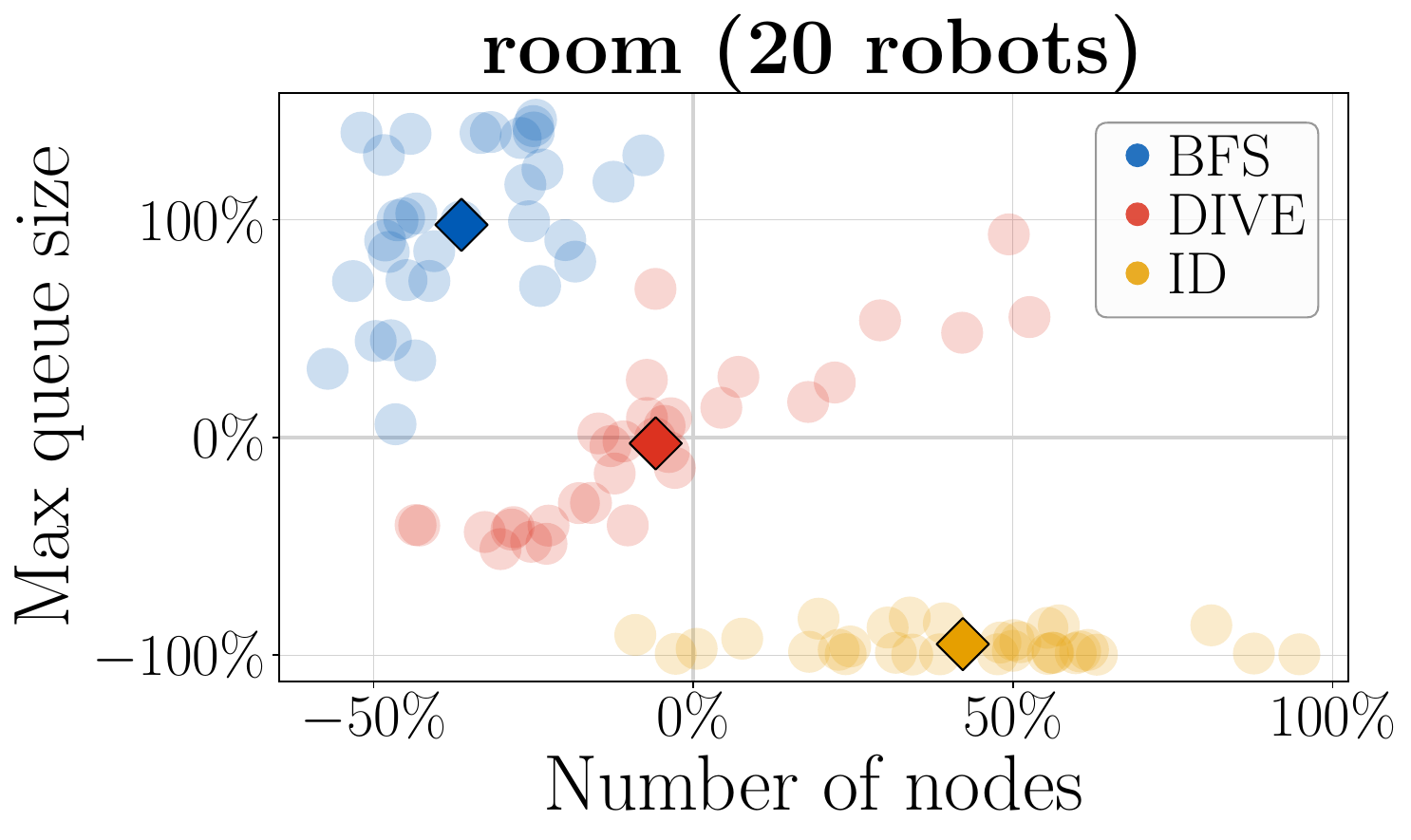}%
    }%
    \subfloat{%
        \includegraphics[width=0.5\linewidth]{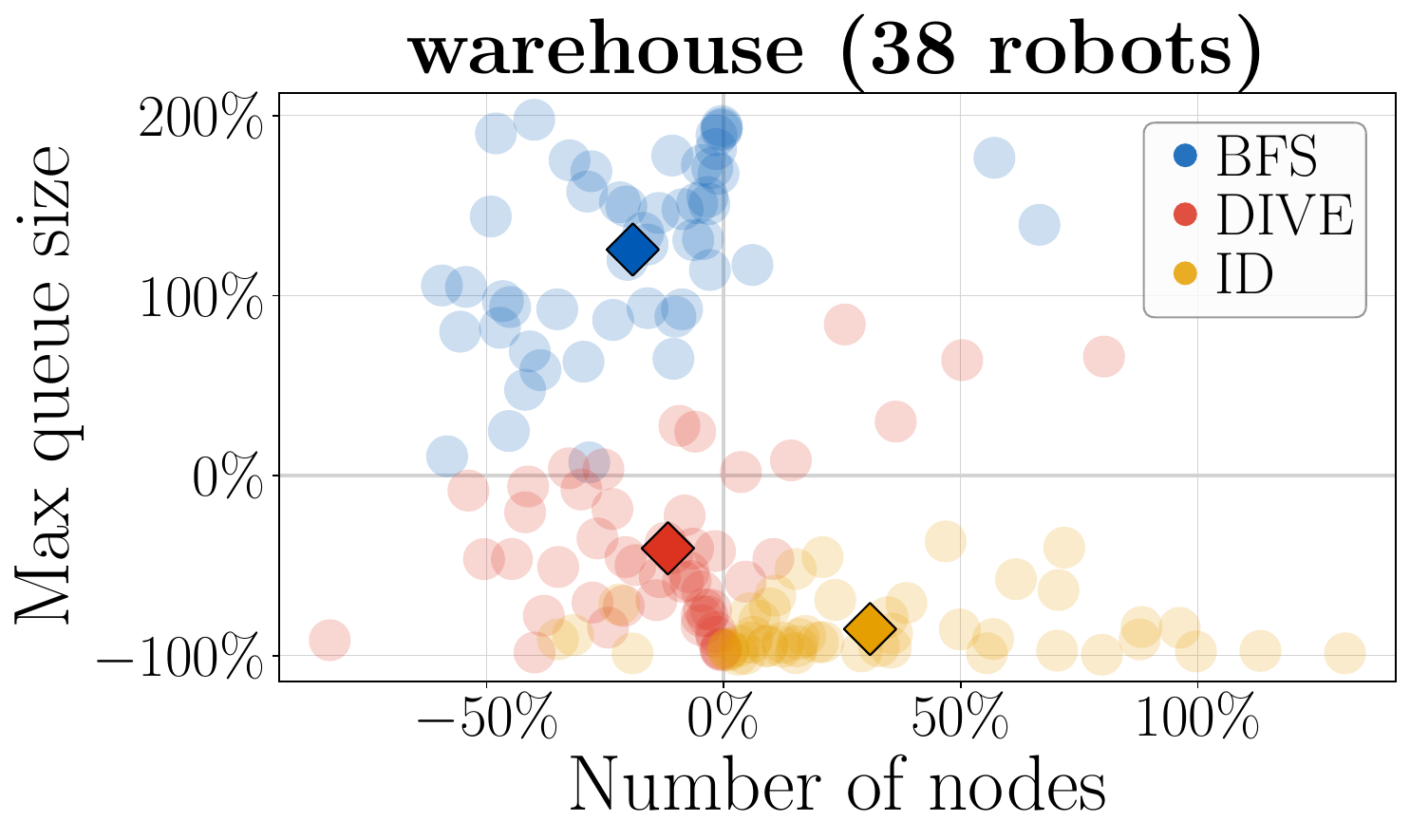}%
    }%
    \caption{Instance-level trade-offs between expanded nodes and maximum queue size on complete instances. 
    Each point is one policy-instance pair, shown as a percentage difference from the per-instance mean across policies (lower is better on both axes). Diamonds mark policy centroids.}
    \label{fig:relative_scatter_plots}
\end{figure}

\Cref{fig:relative_scatter_plots} provides the corresponding instance-level view.
Circles represents individual experiments and stars mark the resulting centroid by policy.
For the metric indicated in its label, each axis measures the percentage difference from the mean across policies on the same instance.
The scatter confirms that the table averages are not caused by a small number of outliers: \gls{acr:bfs} clusters toward lower node count and larger queue, \gls{acr:id} clusters toward smaller queue and larger node count, and \gls{acr:dive} lies between them while moving toward lower queue than \gls{acr:bfs} and lower node count than \gls{acr:id}.

\begin{table}[tb]
\setlength{\abovecaptionskip}{\TableCaptionSkip}
\caption{Number of nodes expanded in each zone as defined in \cref{sec:FormalizedNodeSelection_SolveCompletionCriteria}, normalized by the size of the dual zone\protect\footnotemark.}
\label{tab:zones}
\centering
\renewcommand{\arraystretch}{\TableVerticalStretch}
\setlength{\tabcolsep}{2.5pt}
\begin{NiceTabular}{l|rrr|rrr|rrr}
    \toprule
     \Block{2-1}{Map\! (\# robots)} & \multicolumn{3}{c}{BFS} & \multicolumn{3}{c}{DIVE} & \multicolumn{3}{c}{ID} \\
     & Dual & Opt & Prim & Dual & Opt & Prim & Dual & Opt & Prim \\
    \midrule
    empty\! (51) & 1.00 & 0.09 & 0.00 & 1.00 & 0.08 & 0.01 & 1.14 & 0.12 & 0.00 \\
    random\! (41) & 1.00 & 0.34 & 0.00 & 1.00 & 0.52 & 0.66 & 1.73 & 1.00 & 0.00 \\
    room\! (20) & 1.00 & 0.25 & 0.00 & 1.00 & 0.36 & 0.60 & 2.24 & 0.55 & 0.00 \\
    warehouse\! (38) & 1.00 & 0.63 & 0.00 & 1.00 & 0.58 & 0.40 & 1.48 & 1.29 & 0.00 \\
    \bottomrule
\end{NiceTabular}
\end{table}
\footnotetext{We skip edge cases with a dual zone size of zero. This happens whenever the root value is optimal.}

\Cref{tab:zones} explains where the node overhead comes from. 
By construction, \gls{acr:bfs} resolves the dual zone once and avoids measurable primal-zone work in these complete runs. 
\Gls{acr:id} also avoids the primal zone, but its restarts lead to repeated work in the dual zone and additional work in the optimal zone.
\gls{acr:dive} resolves the dual zone once, like \gls{acr:bfs}, and spends its extra effort in the optimal and primal zones. 
This is precisely the behavior intended in \cref{sec:Dive}, as \gls{acr:dive} pays for dive continuity and incumbents through controlled primal-side exploration, not through repeated restarts or loss of the certificate frontier.

The amount of optimal-zone work varies substantially across maps, which is expected because practical \gls{acr:mapf} instances often have many equal-cost nodes and multiple ways to realize the same optimal cost. 
Nevertheless, the qualitative pattern matches the search-tree analysis in \cref{sec:TechnicalPolicyAnalysis}, showing that \gls{acr:bfs} is node-efficient, \gls{acr:id} is memory-efficient, and \gls{acr:dive} is dive-efficient.

\subsection{Incumbents and Anytime Behavior}
\label{sec:Experiments_Anytime}

Unlike \gls{acr:bfs} or \gls{acr:id}, \gls{acr:dive} finds intermediate feasible solutions before the completion of the solve.
Thus, \gls{acr:dive} turns \gls{acr:cbs} into an anytime algorithm.
We analyze incumbents provided by \gls{acr:dive} at two stages of the solve: the first incumbent and the incumbent at the time of BFS completion.
These points in the solve show that \gls{acr:dive} finds incumbents \emph{quickly} and of \emph{high quality}.

\begin{table}[tb]
\setlength{\abovecaptionskip}{\TableCaptionSkip}
\caption{DIVE's first incumbent on nontrivial runs. 
Truncated runs are included.
Incumbent found is the percentage of runs in which DIVE found a feasible solution before truncation.}
\label{tab:incumbents_first}
\centering
\renewcommand{\arraystretch}{\TableVerticalStretch}
\setlength{\tabcolsep}{2.75pt}
\begin{NiceTabular}{lc|ccccc}
    \toprule
    Map & \# robots & Inc found & Nodes & Seconds & Primal$-$Dual & Gap \\
    \midrule
    \multirow[c]{4}{*}{empty} & 31 & 100\% & 51 & 0.2 & 1.0 & 0.2\% \\
     & 51 & 100\% & 92 & 1.4 & 2.0 & 0.2\% \\
     & 102 & 100\% & 483 & 29.4 & 9.3 & 0.4\% \\
     & 154 & 100\% & 1465 & 191.9 & 35.6 & 1.1\% \\
    \midrule

    \multirow[c]{4}{*}{random} & 25 & 100\% & 44 & 0.5 & 5.2 & 0.9\% \\
     & 41 & 100\% & 74 & 1.1 & 14.3 & 1.5\% \\
     & 82 & 76\% & 542 & 58.1 & 84.6 & 4.2\% \\
     & 123 & 20\% & 1636 & 178.8 & 257.5 & 8.7\% \\
    \midrule

    \multirow[c]{4}{*}{room} & 20 & 91\% & 39 & 0.2 & 19.3 & 3.6\% \\
     & 34 & 64\% & 178 & 5.7 & 72.2 & 7.2\% \\
     & 68 & 4\% & 1264 & 216.9 & 333.0 & 16.0\% \\
     & 102 & 0\% & - & - & - & - \\
    \midrule

    \multirow[c]{4}{*}{warehouse} & 38 & 100\% & 59 & 0.9 & 5.3 & 0.5\% \\
     & 64 & 100\% & 185 & 9.5 & 21.4 & 1.1\% \\
     & 128 & 32\% & 1274 & 218.3 & 137.8 & 3.6\% \\
     & 192 & 0\% & - & - & - & - \\
    \bottomrule
\end{NiceTabular}
\end{table}

\Cref{tab:incumbents_first} reports \gls{acr:dive}'s first incumbent on all nontrivial runs, including runs that are later truncated. 
When an incumbent is found, it is almost always found in the first dive: the mean number of dive breaks is 1.0 in every row with successful incumbents. 
For the complete benchmark densities, this first solution appears after only tens to hundreds of high-level nodes.
The denser 10\% and 15\% cases are harder, especially on \texttt{room-32-32-4} and \texttt{warehouse-small}, but the table shows the expected degradation rather than a hidden failure mode, i.e., as feasible nodes become rarer near the top of the primal region, \gls{acr:dive} needs more nodes and sometimes does not find an incumbent before the five-minute cutoff.

The quality of the first incumbent is summarized by the absolute primal-dual difference $z^{\mathrm P}-z^{\mathrm D}$ and the relative gap. 
These are worst-case certificates, not estimates. 
Even when the first incumbent is not optimal, it comes with a bound on how much improvement remains possible.

\begin{figure}[tb]
    \centering
    \subfloat{%
        \includegraphics[width=0.5\linewidth]{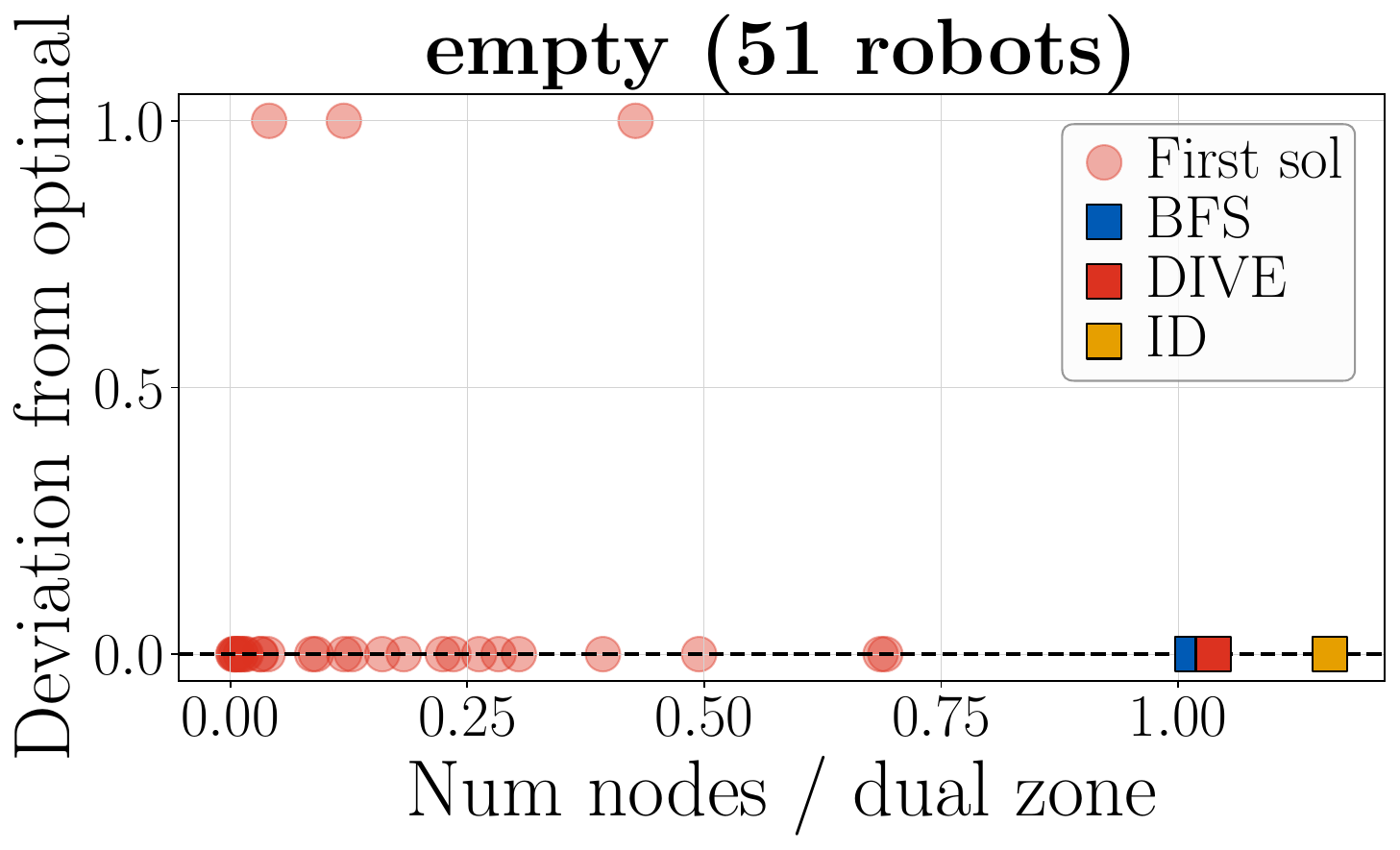}%
    }%
    \subfloat{%
        \includegraphics[width=0.5\linewidth]{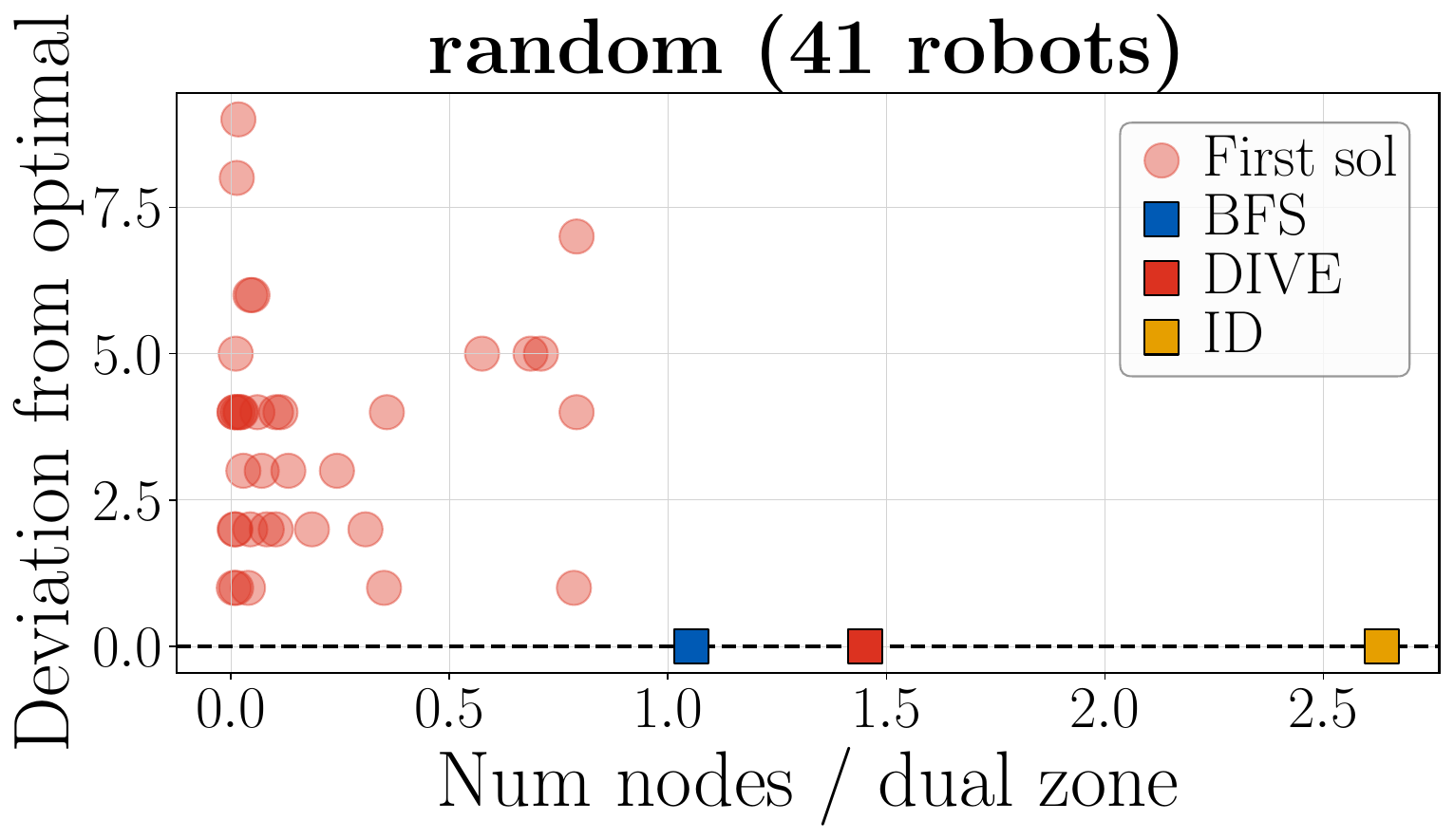}%
    }\\
    \subfloat{%
        \includegraphics[width=0.5\linewidth]{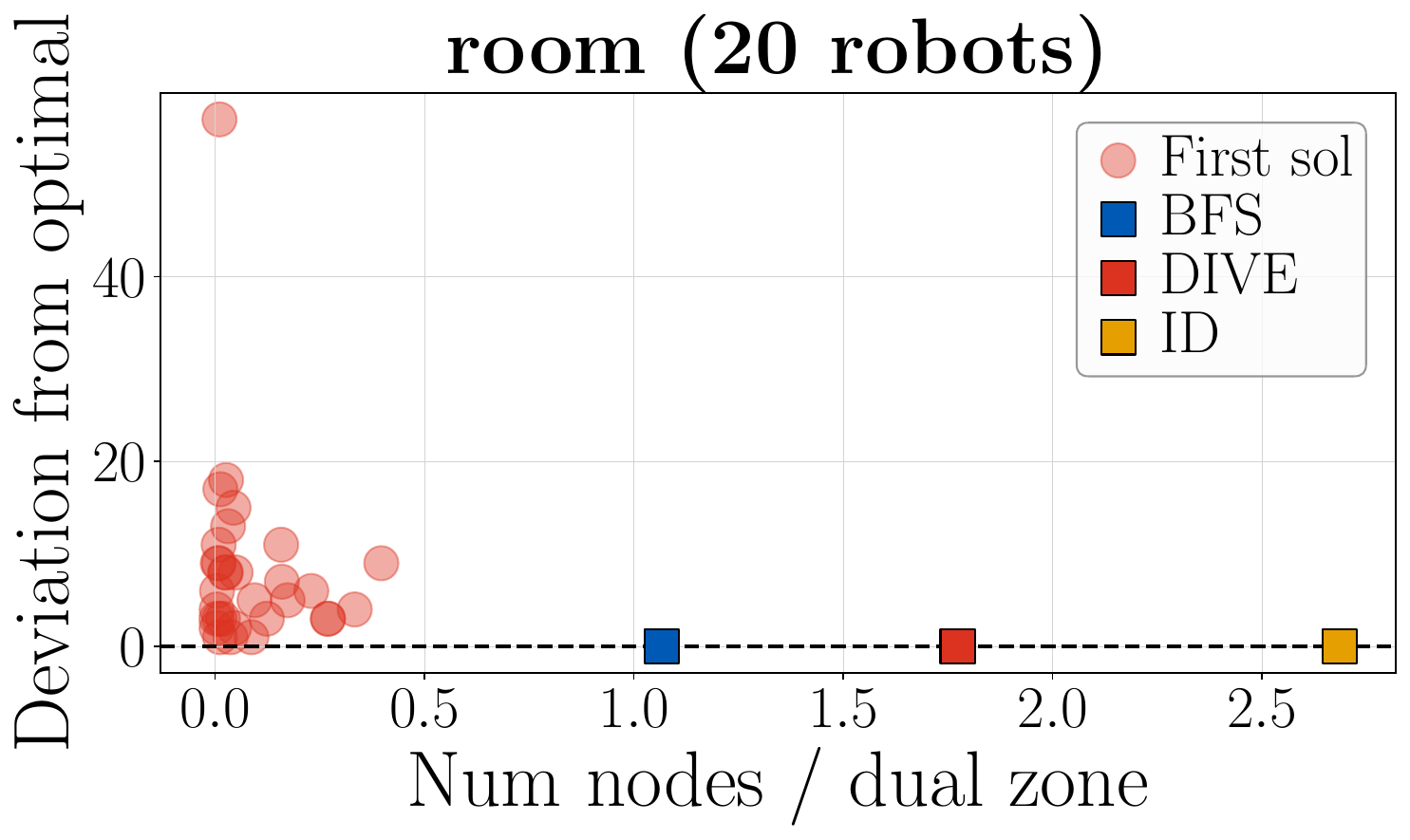}%
    }%
    \subfloat{%
        \includegraphics[width=0.5\linewidth]{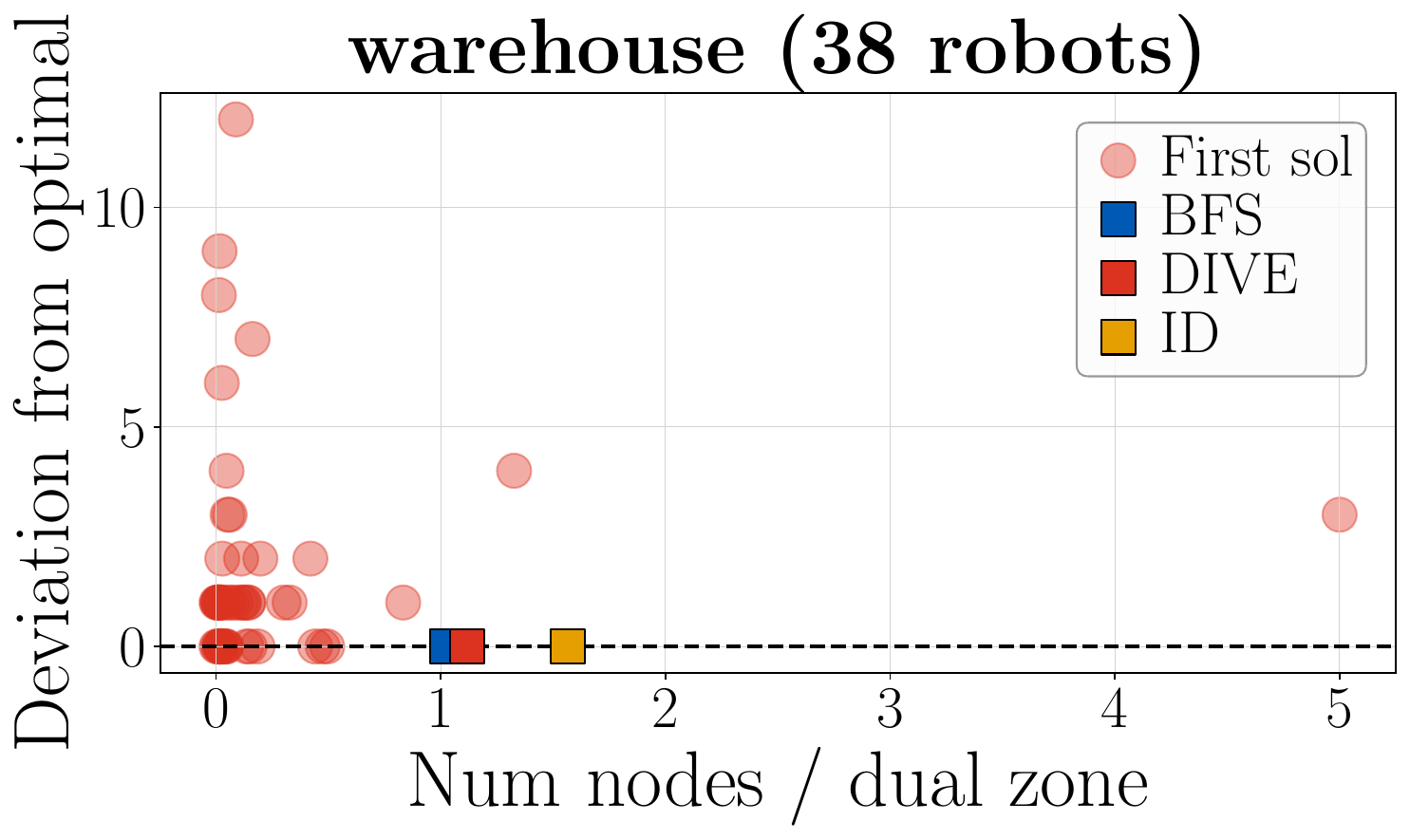}%
    }%
    \caption{First incumbent found by \gls{acr:dive} on complete instances. 
    The horizontal axis is normalized by the dual-zone size. 
    Square markers show the median node budget required by each policy to finish.}
    \label{fig:first_incumbents}
\end{figure}

\Cref{fig:first_incumbents} makes the timing visible. 
On the complete instances, \gls{acr:dive}'s first incumbents are typically found far before any policy reaches its final certificate. 
On \texttt{empty-32-32}, the symmetry of the map creates many equal-cost optimal solutions, so the first dive frequently already reaches an optimal solution.

\begin{table}[tb]
\setlength{\abovecaptionskip}{\TableCaptionSkip}
\caption{Best incumbent found by DIVE after the same number of high-level expansions that \gls{acr:bfs} required to complete the corresponding solve.}
\label{tab:incumbents_bfs_completion}
\centering
\renewcommand{\arraystretch}{\TableVerticalStretch}
\setlength{\tabcolsep}{2.5pt}
\begin{NiceTabular}{l|ccccc}
\toprule
Map\! (\# robots) & Inc found & Primal$-$Dual & Gap & Primal$-$Opt & Inc is Opt \\
\midrule
empty\! (51) & 100\% & 0.5 & 0.0\% & 0.0 & 100\% \\
random\! (41) & 100\% & 2.5 & 0.3\% & 0.9 & 58\% \\
room\! (20) & 100\% & 2.9 & 0.6\% & 1.4 & 47\% \\
warehouse\! (38) & 100\% & 1.1 & 0.1\% & 0.4 & 80\% \\
\bottomrule
\end{NiceTabular}
\end{table}

\Cref{tab:incumbents_bfs_completion} evaluates \gls{acr:dive} at a second reference point, i.e., the number of node expansions required by \gls{acr:bfs} to finish the same instance. 
At this budget, \gls{acr:dive} has found an incumbent in every complete instance. 
The average certified gap is at most 0.6\%, and \gls{acr:dive} has already found an optimal solution in 46--100\% of instances, depending on the map. 
Thus, even when \gls{acr:dive} has not yet completed the dual certificate, it often already has the final solution and only lacks proof.

This is the practical difference between \gls{acr:dive} and the two baselines. 
\Gls{acr:bfs} and \gls{acr:id} concentrate on the certificate side and provide little pre-termination incumbent information. 
\gls{acr:dive} instead exposes a meaningful anytime mode for optimal \gls{acr:cbs}: after the first incumbent, the solver can be interrupted and return both a feasible solution and a certified gap.

\subsection{Primal-Dual Bound Progression}
\label{sec:Experiments_Bounds}

\begin{figure}[tb]
    \centering
    \subfloat{%
        \includegraphics[width=0.5\linewidth]{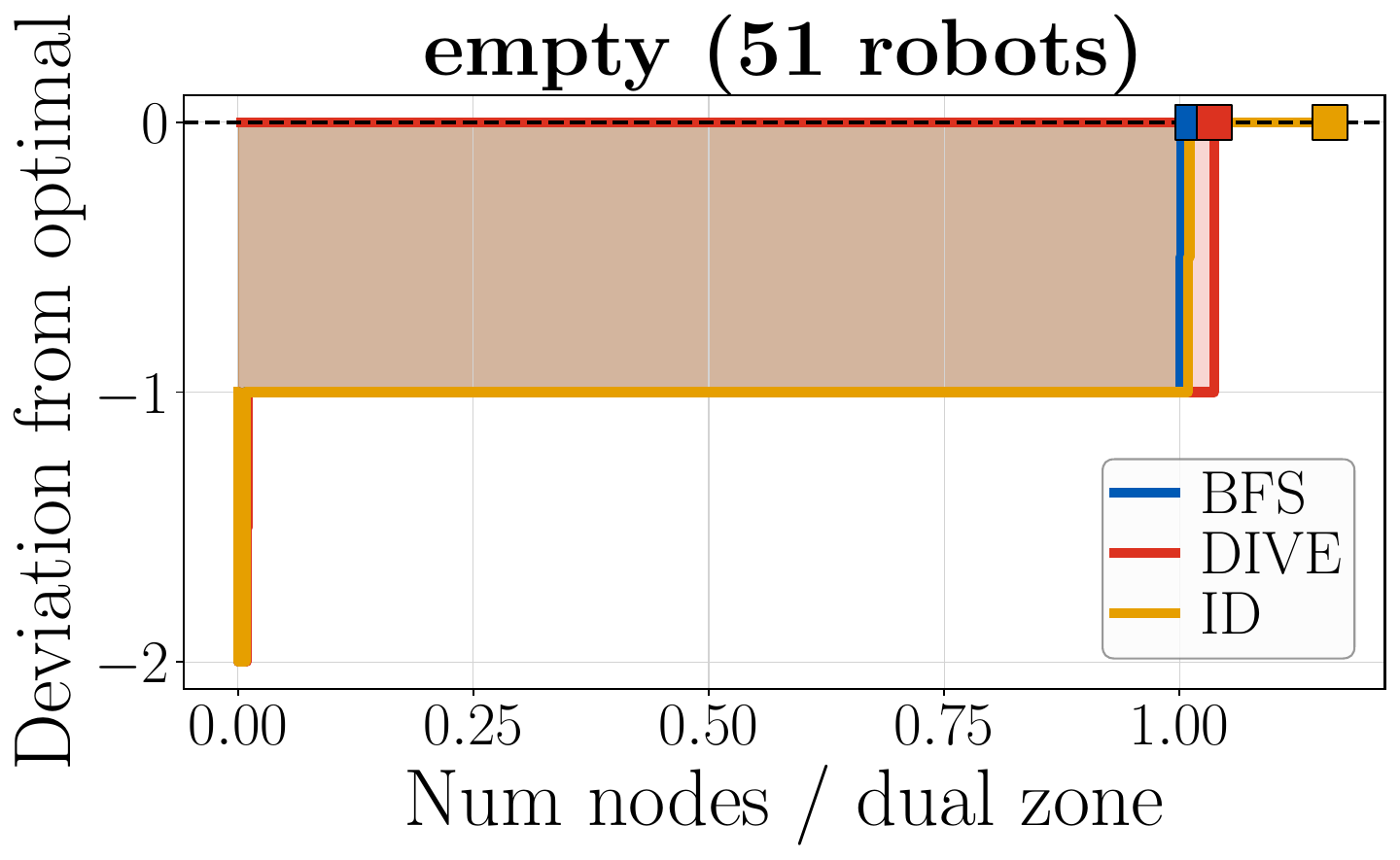}%
    }%
    \subfloat{%
        \includegraphics[width=0.5\linewidth]{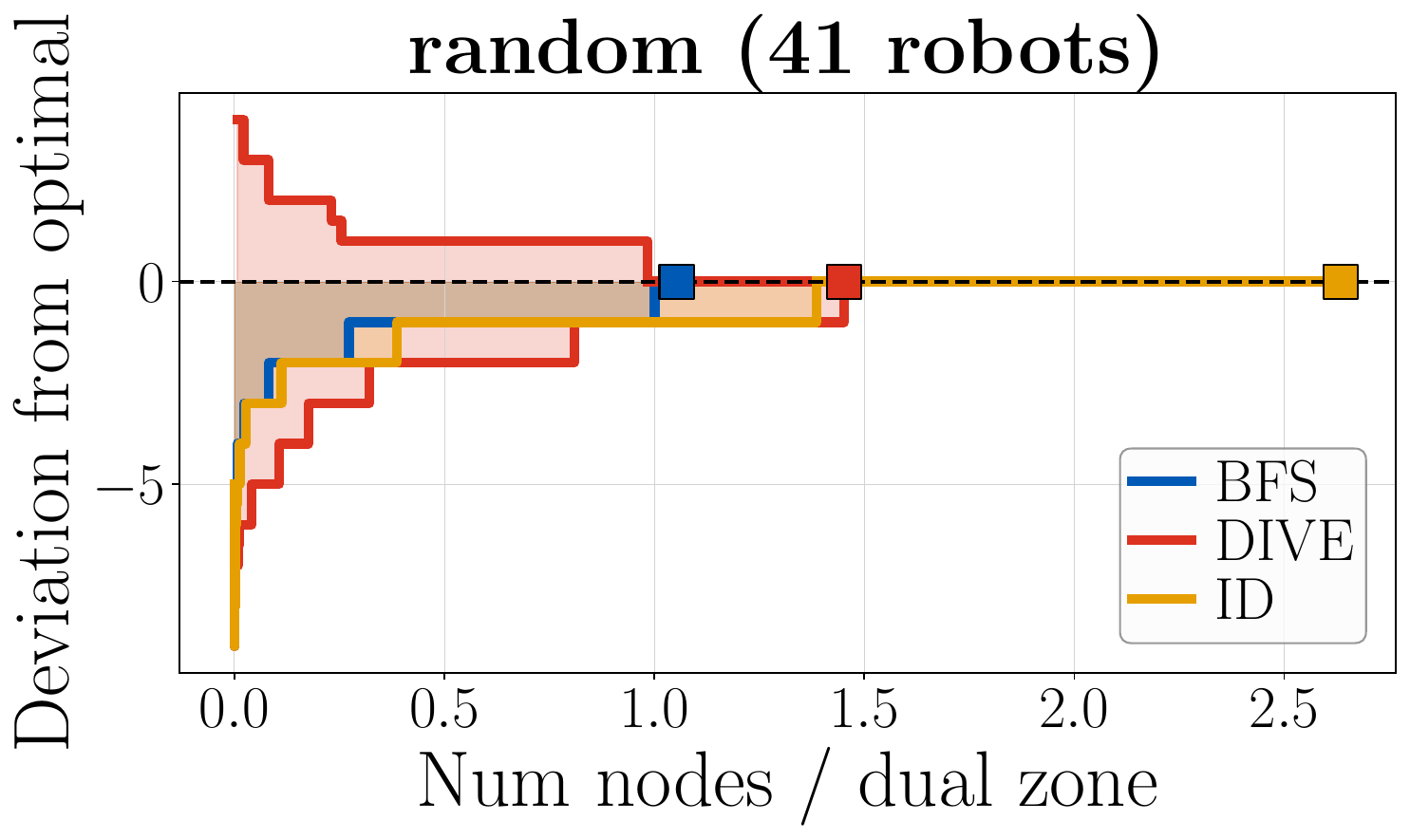}%
    }\\
    \subfloat{%
        \includegraphics[width=0.5\linewidth]{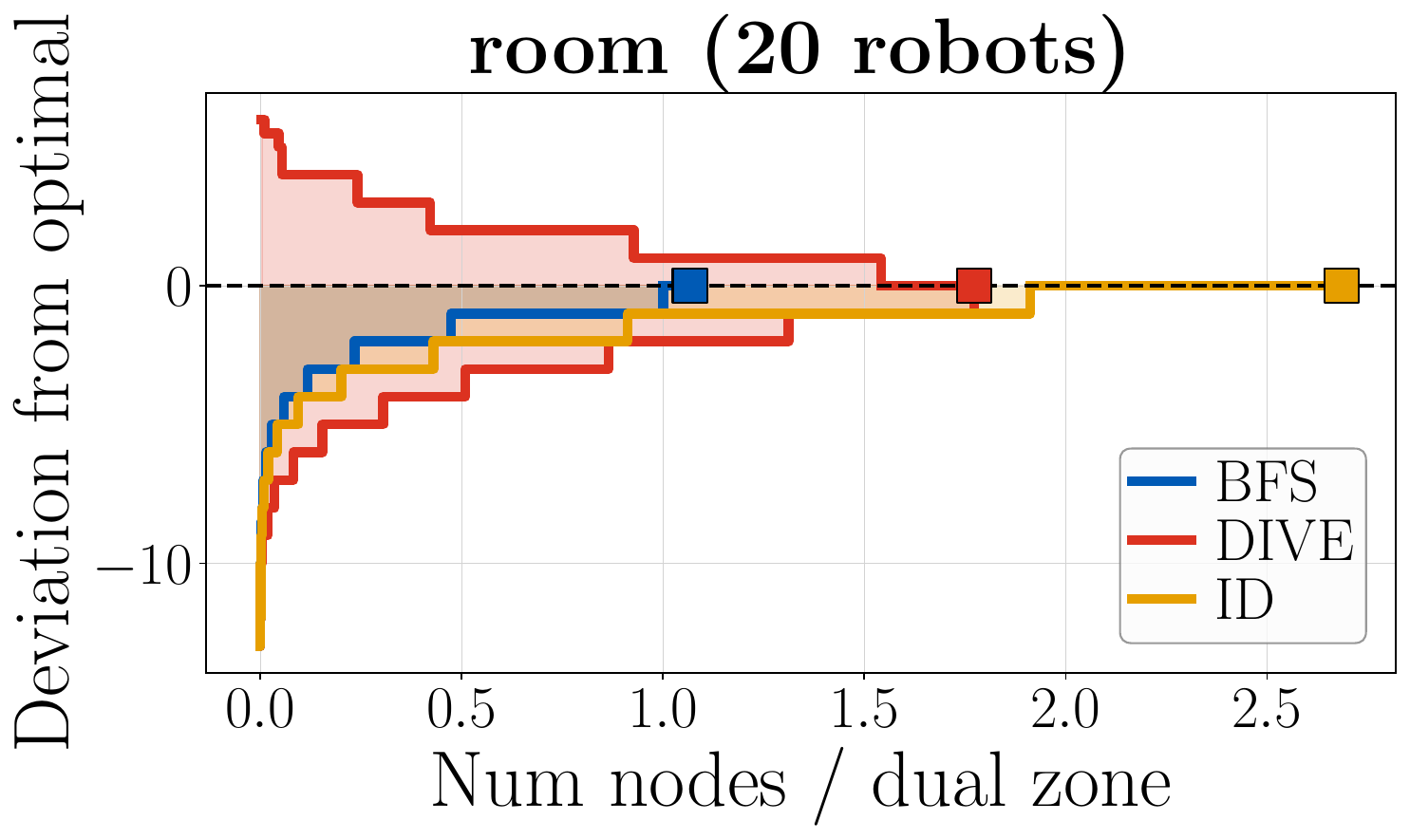}%
    }%
    \subfloat{%
        \includegraphics[width=0.5\linewidth]{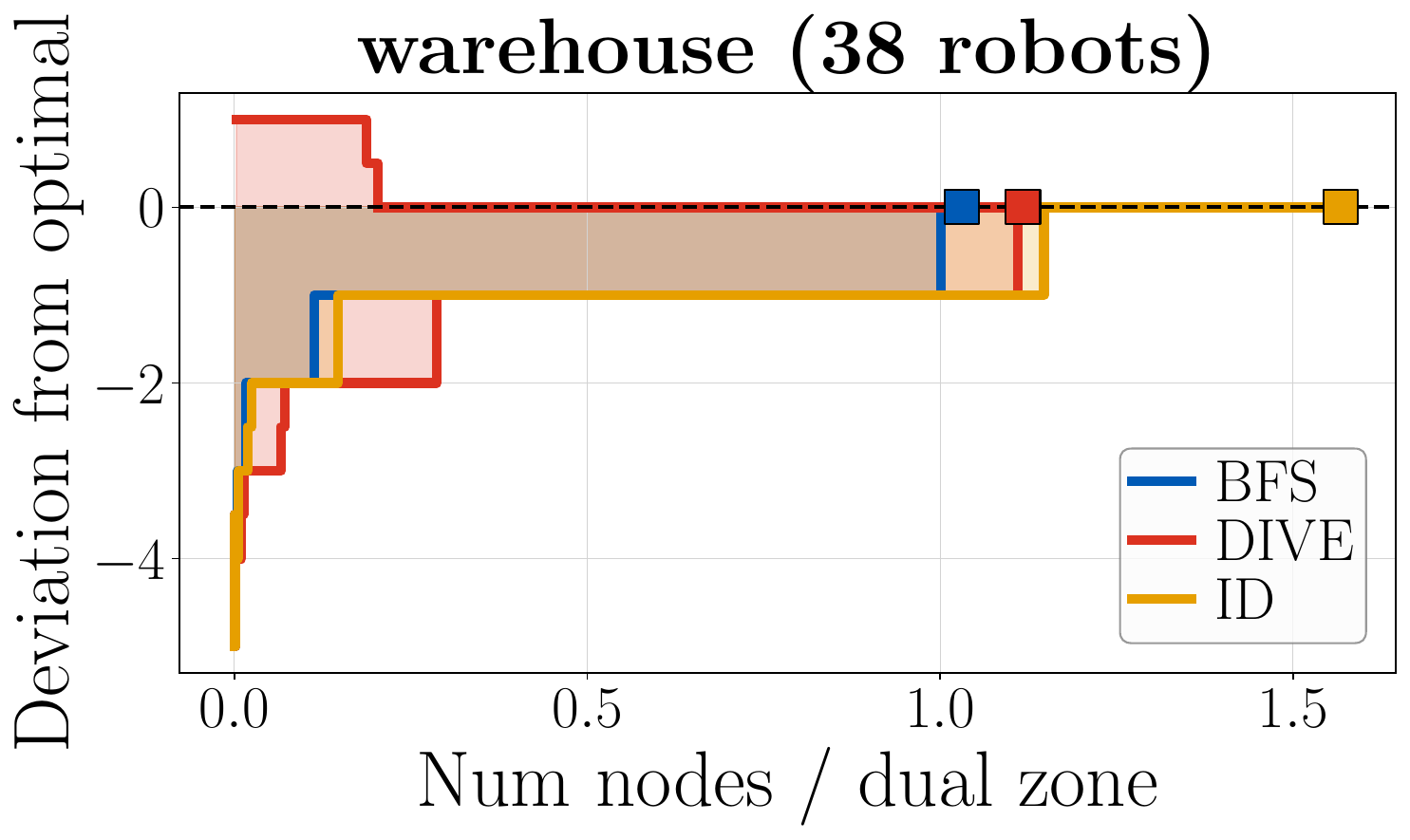}%
    }%
    \caption{Median primal and dual bound progression on complete instances. 
    The horizontal axis is normalized by the dual-zone size. 
    Lower curves indicate smaller remaining deviation from the optimal value.}
    \label{fig:bounds}
\end{figure}

\Cref{fig:bounds} shows the median primal and dual trajectories\footnote{The median primal curve is computed over runs that have already found an incumbent at the given node budget. 
Runs with no incumbent are added to the median only after their first feasible solution appears. 
Because late first incumbents can be worse than earlier incumbents from other runs, a raw median primal curve can increase. 
We therefore display its monotone envelope, which matches the fact that each individual run's incumbent value is nonincreasing over time. 
All tables use ordinary arithmetic means over the relevant instance set.}{}.
The dual curves confirm the expected ordering from \cref{sec:TechnicalPolicyAnalysis}, i.e., \gls{acr:bfs} closes the dual bound most directly, \gls{acr:id} initially follows best-first progress but slows as restarts accumulate, and \gls{acr:dive} remains close to the best-bound frontier because every dive reanchors after termination. 
The primal curves show the missing dimension, as \gls{acr:dive} establishes a feasible incumbent early and then alternates between improving the incumbent and closing the certificate gap.

The plots also clarify why \gls{acr:dive} can be useful even when it expands more nodes than \gls{acr:bfs}. 
A tight dual bound alone is not a usable plan; an incumbent alone is not a proof. 
\gls{acr:dive} intentionally works on both sides of the gap. 
In many runs, it finds the optimal solution before the dual bound reaches $z^*$, so the remaining work is purely certification.

\subsection{Warm-Starts}
\label{sec:Experiments_WarmStart}
Warm-starts can further reduce \gls{acr:dive}'s node count and queue size by by pruning additional primal zone nodes.
\Cref{fig:warm_start} highlights \gls{acr:dive}'s responsiveness to tighter warm start values on \texttt{room-32-32-4}, the map that proved most challenging for \gls{acr:dive} to find its own first incumbents.
Warm-starts may cut off dives earlier but maintain the same starting nodes.
Thus, the number of dive breaks remains consistent, and optimal zone expansions only reduce if the warm-start value is also optimal.

In general, \cref{sec:Experiments_Anytime} shows that \gls{acr:dive} is adept at quickly finding quality incumbents by itself, but warm-starts provide an effective refinement.
In dense instances where feasible solutions are rare near the surface of the primal region, warm-starts become a vital safeguard that allow \gls{acr:dive} to scale to hyper-dense instances. We add \texttt{test-5-5} at 33\% occupancy as a boundary-case stress test rather than as part of the main benchmark suite.

All policies are evaluated with and without the same \gls{acr:pibt} warm start. 
A warm start can help every exact policy, because any open node whose lower bound is no better than the warm-start value can be pruned. 
However, the effect is especially important for \gls{acr:dive}, which intentionally enters deeper regions and therefore benefits strongly from having incumbent pruning active before the first dive.

Without a warm start, \gls{acr:dive} solves none of the 40 dense instances within the five-minute cutoff. 
With the \gls{acr:pibt} incumbent, \gls{acr:dive} solves 35 out of 40 instances, matching \gls{acr:bfs} and solving one more instance than \gls{acr:id}. 
The structural trade-offs of the warm-started complete runs are shown in \cref{fig:warm_start}.

\begin{table*}[tb]
\setlength{\abovecaptionskip}{\TableCaptionSkip}
\caption{Warm-started dense-instance performance on \texttt{test-5-5}. 
Values are mean \% differences relative to \gls{acr:bfs} (lower is better).}
\label{tab:dense}
\centering
\renewcommand{\arraystretch}{\TableVerticalStretch}
\begin{NiceTabular}{lcc|rrr|rrr|rrr}
    \toprule
     Map & & \# robots & \multicolumn{3}{c}{Num nodes} & \multicolumn{3}{c}{Num dive breaks} & \multicolumn{3}{c}{Max queue size} \\
     & &  & BFS & DIVE & ID & BFS & DIVE & ID & BFS & DIVE & ID \\
    \midrule
    test-5-5
        & \parbox[c]{0.5cm}{\centering\includegraphics[width=0.5cm]{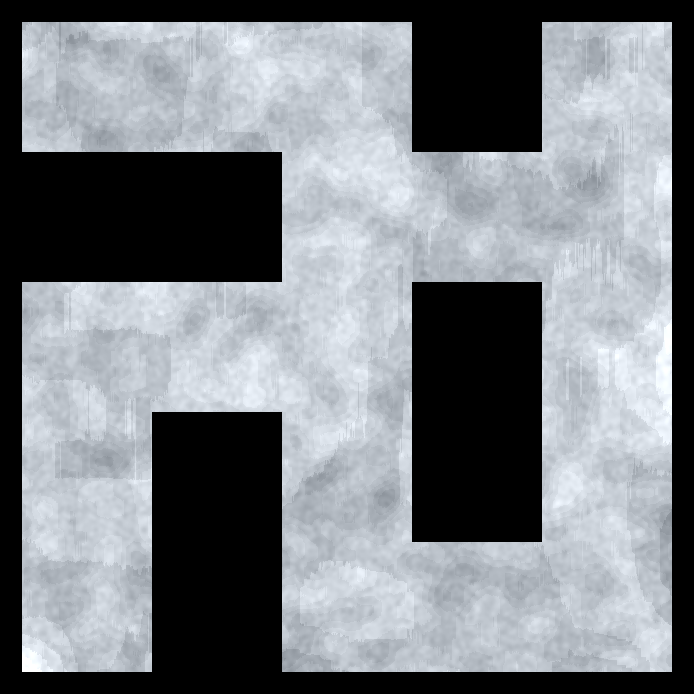}}
        & 6 (33\%) & 0\% & 104\% & 127\% & 0\% & -46\% & 35\% & 0\% & -10\% & -99\% \\
    \bottomrule
\end{NiceTabular}
\end{table*}

\begin{figure}[tb]
    \centering
    \subfloat{%
        \includegraphics[width=0.5\linewidth]{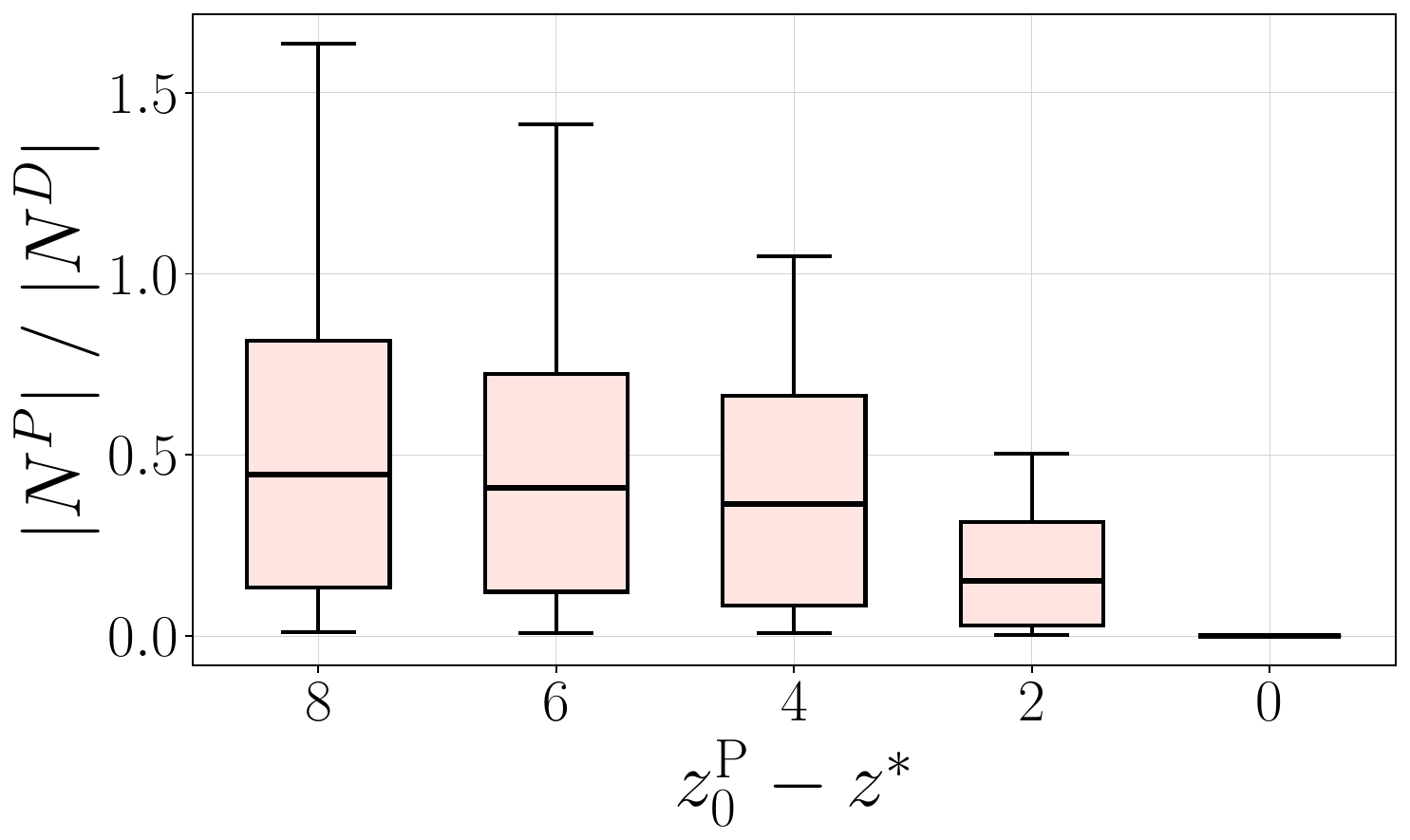}%
    }%
    \subfloat{%
        \includegraphics[width=0.5\linewidth]{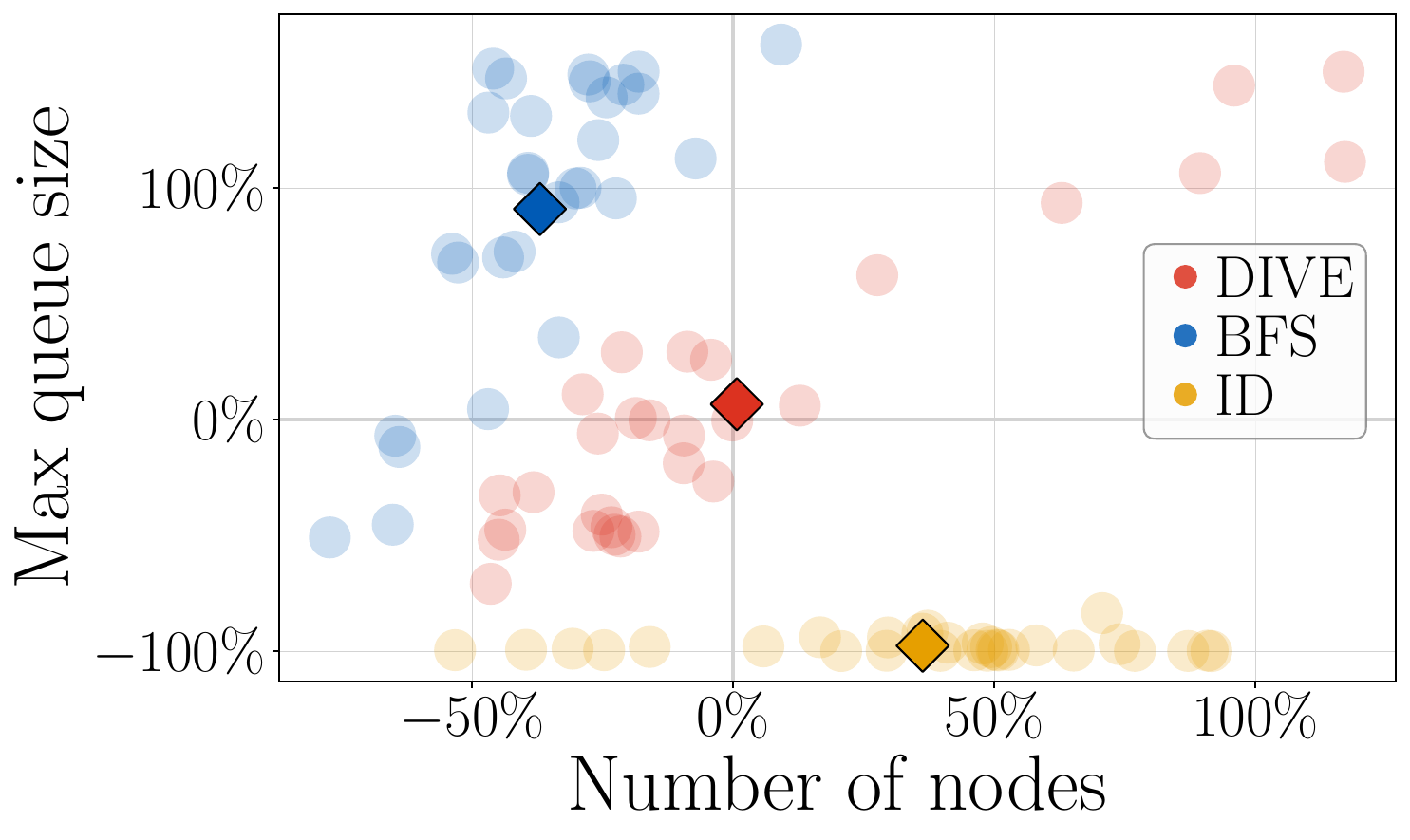}%
    }%
    \caption{Left: Primal zone expansions of eight instances on \texttt{room-32-32-4} with 20 robots for progressively tighter warm-start values. Right: Relative performance on the dense \texttt{test-5-5} warm-start experiment with 33\% robot occupancy.}
    \label{fig:warm_start}
\end{figure}

The warm-started dense experiment preserves the same qualitative ordering as the main benchmarks, as shown in~\cref{tab:dense}. 
\gls{acr:dive} requires more node expansions than \gls{acr:bfs}, but it cuts dive breaks by 46\% and reduces maximum queue size. 
\Gls{acr:id} remains the memory extreme, reducing maximum queue by 99\% from \gls{acr:bfs}, but it pays in both nodes and dive breaks. 
Thus, even in the regime where \gls{acr:dive} needs help from an external incumbent, it returns to the same trade-off profile once the incumbent is available.

On lower-density maps, warm starts are less important because \gls{acr:dive} usually finds a better incumbent on its own during the first dive. 
Outside the dense \texttt{test-5-5} stress test, only a handfulof instances that timed out without a warm start completed with the \gls{acr:pibt} warm start. 
This supports the intended interpretation that warm starts are not needed for \gls{acr:dive}'s main benchmark behavior, but they are a useful safety mechanism for dense regimes.

\subsection{Ablation: Diversified Search}
\label{sec:Experiments_Ablation}

The next experiment isolates diversified search from incumbent pruning. 
We compare \gls{acr:dive} with an incumbent-pruned depth-first policy that uses the same incumbent cutoff from \cref{prop:DIVE_incumbent_cutoff}, but does not restart each dive from the global best-bound frontier. 
In \Cref{fig:dfs,tab:dfs}, this ablated policy is labeled \gls{acr:dfs}. 

\begin{figure}[tb]
    \centering
    \subfloat{%
        \includegraphics[width=0.5\linewidth]{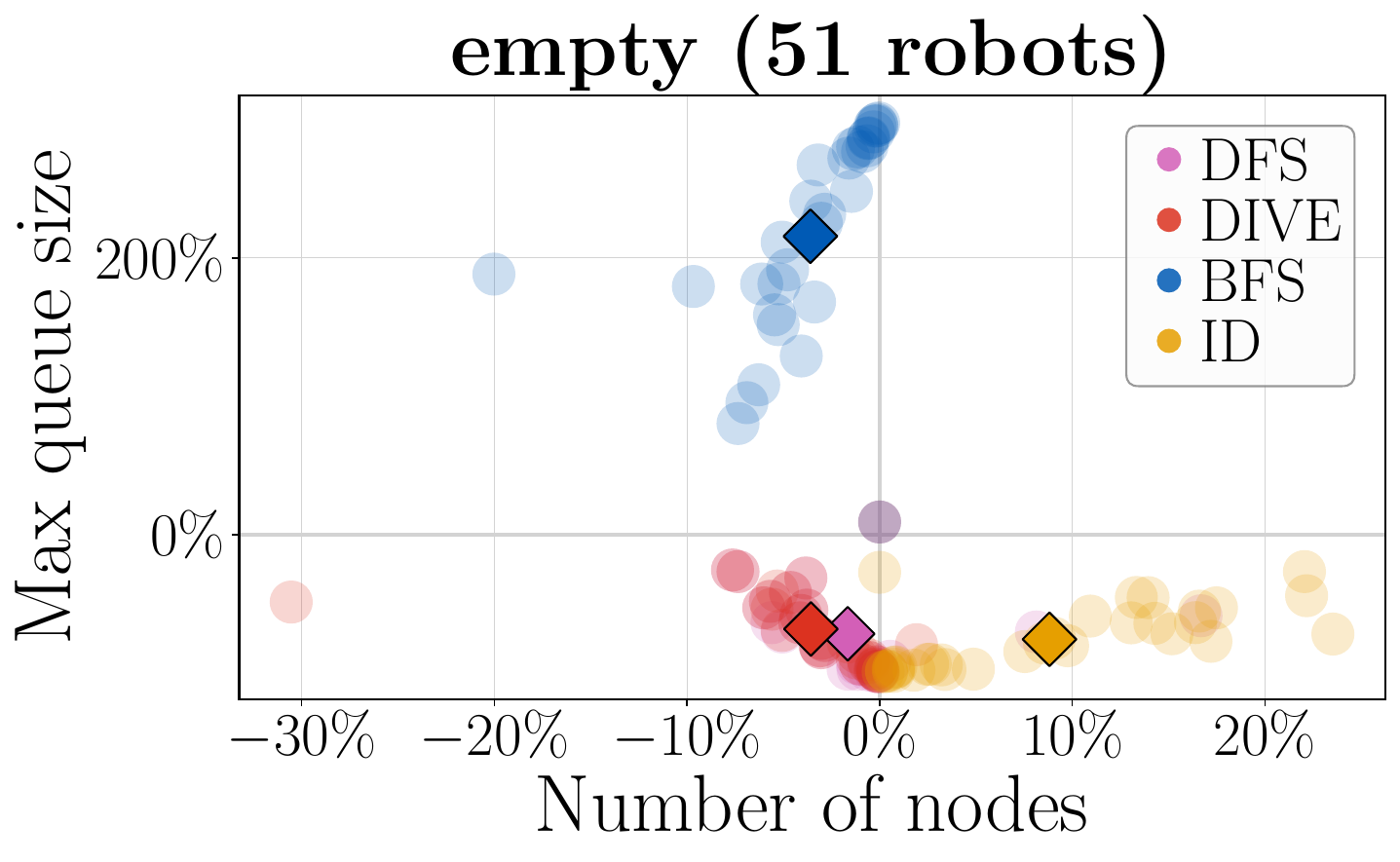}%
    }%
    \subfloat{%
        \includegraphics[width=0.5\linewidth]{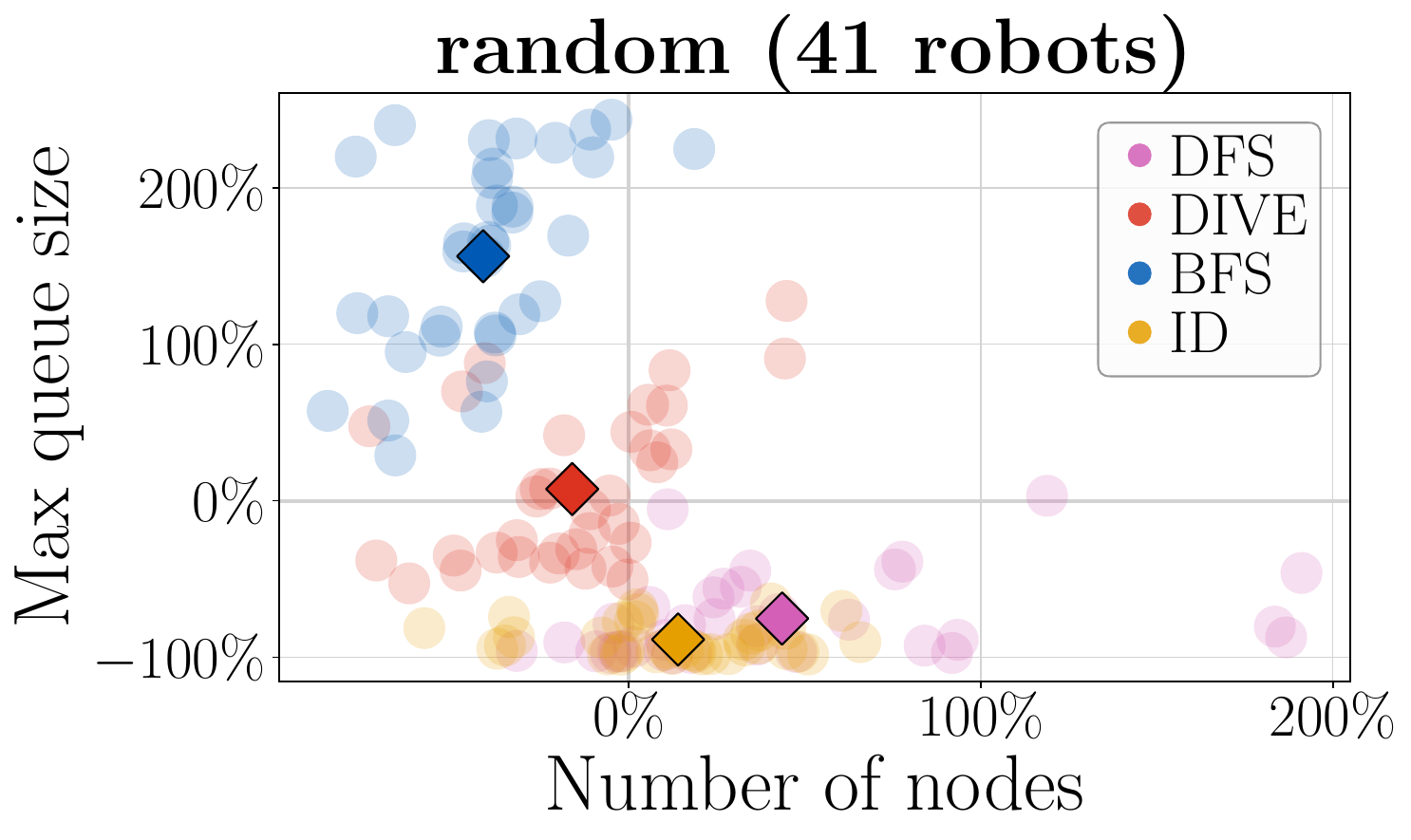}%
    }\\
        \subfloat{%
        \includegraphics[width=0.5\linewidth]{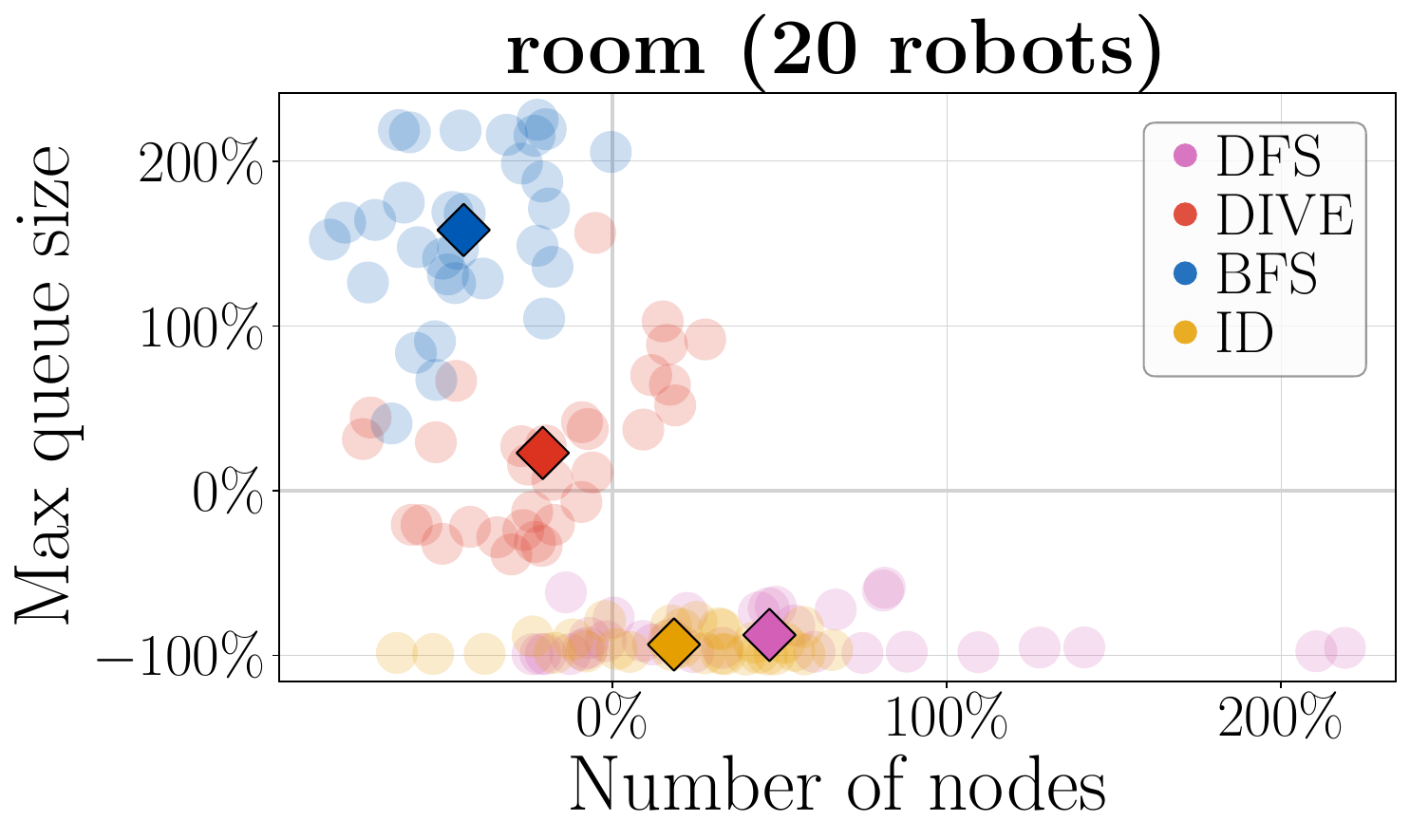}%
    }%
    \subfloat{%
        \includegraphics[width=0.5\linewidth]{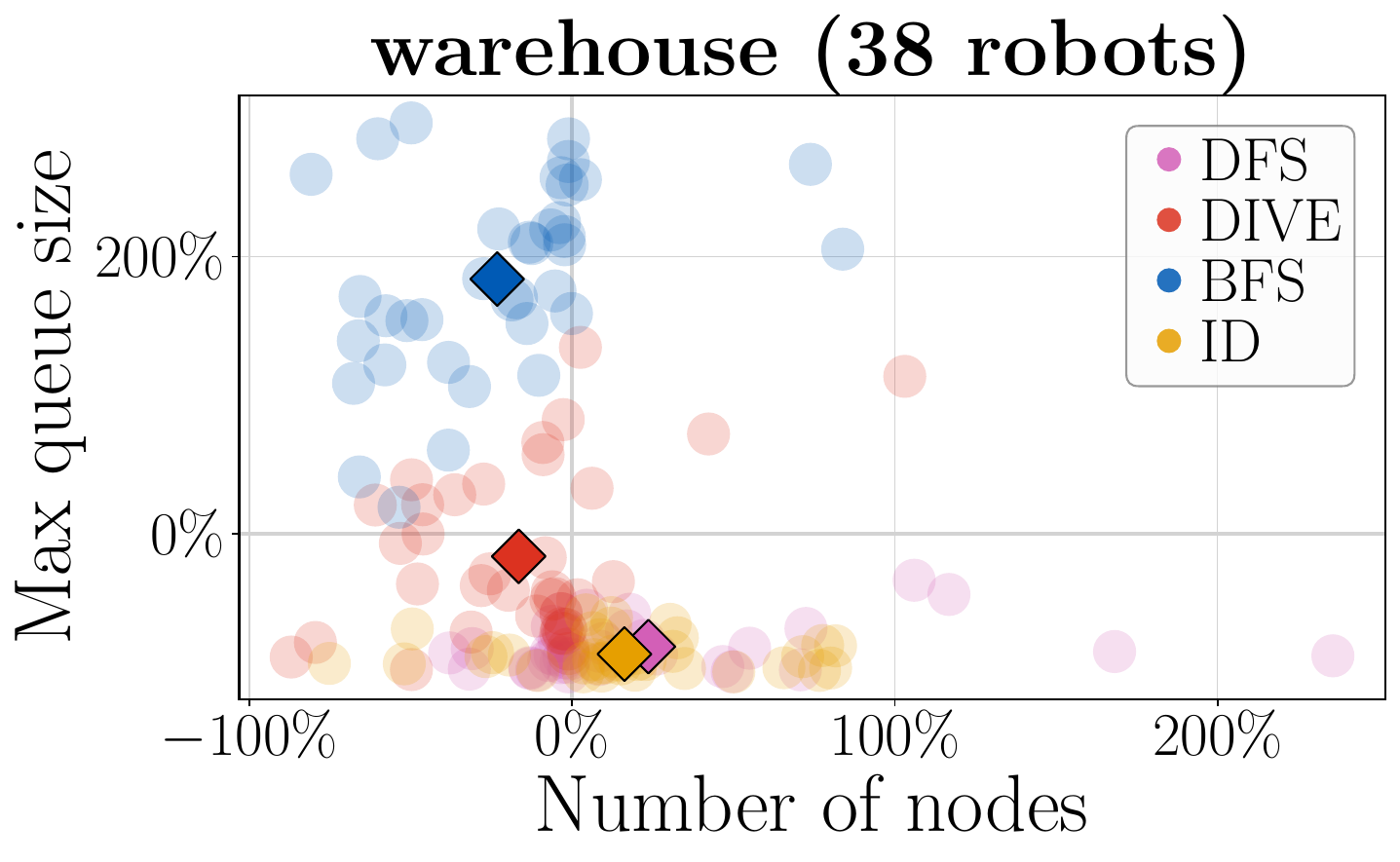}%
    }%
    \caption{Diversified search ablation: The \gls{acr:dfs} label denotes an incumbent-pruned depth-first policy without diversified best-bound restarts.}
    \label{fig:dfs}
\end{figure}

\begin{table}[tb]
\setlength{\abovecaptionskip}{\TableCaptionSkip}
\caption{Diversified search ablation: The \gls{acr:dfs} label denotes an incumbent-pruned depth-first policy without diversified best-bound restarts. Values are mean percentage differences relative to \gls{acr:dive} (lower is better).}
\label{tab:dfs}
\centering
\renewcommand{\arraystretch}{\TableVerticalStretch}
\setlength{\tabcolsep}{6pt}
\begin{NiceTabular}{l|rr|rr|rr}
    \toprule
     \Block{2-1}{Map\! (\# robots)} & \multicolumn{2}{c}{\# nodes} & \multicolumn{2}{c}{\# dive breaks} & \multicolumn{2}{c}{Max queue size} \\
     & DIVE & DFS & DIVE & DFS & DIVE & DFS \\
    \midrule
    empty\! (51) & 0\% & 3\% & 0\% & 6\% & 0\% & -16\%  \\
    random\! (41) & 0\% & 126\% & 0\% & 209\% & 0\% & -75\%  \\
    room\! (20) & 0\% & 157\% & 0\% & 244\% & 0\% & -88\%  \\
    warehouse\! (38) & 0\% & 132\% & 0\% & 187\% & 0\% & -69\%  \\
    \bottomrule
\end{NiceTabular}
\end{table}

\begin{table}[tb]
\setlength{\abovecaptionskip}{\TableCaptionSkip}
\caption{Number of nodes expanded in each zone as defined in \cref{sec:FormalizedNodeSelection_SolveCompletionCriteria}, normalized by the size of the dual zone. \gls{acr:dfs} denotes the incumbent-pruned depth-first policy without diversified best-bound restarts.}
\label{tab:dfs_zones}
\centering
\renewcommand{\arraystretch}{\TableVerticalStretch}
\setlength{\tabcolsep}{5pt}
\begin{NiceTabular}{l|rr|rr|rr}
    \toprule
     \Block{2-1}{Map\! (\# robots)} & \multicolumn{2}{c}{Dual} & \multicolumn{2}{c}{Optimal} & \multicolumn{2}{c}{Primal} \\
     &  DIVE & DFS & DIVE & DFS & DIVE & DFS \\
    \midrule
    empty\! (51) &       1.00 & 1.00 & 0.08 & 0.12   & 0.01 & 0.01  \\
    random\! (41) &   1.00 & 1.00 & 0.54 & 0.96 & 0.68 & 2.70  \\
    room\! (20) &      1.00 & 1.00 & 0.37 & 0.54 & 0.59 & 3.45  \\
    warehouse\! (38) &   1.00 & 1.00 & 0.75 & 1.15 & 0.53 & 2.77  \\
    \bottomrule
\end{NiceTabular}
\end{table}

The ablation shows that incumbent pruning alone is not enough. 
Without diversified search, the depth-first policy requires substantially more expanded nodes and dive breaks than \gls{acr:dive} on average across each of our studied maps.
The reason is visible in the zone accounting summarized in \cref{tab:dfs_zones}.
Work in the optimal zone and in particular the primal zone increases considerably  under the depth-first ablation.

Thus, \gls{acr:dive}'s robustness is not merely a consequence of cutting off dives once an incumbent exists. 
The best-bound restart after each dive is the mechanism that prevents repeated commitment to an unproductive primal-region branch while preserving regular progress on the dual certificate.

\subsection{Responsive enhancements}
\label{sec:Experiments_Responsive}

Finally, we evaluate \gls{acr:mcd}, introduced in \cref{sec:FormalizedNodeSelection_Feedback}, as a minimal example of responsive node selection. 
\Gls{acr:mcd} behaves like \gls{acr:dive} while the queue is below a soft threshold and switches to depth-first behavior once the queue reaches that threshold. 
The goal is not to tune a final solver variant, but to demonstrate that the node-selection trade-off can be closed-loop rather than fixed before the solve begins.

We test \gls{acr:mcd} on \texttt{room-32-32-4} and \texttt{warehouse-small}  with a soft queue threshold of 50 nodes\footnote{\gls{acr:mcd} and \gls{acr:dfs} produce nearly identical statistics for simple instances, motivating our focus on our two most difficult maps.}{}.
Respectively, the average \gls{acr:dive} maximum queue size was 1025 and 171 among the 29 and 33 runs where \gls{acr:dive} exceeded this threshold.
With \gls{acr:mcd}, all of those runs stayed within queue size 60, with the exception of one outlier across the two maps.

\begin{figure}[tb]
    \centering
    \subfloat{%
        \includegraphics[width=0.5\linewidth]{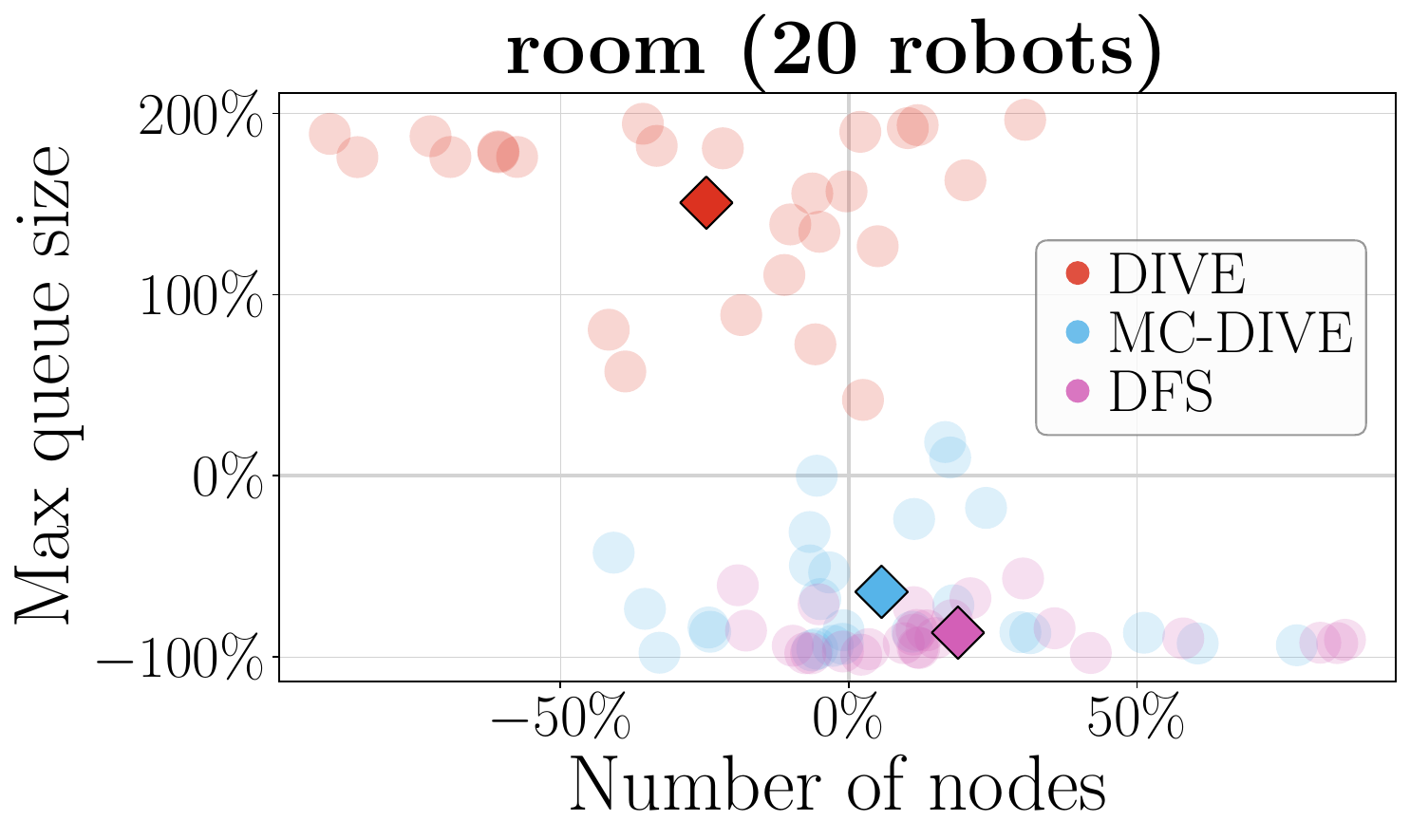}%
    }%
    \subfloat{%
        \includegraphics[width=0.5\linewidth]{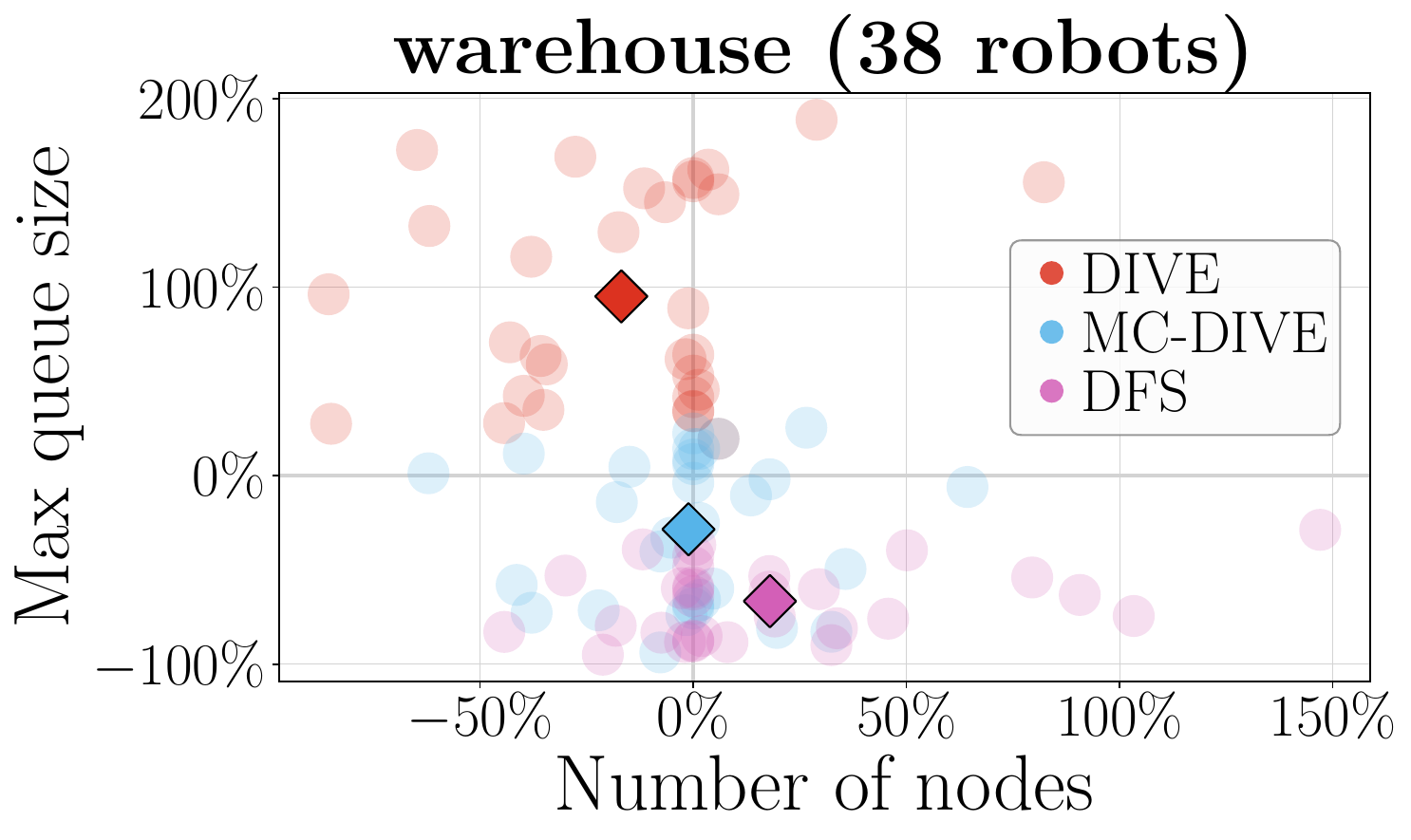}%
    }%
    \caption{Responsive queue control on \texttt{warehouse-small} and \texttt{room-32-32-4}. 
    \Gls{acr:mcd} trades some of \gls{acr:dive}'s dive continuity for substantially lower queue size, while remaining less conservative than always using \gls{acr:dfs}.}
    \label{fig:mcdive}
\end{figure}

\begin{table*}[tb]
\setlength{\abovecaptionskip}{\TableCaptionSkip}
\caption{Structural metrics of \gls{acr:mcd} and \gls{acr:dfs} as mean \% differences relative to \gls{acr:dive}. Lower is better for all three metrics.}
\label{tab:mcdive}
\centering
\renewcommand{\arraystretch}{\TableVerticalStretch}
\begin{NiceTabular}{lc|rrr|rrr|rrr}
    \toprule
     Map & \# robots & \multicolumn{3}{c}{\# nodes} & \multicolumn{3}{c}{\# dive breaks} & \multicolumn{3}{c}{Max queue size}\\
     &  & DIVE & MC-DIVE & DFS & DIVE & MC-DIVE & DFS & DIVE & MC-DIVE & DFS \\
    \midrule
    room-32-32-4 & 20 & 0\% & 145\% & 163\% & 0\%  & 235\% & 259\% & 0\% & -82\% & -93\% \\
    warehouse-small & 38 & 0\% & 53\% & 142\% & 0\%  & 100\% & 195\% & 0\% & -55\% & -79\% \\
    \bottomrule
\end{NiceTabular}
\end{table*}

\Cref{fig:mcdive,tab:mcdive} show the expected trade-off. 
Relative to \gls{acr:dive}, \gls{acr:mcd}'s reduced maximum queue size comes at the cost of more nodes and more dive breaks. 
Always using \gls{acr:dfs} reduces the queue further, but it is over-conservative, as it requires even more nodes and dive breaks than \gls{acr:dive}.
We naturally observe the greatest gains of \Gls{acr:mcd} over \gls{acr:dfs} on \texttt{warehouse-small} where the 50 node queue size limit is less restrictive than \texttt{room-32-32-4}.
\Gls{acr:mcd} behaves as intended, using \gls{acr:dive} when memory allows and becoming more depth-oriented only when the queue demands it.

\subsection{Experimental takeaways}
\label{sec:Experiments_Takeaways}

The experiments support the paper's central node-selection message. 
\gls{acr:dive} is not a uniformly better \gls{acr:bfs}, but a different point in the design space. 
Its extra nodes are the price of controlled primal-side exploration, and that price buys three things that \gls{acr:bfs} and \gls{acr:id} do not provide together, i.e., far fewer dive breaks, substantially lower queue size than \gls{acr:bfs}, and early incumbents with certified gaps. 
The ablation and \gls{acr:mcd} experiments further show that the two mechanisms emphasized in \cref{sec:Dive,sec:FormalizedNodeSelection}, diversified best-bound restarts and feedback, are active contributors to the observed behavior.

\section{Conclusions}
\label{sec:Conclusions}

This paper studied high-level node selection as a central design choice in exact \gls{acr:cbs}. 
The main message is that node selection is not only a device for reducing expanded nodes. 
It also determines memory pressure, parent-child continuity, incumbent discovery, and certificate progress. 
Separating node selection from the correctness-critical parts of \gls{acr:cbs} makes this trade-off explicit while preserving exactness.

\gls{acr:dive} demonstrates the value of this perspective. 
It alternates between global best-bound reanchoring and local depth-oriented diving, using incumbents to prune dives that can no longer improve the best solution found so far. 
The formal analysis shows why this creates a distinct operating point. 
Best-first search remains the natural policy for minimizing expanded nodes and advancing the frontier bound, iterative deepening remains the explicit-memory extreme, and \gls{acr:dive} targets dive continuity while preserving regular returns to the certificate frontier.

The experiments support this trade-off map on practical \gls{acr:mapf} instances. 
\gls{acr:dive} is not a universal replacement for best-first search or iterative deepening, and it is not intended to be one. 
Its value is that it provides capabilities that the standard exact-\gls{acr:cbs} policies do not provide together. 
It substantially reduces dive breaks, keeps a smaller queue than best-first search, exposes early incumbents with certified gaps, and can be stabilized by warm starts in dense regimes. 
These are precisely the capabilities a deployed robotic planner needs when it operates under an onboard memory budget, must react within a bounded replanning window, and may have to act on a feasible plan before a proof of optimality is available.
The ablation experiments further show that diversified best-bound restarts are essential, and the \gls{acr:mcd} experiment illustrates how node selection can respond to memory pressure without changing the exactness argument.

Several directions follow naturally. 
The first is to design responsive hybrids that move continuously between best-bound, depth-oriented, and \gls{acr:dive}-like behavior based on queue growth, incumbent quality, and bound progress.
The second is to learn such responses from families of related \gls{acr:mapf} instances while keeping the underlying pruning rules admissible. 
The third is to extend the analysis beyond perfect trees by using stochastic tree models or measured \gls{acr:cbs} search trees from real benchmarks. 
A final direction is to combine \gls{acr:dive} with stronger warm-start generators and state-of-the-art \gls{acr:cbs} enhancements, where the same node-selection principles can be used without altering the solver's low-level components.

\bibliographystyle{IEEEtran}  %
\bibliography{references}

\appendix
\label{Appendix}
\begin{proof}[Proof of \cref{prop:bnb_nodesel}]
Let $z^{\mathrm P}_t$ be the incumbent value after iteration $t$, and let
\begin{equation*}
    F_t := \{q\in Q : g(q)< z^{\mathrm P}_t\}
\end{equation*}
be the feasible solutions that would strictly improve the current incumbent. 
Let
\begin{equation*}
    U_t := \bigcup_{P\in \mathrm{OPEN}_t} (P\cap Q)
\end{equation*}
be the part of the original feasible set still covered by open nodes. 
We prove the invariant
\begin{equation*}
    F_t \subseteq U_t .
\end{equation*}
At initialization, $\mathrm{OPEN}_0=\{O\}$ and $Q\subseteq O$, so the invariant holds.

Consider an iteration that removes node $P$ from $\mathrm{OPEN}$. 
If $P$ is infeasible, then $P\cap Q=\emptyset$, and removing it cannot remove any member of $F_t$. 
If $\ell(P)\ge z^{\mathrm P}_t$, then every $q\in P\cap Q$ satisfies $g(q)\ge\ell(P)\ge z^{\mathrm P}_t$ by \eqref{eq:valid_lower_bound}. 
Thus, no point in $P\cap Q$ can belong to $F_t$, and incumbent pruning is safe.

Now suppose the relaxed optimizer is feasible, $s(P)\in Q$. 
Since $s(P)$ is the optimum of the relaxation over $P$, every feasible point in $P\cap Q$ has objective at least $g(s(P))$. 
If $s(P)$ improves the incumbent, the incumbent is updated to $g(s(P))$, and no other feasible point in $P$ can improve the new incumbent. 
If it does not improve the incumbent, no feasible point in $P$ can improve the old incumbent either. 
In both cases, the whole node $P$ can be fathomed without violating the invariant.

It remains to consider the branching case. 
The incumbent is unchanged. 
Any improving feasible solution outside $P$ remains covered by the other open nodes. 
Any improving feasible solution inside $P$ lies in $P\cap Q$ and is therefore contained in $L\cup R$ by \eqref{eq:bb_cover}. 
After replacing $P$ by its children, all improving feasible solutions remain covered.

The invariant therefore holds after every iteration, independently of the order in which nodes are selected. 
Because the induced search tree is finite and the policy always selects a node from $\mathrm{OPEN}$, the algorithm eventually terminates with $\mathrm{OPEN}=\emptyset$. 
At termination the invariant gives $F_T=\emptyset$. 
If an incumbent exists, no feasible solution has smaller objective value, so the incumbent is optimal. 
If no incumbent exists, then $z^{\mathrm P}_T=+\infty$ and $F_T=Q$, which implies $Q=\emptyset$ and the original problem is infeasible.
\end{proof}

\end{document}